\pdfoutput=1 
\documentclass[12pt,twoside,singlespace]{mitthesis}
\usepackage{lgrind}
\usepackage{cmap}
\usepackage[T1]{fontenc}
\usepackage{graphicx}
\usepackage{algorithm,multicol,float}

\usepackage{amsmath}
\usepackage{amsthm}
\usepackage{verbatim} 
\usepackage{amsfonts}
\usepackage{multirow}

\usepackage[numbers]{natbib}

\usepackage{color}

\DeclareMathOperator*{\argmax}{argmax}
\newcommand{\expect}{\mathbb{E}}
\newcommand{\reals}{\mathbb{R}}

\newcommand{\figref}[1]{{Figure~\ref{#1}}}

\DeclareSymbolFont{symbolsC}{U}{txsyc}{m}{n}
\usepackage{amssymb}
\usepackage{enumerate}
\usepackage{amsthm}
\usepackage{color}
\newcommand{\red}[1]{{\textcolor{red}{#1}}}
\newcommand{\todo}[1]{\red{TODO: {#1}}}

\usepackage[noend]{algpseudocode}
\usepackage{esvect}
\algnewcommand\algorithmicinput{\textbf{Input:}}
\algnewcommand\INPUT{\item[\algorithmicinput]}
\algnewcommand\algorithmicoutput{\textbf{Output:}}
\algnewcommand\OUTPUT{\item[\algorithmicoutput]}
\algnewcommand\algorithmicinitialize{\textbf{Initialize:}}
\algnewcommand\INIT{\item[\algorithmicinitialize]}

\usepackage{float}
\newtheorem{innercustomgeneric}{\customgenericname}
\providecommand{\customgenericname}{}
\newcommand{\newcustomtheorem}[2]{%
  \newenvironment{#1}[1]
  {%
   \renewcommand\customgenericname{#2}%
   \renewcommand\theinnercustomgeneric{##1}%
   \innercustomgeneric
  }
  {\endinnercustomgeneric}
}

\newcustomtheorem{customthm}{Theorem}
\newcustomtheorem{customlemma}{Lemma}
\newtheorem*{rep@theorem}{\rep@title}
\newcommand{\newreptheorem}[2]{%
\newenvironment{rep#1}[1]{%
 \def\rep@title{#2 \ref{##1}}%
 \begin{rep@theorem}}%
 {\end{rep@theorem}}}
 \makeatother

\newtheorem{thm}{Theorem}
\newtheorem{lem}{Lemma}
\newtheorem{theorem}{Theorem}
\newreptheorem{theorem}{Theorem}
\newtheorem{lemma}{Lemma}
\newreptheorem{lemma}{Lemma}

\newtheorem{corollary}[theorem]{Corollary}
\newreptheorem{corollary}{Corollary}
\newtheorem{definition}{Definition}

\newreptheorem{definition}{Definition}

\usepackage{times}
\usepackage{caption}
\usepackage{subcaption}
\usepackage{thmtools}
\usepackage{nameref}

\begin{document}

% -*-latex-*-
% 
% For questions, comments, concerns or complaints:
% thesis@mit.edu
% 
%
% $Log: cover.tex,v $
% Revision 1.8  2008/05/13 15:02:15  jdreed
% Degree month is June, not May.  Added note about prevdegrees.
% Arthur Smith's title updated
%
% Revision 1.7  2001/02/08 18:53:16  boojum
% changed some \newpages to \cleardoublepages
%
% Revision 1.6  1999/10/21 14:49:31  boojum
% changed comment referring to documentstyle
%
% Revision 1.5  1999/10/21 14:39:04  boojum
% *** empty log message ***
%
% Revision 1.4  1997/04/18  17:54:10  othomas
% added page numbers on abstract and cover, and made 1 abstract
% page the default rather than 2.  (anne hunter tells me this
% is the new institute standard.)
%
% Revision 1.4  1997/04/18  17:54:10  othomas
% added page numbers on abstract and cover, and made 1 abstract
% page the default rather than 2.  (anne hunter tells me this
% is the new institute standard.)
%
% Revision 1.3  93/05/17  17:06:29  starflt
% Added acknowledgements section (suggested by tompalka)
% 
% Revision 1.2  92/04/22  13:13:13  epeisach
% Fixes for 1991 course 6 requirements
% Phrase "and to grant others the right to do so" has been added to 
% permission clause
% Second copy of abstract is not counted as separate pages so numbering works
% out
% 
% Revision 1.1  92/04/22  13:08:20  epeisach

% NOTE:
% These templates make an effort to conform to the MIT Thesis specifications,
% however the specifications can change.  We recommend that you verify the
% layout of your title page with your thesis advisor and/or the MIT 
% Libraries before printing your final copy.
\title{Efficient Structured Surrogate Loss and Regularization in  Structured Prediction}

\author{Heejin Choi}
% If you wish to list your previous degrees on the cover page, use the 
% previous degrees command:
%       \prevdegrees{A.A., Harvard University (1985)}
% You can use the \\ command to list multiple previous degrees
%       \prevdegrees{B.S., University of California (1978) \\
%                    S.M., Massachusetts Institute of Technology (1981)}
\department{Toyota Technological Institute of Technology}

% If the thesis is for two degrees simultaneously, list them both
% separated by \and like this:
 \degree{Doctor of Philosophy in Computer Science}
%\degree{Bachelor of Science in Computer Science and Engineering}

% As of the 2007-08 academic year, valid degree months are September, 
% February, or June.  The default is June.
\degreemonth{August}
\degreeyear{2018}
\thesisdate{August 31, 2018}

%% By default, the thesis will be copyrighted to MIT.  If you need to copyright
%% the thesis to yourself, just specify the `vi' documentclass option.  If for
%% some reason you want to exactly specify the copyright notice text, you can
%% use the \copyrightnoticetext command.  
%\copyrightnoticetext{\copyright IBM, 1990.  Do not open till Xmas.}

% If there is more than one supervisor, use the \supervisor command
% once for each.
\supervisor{Nathan Srebro}{Professor}

% This is the department committee chairman, not the thesis committee
% chairman.  You should replace this with your Department's Committee
% Chairman.
\chairman{Avrim Blum}{Chief Academic Officer}

% Make the titlepage based on the above information.  If you need
% something special and can't use the standard form, you can specify
% the exact text of the titlepage yourself.  Put it in a titlepage
% environment and leave blank lines where you want vertical space.
% The spaces will be adjusted to fill the entire page.  The dotted
% lines for the signatures are made with the \signature command.
\maketitle

% The abstractpage environment sets up everything on the page except
% the text itself.  The title and other header material are put at the
% top of the page, and the supervisors are listed at the bottom.  A
% new page is begun both before and after.  Of course, an abstract may
% be more than one page itself.  If you need more control over the
% format of the page, you can use the abstract environment, which puts
% the word "Abstract" at the beginning and single spaces its text.

%% You can either \input (*not* \include) your abstract file, or you can put
%% the text of the abstract directly between the \begin{abstractpage} and
%% \end{abstractpage} commands.

% First copy: start a new page, and save the page number.
\cleardoublepage
% Uncomment the next line if you do NOT want a page number on your
% abstract and acknowledgments pages.
% \pagestyle{empty}
\setcounter{savepage}{\thepage}
\begin{abstractpage}
% $Log: abstract.tex,v $
% Revision 1.1  93/05/14  14:56:25  starflt
% Initial revision
% 
% Revision 1.1  90/05/04  10:41:01  lwvanels
% Initial revision
% 
%
%% The text of your abstract and nothing else (other than comments) goes here.
%% It will be single-spaced and the rest of the text that is supposed to go on
%% the abstract page will be generated by the abstractpage environment.  This
%% file should be \input (not \include 'd) from cover.tex.
In this dissertation, we focus on several important problems in structured prediction. 
In structured prediction, the label has a rich intrinsic substructure, and the loss varies with respect to the predicted label and the true label pair.  Structured SVM is an extension of binary SVM\ to adapt to such structured tasks.

In the first part of the dissertation, we study the surrogate losses and its efficient methods. To minimize the empirical risk,  a surrogate loss which upper bounds the loss, is used as a proxy to minimize the actual loss. Since the objective function is written in terms of the surrogate loss,  the choice of the surrogate loss is important, and the performance depends on it. Another issue regarding the surrogate loss is  the efficiency of the argmax label inference for the surrogate loss. Efficient inference is necessary for the optimization since it is often the most time-consuming step. We present a new class of surrogate losses named  {\em bi-criteria surrogate loss}, which is a generalization of the popular surrogate losses.  We first investigate an efficient method for a slack rescaling formulation as a starting point utilizing decomposability of the model. Then, we extend the algorithm to the bi-criteria surrogate loss, which is very efficient and also shows  performance improvements.

In the second part of the dissertation, another important issue of regularization is studied. Specifically,  we investigate  a problem of regularization in hierarchical classification when a structural imbalance exists in the label structure. We present a method to normalize the structure, as well as a new norm, namely shared Frobenius norm. It is suitable for hierarchical classification that adapts to the data in addition to the label structure. 
%We show the empirical improvement in large datasets.  

\end{abstractpage}

% Additional copy: start a new page, and reset the page number.  This way,
% the second copy of the abstract is not counted as separate pages.
% Uncomment the next 6 lines if you need two copies of the abstract
% page.
% \setcounter{page}{\thesavepage}
% \begin{abstractpage}
% \input{abstract}
% \end{abstractpage}

\cleardoublepage
%\begin{comment}
\section*{Acknowledgments}
The way to this dissertation has been quite a journey. I sincerely thank my advisor Nathan Srebro for giving me advice from the start and till the end. He gave me challenging questions to develop idea further. Also, I deeply appreciate Ofer Meshi who introduced me to the new exciting problem of surrogate losses and structured prediction, and has helped me through new the ideas. I also thank my other committee members. Greg Shakhnarovicha in TTIC gave me  kindly gave advices and answered my questions when I\ try to apply my ideas to new area of computer vision problems. Yutaka Sasaki in TTI Japan helped me develop my ideas in hierarchical classification.
 Other faculties in TTIC, especially, Kevin Gimpel, 
Karl Stratos, and  Madhur Tulsiani were always encouraging and gave me interesting problems and ideas, and I thank them for the opportunity and help. I thank all the good friends in TTIC and University of Chicago; Lifu Tu, Wooseok Ha, Qingming Tang, Taewhan Kim,  Payman Yadollahpour, Jialei Wang, and Behnam Neyshabur, to  name a few. Lastly, I cannot thank my family, Soo Bong Choi, Yoon Hee Yeom, Hyung Jin Choi, and Woo Hyun Choi, enough since it was the encouragement and support of them which made me overcome difficulties on the way.
%\end{comment}

%%%%%%%%%%%%%%%%%%%%%%%%%%%%%%%%%%%%%%%%%%%%%%%%%%%%%%%%%%%%%%%%%%%%%%
% -*-latex-*-

% Some departments (e.g. 5) require an additional signature page.  See
% signature.tex for more information and uncomment the following line if
% applicable.
% -*- Mode:TeX -*-
%
% Some departments (e.g. Chemistry) require an additional cover page
% with signatures of the thesis committee.  Please check with your
% thesis advisor or other appropriate person to determine if such a 
% page is required for your thesis.  
%
% If you choose not to use the "titlepage" environment, a \newpage
% commands, and several \vspace{\fill} commands may be necessary to
% achieve the required spacing.  The \signature command is defined in
% the "mitthesis" class
%
% The following sample appears courtesy of Ben Kaduk <kaduk@mit.edu> and
% was used in his June 2012 doctoral thesis in Chemistry. 

\begin{titlepage}
\begin{large}
This doctoral thesis has been examined by a Thesis Committee as follows:

\signature{Professor Nathan Srebro}{Thesis Supervisor \\
   Professor}

\signature{Professor Greg Shakhnarovich}{Member, Thesis Committee \\
   Associate Professor}

\signature{Professor Yutaka Sasaki}{Member, Thesis Committee \\
   Faculty of Engineering}

\signature{Ofer Meshi}{Member, Thesis Committee }

\end{large}
\end{titlepage}

\pagestyle{plain}
  % -*- Mode:TeX -*-
%% This file simply contains the commands that actually generate the table of
%% contents and lists of figures and tables.  You can omit any or all of
%% these files by simply taking out the appropriate command.  For more
%% information on these files, see appendix C.3.3 of the LaTeX manual. 
\tableofcontents
\newpage
\listoffigures
\newpage
\listoftables

%&latex
\chapter{Introduction}
   
Departing from the dataset with simple labels, many real-world problems can be formulated as problems that deal with a rich internal label structure. For instance, in a multi-label problem, the labels are sets of dependent micro-labels and in a dependency parsing task,  a label forms a tree of micro-labels. Such problems are called {\em structured prediction}\cite{bakir2007predicting} in machine learning.  This framework is applicable widely to computer vision\cite{forsyth2003modern, nowozin2011structured}, natural language processing\cite{charniak1996statistical,manning1999foundations, daume2006practical}, computational biology\cite{gusfield1997algorithms,durbin1998biological}, and many other fields. %Also it has been applied to various deep models\cite{, jurafsky2014speech, vinyals2015grammar}.
 \emph{Structured SVM}\cite{tsochantaridis2004support} is a generalization of binary SVM to such structured outputs, and is a prevalent method with much success.

 The complex structure of the label in the structured prediction often involves a {\em cost-sensitive learning}
\cite{bradford1998pruning,elkan2001foundations,margineantu2002class}
 where the loss is a function of a correct label and a predicted label rather than a $0-1$ loss. For  instance, consider a multi-label classification problem.  The cost of predicting to a label with one error micro-label should have less penalty than the label with entirely different labels. %Another example is when the label is a sequence of micro-labels where the cost varies with the correctness of the prediction. 
The cost is called the performance measure or {\em task loss} and there are many performance measures in the literature, for instance, Micro-F1 or Hamming loss, to name a few.

 In structured SVM,  {\em surrogate loss} is used in the objective to optimize the task loss. It eases the optimization by removing the complex non-convex dependence on weight vector $w$. Structured SVM\ uses  the surrogate loss that is written as a sum of the losses for each instance, and the loss for an instance is defined to be the  maximum loss over the labels, i.e., $\sum_i\max_y \tilde{\Phi}(w,y,y_i,x_i)$ where  $\tilde{\Phi}$ is the loss for the instance $x_i$, the correct label   $y_i$, the predicted label $y$, and the learned weight vector  $w$. Surrogate loss is connected to the task loss by the fact that  surrogate loss upper bounds the task loss. Then, minimizing the surrogate loss results minimizing the task loss.
Sometimes when task loss is even hard to upper bound, a surrogate loss that is highly correlated with the task is used. For instance, Hamming loss is used as a surrogate loss for optimizing Macro-F1. 

Surrogate loss introduces a critical issue: the efficiency of loss augmented inference, which amounts to finding the label whose loss attains the maximum given an instance. It is required in the optimization and often the most time-consuming step in the structured prediction due to the exponentially large size of the labels.

 In the first part of the dissertation, we study the surrogate loss for structured prediction. Since  surrogate loss defines the objective function, the choice of the surrogate loss is essential for the good performance. Since the efficiency of the inference is crucial, we study the surrogate loss that its efficiency of inference benefits from decomposability. Decomposability is an important property of the model that its potential and task loss decompose respect to the same substructure. Often if the model is decomposable, loss augmented inference of additive form, which is called margin rescaling, is efficiently solvable. This is widely used in practice for efficient learning. Using this efficient inference, margin rescaling inference, as a tool, we try to enable efficient inferences for other complex surrogate losses. We present many efficient inference methods in addition to the study of the limitations of this approach. Also in the era of deep networks\cite{
chen2015learning,
schwing2015fully,
krizhevsky2012imagenet,
goodfellow2014generative}, we show that our methods can be successfully applied to deep networks.

In the second part of the dissertation, we investigate an issue of regularization in hierarchical classification, which is one important problem in structured prediction. Regularization is an essential issue in hierarchical classification because structured prediction involves many micro-labels and its parameters, and often the structural relationship is expressed only in the regularization. This is particularly true for the hierarchical classification where the micro-labels form a hierarchy, such as a tree. A hierarchical structure is a natural structure of the labels that expresses its granularity by groupings, which is commonly used in many datasets.  We investigate how to normalize the regularization for the unbalanced structure. We also extend it by presenting a new norm to apply for such a problem.

\section{Main Contributions and the overview}
 We briefly give an overview and summarize the main contributions.
\begin{itemize}
\item Chapter \ref{cha:surrogate}: \nameref{cha:surrogate}. 
 
Task loss is the actual performance measure of interest in test time, e.g. the number of wrongly predicted labels in a multi-label problem.  Since the task loss is mostly non-convex and discontinuous, it is difficult to optimize the task loss directly. Surrogate losses, which are commonly convex and continuous, are used as a proxy term in the place of the task loss in the objective function to ease the optimization. 
%In this chapter, we first review the common surrogate losses, and identify the potential problems, and present our new class of surrogate losses as a solution to the problem.  
We first compare the two common surrogate losses that are commonly used, margin rescaling and slack rescaling. While  inference of the margin rescaling formulation is efficient, it requires a large separating margin from the true label. This causes a problem that a well-separated label in the correct side of the decision boundary can be selected as the most violating label even though there exists a label in the wrong side of the boundary. On the other hand, the slack rescaling formulation only requires a unit margin. Thus it is often suggested as a solution to the previously mentioned problem.  Secondly, we focus on the property that the previously mentioned surrogate losses are defined by a simple relationship, an addition or a multiplication, between two factors, the structural loss (a loss that only depends on the labels) and the classification margin (the margin of score given a feature and a  label). We generalize the surrogate losses to {\em bi-criteria surrogate loss} by extending the relationship to quasi-concave functions, includes the common surrogate losses previously mentioned. We present two interesting examples of the bi-criteria surrogate losses, ProbLoss and Micro-F1 surrogate. 

\begin{itemize}
\item Bi-criteria surrogate loss (Definition \ref{def:gbicriteria}): A new class of surrogate loss is presented that models complicated interaction between the structural loss and the classification margin.

\item ProbLoss (Section \ref{sec:probloss}) and Micro-F1 surrogate (Section \ref{sec:microf1}): Two examples for the bi-criteria surrogate loss are given. ProbLoss is a bi-criteria surrogate loss that directly upper bounds the risk under  normal noise assumptions.
Micro-F1 surrogate is 
a new surrogate loss that optimizes Micro-F1 score.

\end{itemize}

 Contributions:
\begin{enumerate}
 \item  We give a novel analysis of the difference between margin rescaling and slack rescaling. 
 \item We present a new group of surrogate
losses termed bi-criteria surrogate loss. It extends the simple relationship of structural loss and margin loss to a quasi-concave function.

 \item We give two examples of the bi-criteria surrogate loss: ProbLoss and Micro-F1 surrogate.
\end{enumerate}
 
\item Chapter \ref{cha:inference}: \nameref{cha:inference}. (From the previous publication \cite{choi2016fast})
 
  In this chapter, we first focus on the problem of the argmax label inference for the slack rescaling surrogate loss. Our approach is based on an oracle called {\em $\lambda$-oracle}. $\lambda$-oracle is similar to  argmax inference of margin rescaling surrogate loss, but by using different  $\lambda$ iteratively, which controls the weight on the structural loss, we are able to obtain different labels. We study the efficient method to obtain the argmax label of slack rescaling inference using the oracle and its limitations. First, we present a binary search method that is applicable to the stochastic gradient descent method removing the constraint of the method first suggested in \cite{sarawagi2008accurate}. Then, we show that by mapping the labels into the 2D plane, we are able to construct a more efficient binary search method. Additionally, we show a negative result that if only $\lambda$-oracle is used exact inference is not possible. To address the  problem, we introduce a new more powerful oracle that has control over the label with added two linear constraints, which we call a constrained $\lambda$-oracle. We present a new method called angular search  that uses the constrained $\lambda$-oracle. Angular search finds the optimal label with the exponential decaying suboptimality.  Finally, we show another negative result that even with the constrained $\lambda$-oracle or more  powerful oracle with unlimited linear constraints, the number of oracle call can be as large as the number of labels. However, the worst case relies on the extremely high precision of location of labels that increases with the iterations, and  we conclude that it is unlikely to happen in practice.
\begin{itemize}
\item Geometrical interpretation (2D Label mapping)
(Section \ref{sec:geometrical_inter})

 We map labels into the 2D plane, which gives new intuitions for $\lambda$-oracle search.

\item Various new inference methods for slack rescaling formulation with $\lambda$-oracle are presented. 
\begin{enumerate}
\item Binary search for SGD (Section \ref{sec:binary_sgd}):
 A new method using binary search that applies to SGD.

\item Bisecting search (Algorithm \ref{alg:bisecting}): New improved binary search method utilizing the 2D geometrical interpretation of labels, and also applicable to SGD.

\item Angular search (Algorithm \ref{alg:angular}):  New method that utilizes constrained $\lambda$-oracle, which is more powerful than $\lambda$-oracle, to find the exact optimum. \end{enumerate}

\item Limitation of $\lambda$-oracle (Theorem \ref{loracle_l1}):
 In the worst case, $\lambda$-oracle cannot return the optimal label even with the unlimited calls to the oracle. 

\item Lower bound of the number of the constrained $\lambda$-oracle calls to find the optimal label (Theorem \ref{loracle_l2}): In the worst case, the number of the constrained $\lambda$-oracle calls cannot be less than the number of labels. 

\end{itemize}

Contributions:
\begin{enumerate}

 \item  We give a new analysis of the $\lambda$-oracle  by introducing $2$ dimensional mapping of labels. It shows that each call to the $\lambda$-oracle  shrinks the candidate space of the optimal label.
 \item   We present new binary search methods for slack rescaling formulation that are applicable to SGD.
 \item  We present a new method, termed angular search, that finds the exact optimal label in slack rescaling formulation with the constrained $\lambda$-oracle, a $\lambda$-oracle with two additional linear constraints.
 \item We give  analyses of theoretical limitations of the method using $\lambda$-oracle: In the worst case, the optimal label in slack rescaling formulation cannot be returned by $\lambda$-oracle. Also in the worst case, any method that uses the constrained $\lambda$-oracle needs to call the oracle $|\mathcal{Y}|$ times to find the optimal label in slack rescaling formulation where $|\mathcal{Y}|$ is the size of the label set.

 \end{enumerate}

\item Chapter \ref{cha:inference2}: \nameref{cha:inference2}. 

  In this chapter, we focus on the efficient inference method for bi-criteria surrogate loss using $\lambda$-oracle, which we call  {\em convex hull search}. We show that for the bi-criteria surrogate loss, there exists a linear lower bound
of the candidate space of the optimal label in the 2D plane, and the $\lambda$ corresponding to the linear lower bounds gives the optimal $\lambda$ for the method. It is used to find the optimal label achievable using only $\lambda$-oracle. We obtain the maximum label in the fractional label space which is on the one of the neighboring edges.  Then, the integral label is obtained using other  oracles more often available than the constrained $\lambda$-oracle, such as a k-best oracle or a ban-list oracle. We show the optimality result in an adversarial game between two players that the binary search more often fails to find the optimum but our search finds the optimum in a few iterations. In practice, our convex hull search is surprisingly efficient that finds the integral solution in the average of $2$ oracle calls in name entity recognition tasks. We show empirical improvement in multi-label problem, and name entity recognition problem. Finally, in a dependency parsing problem, we show that our approach is also applicable to the deep network. 
\begin{itemize}
\item Convex Hull Search (Algorithm \ref{alg:convexhull_search}): An efficient method for bi-criteria surrogate loss using $\lambda$-oracle. The number of the call is upper bounded by the number of vertices of the convex hull. It uses the optimal $\lambda$ at each iteration.

\item We empirically show the improvement over baseline method in the multi-label problem, name entity recognition problem, and dependency parsing problem.

\end{itemize}

Contributions:
\begin{enumerate}
 \item  We present  a new inference method, termed convex hull search, for bi-criteria surrogate loss. It finds the optimal label in the  fractional label space, which is a convex space of the labels in the $2$D plane.
 \item  We analyze the optimality of choice of the $\lambda$ used in convex hull search in an adversarial setting.

 \item  Empirical evaluations of the convex hull search and the bi-criteria loss are given in the multi-label problem, name entity recognition problem, and dependency parsing problem.
\end{enumerate}

\item Chapter \ref{cha:regularization}: \nameref{cha:regularization}. (From the previous publication \cite{choi2016fast})

 Hierarchical classification is a problem where given a tree or a direct acyclic graph, classifying the instances to the leaves of the graph. We first introduce {\em hierarchical structured SVM}. In hierarchical structured SVM, a linear classifier is assigned to each node with a weight vector, and its  score, which we refer to  potential, is given as the innerproduct between the weight vector and the instance vector. Then, the potential for a label (or  leaves if we consider a multi-label problem) as the sum of the potentials of corresponding the leaf and that of its ancestors. Regularizer is the sum of the $\ell_2$-norm of the weight vectors. We show that the model imposes a bias toward the deeper labels in the hierarchy. That is, $1$ norm ball has a larger classifier for the deeper leaves than the shallow leaves. To address this problem, we propose {\em Normalized Hierarchical SVM} to remove such bias only considering the structure of the graph. We show that it also removes the redundant nodes in the graph.  We extend the formulation to learn the optimal normalization from the data. The extension is presented as a new norm, {\em Shared Frobenius Norm}.  Efficient convex optimization methods are presented. Finally, we show that our normalization is necessary for achieving high performance.

\begin{itemize}

\item Normalized Hierarchical SVM (NHSVM) (Section \ref{sec:NHSVM})
 : Hierarchical SVM\ with normalized regularization respect to the unbalance in the graphical structural of the label.
\item Shared  SVM (SSVM)
(Section \ref{sec:SharedSVM})
: Hierarchical SVM with a data-dependent normalization of regularization, a shared Frobenius Norm.

\item Theorem \ref{thm:invarinace} (Section \ref{sec:invar}): We show that Normalized Hierarchical SVM is invariant to the duplication of the labels.

\item Shared Frobenius Norm(Section \ref{sec:SharedSVM}):
A new norm that is suitable for unbalanced hierarchical structure.

\item 
 We empirically evaluate our method with the improvement in the synthetic datasets and large-scale datasets.

\end{itemize}

Contributions:
\begin{enumerate}
 \item  We identify the problem of regularization in hierarchical structured SVM when there exists an imbalance in the number of ancestors of leaves. \item  We present Normalized Hierarchical SVM that removes the imbalance  considering the label structure.

 \item  A new norm, called shared Frobenius norm, is presented as an alternative to the Frobenius norm in  hierarchical classification. 

 \item  An efficient optimization method is presented.

 \item  We empirically evaluate our method with a large hierarchical dataset to show that our normalization is essential for high performance.
\end{enumerate}

\end{itemize}

%% This is an example first chapter.  You should put chapter/appendix that you
%% write into a separate file, and add a line \include{yourfilename} to
%% main.tex, where `yourfilename.tex' is the name of the chapter/appendix file.
%% You can process specific files by typing their names in at the 
%% \files=
%% prompt when you run the file main.tex through LaTeX.
\chapter{Preliminary}\label{cha:preliminary}

In this chapter, we briefly clarify our problem of focused, structured prediction. Then, we introduce the basic statistical learning frameworks in this dissertation. Also, we review Structured SVM \cite{tsochantaridis2004support,Taskar03}, which is a popular model to deal with structured output space.  It also serves as our principle approach to the structured prediction.   %Also we point out an important property of structured prediction, decomposability, which enables efficient inference. 

\section{Classification}
Classification is one of the most common problems in machine learning. 
 In the usual supervised setting, the goal is to learn a function that maps a description of a data $x\in\mathcal{X}$, usually $\mathcal{X}\subseteq \mathbb{R}^d$, to a label $y\in\mathcal{Y}$ as correctly as possible given dataset $\mathcal{D}=\{x_i\in\mathcal{X},y_i\in\mathcal{Y}\}^N$. One way to distinguish the problems of classification is the size of the label set. 
\begin{enumerate}
\item Binary Classification: The size of the label set is $2$, i.e., $|\mathcal{Y}|=2$, $\mathcal{Y}=\{-1,1\}$. Examples: Is the email spam or not spam? Will the stock price go up or down?
\item Multi-class classification: The size of the label set is $M>2$, i.e., $|\mathcal{Y}|=M$, $\mathcal{Y}=[M]$. Examples: Which digit does this image belong to? Which one of the categories does this document fall into?

\item Structured Prediction: The size of the label set is $M^m$, i.e., $|\mathcal{Y}|=M^m$, $\mathcal{Y}=[M]^m$. The size of the label is exponentially large. Examples: Which tree does this sentence parse into? Which multiple choice of the categories does this document fall into?
\end{enumerate}

 In this dissertation, we are interested in the last category, the structured prediction.

\section{Structured Prediction}

 As in the last section, the structured prediction problem is noted by its exponential label size $|\mathcal{Y}|$. The exponential size commonly comes from its decomposability into micro-labels. That is,  $y$ is a vector of length $m$, and each element is termed micro-labels taking one of $M$ values, i.e.,  $y\in\mathcal{Y}= [M]^m$, $y^{(k)}\in[M]$ where the superscript corresponds to the index of the elements in the vector. Therefore, the structured prediction tries to learn $m$ labels at once. If there is no dependence between the micro-labels, we can change the problem into learning $m$ micro-labels independently, and the problem becomes multi-class classification problem with $m$ independent classifiers. The premise
of structured prediction is that the micro-labels are dependent and it is beneficial to learn the label concurrently by modeling the dependence. 
%In other words, a potential of a label does not decompose  as a sum over its micro-labels, but as a sum over cliques, i.e.,
%\begin{align*}
% \Phi(y)=\sum_k \bar{\Phi}_k(y^{(k)})+\sum_{C} \bar{\Phi}_{C}(y^{C_1})
%\end{align*}
% is not multivariate function of $y^{(t)}, t\in C\subseteq [M].$  
In such problems, incorporating the structure is important for achieving good performance. Structured Prediction is used widely in computer vision\cite{forsyth2003modern, nowozin2011structured}, natural language processing\cite{charniak1996statistical,manning1999foundations, daume2006practical}, computational biology\cite{gusfield1997algorithms,durbin1998biological}, and many other fields.  Following are examples of the structured prediction problems and the datasets used in the dissertation.

\begin{enumerate}
\item Multi-label Problem

In a multi-label problem, an instance can belong to $m$ multiple labels. This can be formulated in to $m$ independent learning problem. However, in many cases, each labels are dependent on other labels, and gives an important information, e.g., $P(y_1=1|y_2=1,x)\ne P(y_1=1|x)$. In this dissertation,  we only focus on a fully pair-wise potential model\cite{finley2008training}. 
 \begin{table}[H]
\begin{center}
\begin{tabular}{|c|c|}
\hline 
Dataset &  Description   \\
\hline
Yeast & Predict cellular 
          localization sites for bacteria\cite{elisseeff2002kernel}\\ 
RCV1          & News document classification\cite{lewis2004rcv1} \\
Mediamill & Predict labels from extracted features from the video\cite{snoek2006challenge}       \\
Bibtex & Predict tags from the document\cite{katakis2008multilabel}\\

\hline
\end{tabular}
\caption{Multi-label  Dataset.}
\end{center}
\end{table}

\item Sequence Prediction

 In a sequence prediction problem, labels form a linear chain. If we limit the dependence to neighboring nodes, the labels become  Markov random fields. In the named-entity recognition (NER) task, given text data, the problem is to identify entity types; person, location, organization, and miscellaneous.
In this dissertation,  we consider a model with a pair-wise potential of neighboring nodes in the chain\cite{gimpel2010softmax}. 

\begin{table}[H]
\begin{center}
\begin{tabular}{|c|c|}
\hline 
Dataset &  Description   \\
\hline
CoNLL 2003 shared task (NER) data & A subset of
Reuters 1996 news corpus\cite{ratinov2009design}\\
\hline
\end{tabular}
\caption{Sequence Prediction Dataset.}
\end{center}
\end{table}

\item Tree Structure Prediction
 
 Each label forms a tree. In dependency parsing tasks, we assign dependency on the words in a sentence
as a form of a tree. Each arc is labeled as a part of speech (POS) tag\cite{kiperwasser2016simple}. See Figure \ref{fig:tree_structure}\footnote{Natural language processing with Python, O'Reilly, 2009.}.
\begin{figure}[H]
\begin{center}
\includegraphics[width=.5\linewidth]{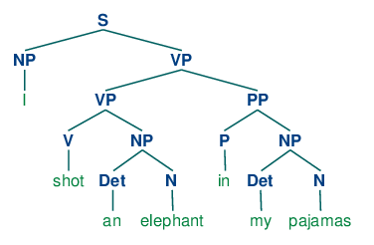}
\caption{Tree Structure in Dependency Parsing Task}
\label{fig:tree_structure}
\end{center}
\end{figure}

\begin{table}[H]
\begin{center}
\begin{tabular}{|c|c|}
\hline 
Dataset &  Description   \\
\hline
Penn Treebank 
 & Selected stories from a three year Wall Street Journal collection \cite{marcus1993building} \\
\hline
\end{tabular}
\caption{Tree Structured Dataset.}
\end{center}
\end{table}

\end{enumerate}

\section{Statistical Learning Framework
}\label{sec:framework}
In this section, we first briefly  review the statistical frameworks based on \cite{friedman2001elements,
mohri2012foundations,
NIPS2016_6485}. 
%\subsection{Structured SVM}
%\subsection{Objective function and loss functions}

 Let $\mathcal{X}$ be the domain of instances, $\mathcal{Y}$ be the set of the output label space.  Let $\mathcal{D}$ be the unknown true distribution over $\mathcal{X}\times\mathcal{Y}$. In this dissertation, we are interested in structured prediction, where $\mathcal{Y}$ is a combinatorial space. For instance, for the sequence prediction of length $m$, $|\mathcal{Y}|=M^m$ where $M$ is the size of micro-label. 
%This makes the inference, which includes solving argmax label over $\mathcal{Y}$ difficult. 
We consider a learning framework with a linear potential function $f_w(x,y)=\phi(x,y)^T\cdot w$ for $x\in\mathcal{X}$ and $y\in\mathcal{Y}$ which is parameterized by $w\in \mathcal{H}$. Let $\mathcal{H}=\{w\in \mathbb{R}^d|\;\|w\|_2\le\rho\}$  be the set of hypothesis.  
 %Let $\tilde{m}(w,x,y,y')= f_{w}(x,y)-f_{w}(x,y')$ is the margin. 
  Inference function is defined as $\hat{y}_w(x)=\argmax_y f(x,y)$.  Let $L(y,y')$ be the task loss, the specific loss or the performance measure of interest. It defines the loss predicting an instance with true label $y'$ as label $y$. We assume $L(y,y')> 0$  if $\forall y\ne y'$, and $L(y,y')=0 \iff y=y'$.
 Then , the objective is to minimize expected task loss  over the unknown population $\mathcal{D}$,
\begin{align*}
 R(f)=\expect_{(x,y)\sim \mathcal{D}} L(\hat{y}_w(x),y)
\end{align*}
\begin{comment}
Inference for task loss $L$ can be difficult, which will be discussed in detail later. Then, instead of directly optimizing respect to $L$, often we rely on simple surrogate task loss $\dot{L}$ which is not directly connected to but has a high correlation to the task loss.  For instance, since it is hard to optimize with respect to micro-F1 since it is a rational function, often hamming loss is used instead. From now on, we assume that our objective is to minimize expected surrogate risk $\dot{R}$ respect to surrogate task loss $\dot{L}$. 
\begin{align*}
 \dot{R}(f)=\expect_{(x,y)\sim \mathcal{D}} \dot{L}(\hat{y}_w(x),y)
\end{align*}
\end{comment}
Since the true population is unknown, we take structured risk minimization approach. It minimizes the empirical loss with the surrogate loss $\tilde{L}(w,x_i,y_i)$ that upper bounds $L(\hat{y}_w(x_{i}),y_{i})$ over the identically and independently distributed sample $S=\{(x_i,y_i)\}_{i=1}^{N}$.  Then, with probability of $1-\delta$ , the generalization bound   
\begin{align}
 R(f)\le   \frac{1}{N}\sum_{i=1}^N \tilde{L}(w,x_i,y_i)+\Lambda(\tilde{L},\mathcal{H})+\epsilon(N,\delta) \label{eq:generalization_bound}
\end{align}
holds which upper bounds the  risk where  $\tilde{L}:\mathcal{H}\times\mathcal{X}\times\mathcal{Y}\mapsto \mathbb{R}_+$ is the surrogate loss. The generalization bound has  additional terms: $\Lambda$ that upper bounds maximum difference of two losses given the hypothesis size and the surrogate loss (which is represented by Rademacher complexity), and a probability term $\epsilon$ that is a function of $N$ and probability  $\delta$\cite{NIPS2016_6485}.

To obtain minimum of RHS of \eqref{eq:generalization_bound}, regularization is added with constant $C\ge 0$ to control $\Lambda(\hat{L},w)$, then objective function becomes,   
\begin{align}
 \frac{C}{2}\|w\|^{2}_F+\frac{1}{N}\sum_{i=1}^N \tilde{L}(w,x_i,y_i)
 \label{eq:objective_function_general}
\end{align}
We only focus on the following form of a surrogate loss,
  \begin{align} %\textstyle
\tilde{L}(w,x_i,y_i)=\max_y\tilde{\Phi}(w,y,y_{i},x_{i})
\label{eq:max_surrogate}
\end{align}
Then, the objective function is
  \begin{align} %\textstyle
 \frac{C}{2}\|w\|^2_2
        +\frac{1}{N}\!\sum_i\max_{y\in\mathcal{Y}} \tilde{\Phi}(w,y,y_{i},x_{i})
        \label{eq:objective_function_struct_svm}
\end{align}
The optimization methods  can be divided into two categories: the first order method\cite{RatliffSubgradient,shalev2011pegasos,shalev2013accelerated, lacoste2013block} or the cutting plane method\cite{tsochantaridis2004support,joachims2009cutting}. In both optimization methods, the most important step is {\em loss augmented inference}, which amounts to finding the argmax label in \eqref{eq:max_surrogate} which attains the loss.
%This step is in the center of our problem and we will discuss in much detail in the  section \ref{sec:computational}. 
While our approach can be used in both approaches, we focus on using the stochastic first order method which is more suitable for the large-scale problem with a better convergence guarantee compared to cutting plane methods.
 %Thus surrogate loss $\tilde{L}_w$ is chosen among the upper bounds surrogate label loss $\dot{L}$ (or task loss $L$) that (1)  easy to optimize, (2) easy to infer, (3) small gap between    task loss $L$, and (4)\ small generalization gap $\Lambda$.
 %Structured risk minimization approach minimizes \eqref{eq:generalization_bound} to obtain minimum surrogate risk. 

\section{Structured SVM}

%\subsection{Structured SVM}
%\subsection{Objective function and loss functions}
  
 We give an overview of our primary approach, Structured SVM. This section is mainly from 
\cite{choi2016fast}. Structured SVM is an elegant max-margin approach which uses a structured hinge loss surrogate \cite{tsochantaridis2004support,Taskar03}.

Structured SVM in constrained from is as follow:
 \begin{align}\label{obj_structured_svm_unconstrained}
        &\min_{w,\xi} \dfrac{C}{2}\|w\|^2_2 +\dfrac{1}{n}\sum_i\xi_i \\
    \text{s.t.} \;\;& \xi_{i}\ge \tilde{\Phi}(w,y,y_{i},x_{i}) &\forall i,y\ne y_i \nonumber\\
     & \xi_i\ge 0 &\forall i\nonumber
 \end{align}
 where $C$ is the regularization constant. 

Two popular surrogates are margin and slack rescaling.
\begin{align*}
\text{Margin Rescaling : } & \tilde{\Phi}(w,y,y_{i},x_{i})=L(y,y_i)+f(x_i,y)-f(x_i,y_i)\\
\text{Slack Rescaling : } & \tilde{\Phi}(w,y,y_{i},x_{i})=L(y,y_i)(1+f(x_i,y)-f(x_i,y_i))
\end{align*}
We dropped the dependence on $w$ when it is clear from the context. In a constrained optimization form, the margin rescaling training  objective is given by:
 \begin{align}\label{obj_margin}
        &\min_{w,\xi} \dfrac{C}{2}\|w\|^2_2 +\dfrac{1}{n}\sum_i\xi_i \\
    \text{s.t.} \;\;& f(x_i,y_i)-f(x_i,y)\ge L(y,y_i)-\xi_i &\forall i,y\ne y_i \nonumber\\
     & \xi_i\ge 0 &\forall i\nonumber
 \end{align}
 
 Similarly, the slack rescaling formulation scales the slack variables by the error term:
 \begin{align}\label{obj_slack}
        &\min_{w,\xi}\ \dfrac{C}{2}\|w\|^2_2 +\dfrac{1}{n}\sum_i\xi_i \\
    \text{s.t.} \;\;& f(x_i,y_i)-f(x_i,y)\ge 1- \dfrac{\xi_i}{L(y,y_i)} &\forall i,y\ne y_i \nonumber\\
     & \xi_i\ge 0 &\forall i \nonumber
 \end{align}

Intuitively, both formulations seek to find a $w$ which assigns high scores to the ground-truth compared to the other possible labelings. When $y$ is very different than the true $y_i$ ($L$ is large) then the difference in scores should also be larger.
There is, however, an important difference between the two forms. In margin rescaling, a high loss can occur for labelings with the high error even though they are already classified correctly with a margin. This may divert training from the interesting labelings where the classifier errors, especially when $L$ can take large values, as in the common case of hamming error. In contrast, in slack rescaling labelings that are classified with a margin incur no loss.
Another difference between the two formulations is that the slack rescaling loss is invariant to scaling of the error term, while in margin rescaling such scaling changes the meaning of the features $\phi$.

In many cases, it is easier to optimize an unconstrained problem:
{ \begin{align} %\textstyle
\label{unconstrained_obj_margin}
        &\min_{w} \tfrac{C}{2}\|w\|^2_2
        +\!\tfrac{1}{n}\!\sum_i\max_{y\in\mathcal{Y}} \left( L(y,y_i)+f(x_i,y)-f(x_i,y_i)\right) \\
\label{unconstrained_obj_slack}
        &\min_{w} \tfrac{C}{2}\|w\|^2_2 +\!\tfrac{1}{n}\!\sum_i
        \max_{y\in\mathcal{Y}} L(y,y_i)\left(1+f(x_i,y)-f(x_i,y_i)\right) 
\end{align}}

Most of the existing training algorithms for structured SVM require solving the maximization-over-labellings problems for fixed $i$ and $w$ in \eqref{unconstrained_obj_slack} and \eqref{unconstrained_obj_margin}:
\begin{align*}
\begin{array}{c}
\text{Margin}\\
\text{rescaling}
\end{array}:&~\argmax_{y\in\mathcal{Y}} L(y,y_i)+f(x_i,y)-f(x_i,y_i) \\
\begin{array}{c}
\text{Slack}\\
\text{rescaling}
\end{array}:&~\argmax_{y\in\mathcal{Y}} L(y,y_i)\left (1+f(x_i,y)-f(x_i,y_i)\right )
\end{align*}

\begin{comment}

It is useful to consider the objective with respect to two functions, margin loss $m(y)=f(x_i,y)-f(x_i,y_i)$ and structural loss $L(y,y_i)$,
\begin{align}
\label{eq:margin_argmax}
\begin{array}{c}
\text{Margin}\\
\text{rescaling}
\end{array}:&~\argmax_{y\in\mathcal{Y}} L(y,y_i)+m(y) \\
\label{eq:slack_argmax}
\begin{array}{c}
\text{Slack}\\
\text{rescaling}
\end{array}:&~\argmax_{y\in\mathcal{Y}} L(y,y_i)\left (1+m(y))\right )
\end{align}

 Structural loss can can be any loss function over two labels. Usually, the task loss and the structural loss is the same, and only different in some cases. For instance, 
\end{comment}
 As we mentioned previously, for the optimization,  \eqref{eq:margin_argmax} and \eqref{eq:slack_argmax} require a search over the exponential label space $\mathcal{Y}$, which is termed the loss augmented inference. It is often very computationally intensive and time-consuming. Therefore, efficiency is essential, which is an important topic of the dissertation. In the later chapter, we first start with comparing two surrogate losses which serve as an intuition for possible improvement.

\section{Notation on Surrogate Losses}\label{sec:surrogate_loss_notation}

In this section, we want to clarify the notations introduced for the surrogate losses under different context where dropping the irrelevant dependencies   is helpful for the clarity.    

   \begin{align*} %\textstyle
\tilde{L}(w,x_i,y_i)=\max_y\tilde{\Phi}(w,y,y_{i},x_{i})=\max_y \Phi(y)=\max_y\psi(g(y),h(y))
\end{align*}
 $\tilde{L}(w,x_i,y_i)$ is used to denote the general surrogate loss in the literature  in the objective function \eqref{eq:objective_function_general}
 and the generalization bound \eqref{eq:generalization_bound}. $\max_y\tilde{\Phi}(w,y,y_{i},x_{i})$ refers to the form of surrogate losses used in the structured SVM, which is the only surrogate loss of interest in this dissertation. 

However, when discussing solving above max over $y$, which is called loss augmented inference,  we consider $x_i, y_i,$ and $w$ fixed and  use notation $\Phi$ for  the function only depends on $y$. 

Later on, we show that it is useful to consider the surrogate loss as a bivariate function $\psi$ taking two factor functions $h(y)$ and $g(y)$. For instance, for margin rescaling, $\psi(a,b)=a+b$, $h(y)=m(y)$,  and $g(y)=L(y,y_i)$, and for slack rescaling $\psi(a,b)=a(1+b$), $h(y)=m(y)$,  and $g(y)=L(y,y_i)$.

\title{Efficient Structured Learning with Bi-criterian Surrogate Loss}
\author{%
  Heejin Choi\footremember{alley}{TTIC, heejincs@ttic.edu}%
  \and  Nathan Srebro\footremember{trailer}{TTIC, nati@ttic.edu}%
  }

\date{}
\maketitle

\begin{abstract}
 In structured prediction, structural loss, the cost of misclassification,  varies respect to the label wrongly assigned. This makes the loss in max-margin framework bivariate function of 
structural loss and margin. The bivariate function is called  the surrogate loss, and it expresses the loss or a preference of a label. Since the surrogate loss defines the  objective function, the choice of surrogate loss is important. However, the choice of the the surrogate loss is limited to be mostly a simple addition of the two since only the additive surrogate loss benefits from the decomposability, and makes loss augmented inference efficient. In this paper, we aim to loosen this restriction by providing an efficient practical inference method for a wide classes of the bivariate surrogate loss. Specifically, the inference method is applicable for  any quasi-concave function, which  is considered a natural property of a function that expresses the preference, given inference method for an additive surrogate loss, .   
This enables optimizing a direct upper bound of the task loss and a novel quasi-convex probabilistic loss function that expresses complicated relationship between structural loss and  margin. 
We show that the inference method is exact with relaxation and the algorithm is optimal in runtime within the framework. We empirically validate our method showing the performance improvement in real-world problem using bi-criterian surrogate losses. 

 a surrogate loss to direct upper bound  Empirically, we show that this approach of directly optimizing the micro-f1 can be deployed to existing complex systems easily providing improvement over simple surrogate losses. As for the efficiency,  our result includes surprising result that slack rescaling inference (which is considered much harder problem) is done almost as same runtime as margin rescaling. Specifically,  only 1.3 to 1.0048 calls to the margin rescaling oracle is neeed to the slack rescaling inference in average.
Margin rescaling (or additive loss) has been extensively used in structured prediction due to its computational efficiency which is neccesity due deal with the combinatorially large output space despite its problem that even though the label is well separated, label with a large label cost induces error in the objective function.  Slack rescaling (or multiplicated loss) has gained a focus as an alternative solution, but its practical usage has been prohibited due to the lack of an efficient and practical algorithm  and  the guarantee of the performance benefit. In this paper, we address both problems. We establish a tighter generalization bound by presenting a  new loss generalizing the two.   While the loss is much more complicated,  we present an exact and efficient method for the inference in convex relaxed setting. This method can be applied to any problem where margin rescaling inference is applicable, such as problems using graph cut. Also,  the optimality of a candidate label can be tested with one oracle call, which makes the method very efficient   in settings such as iterative optimization where the optimal label changes smoothly. We verify the improvement empirically in various interesting problems. Its practicability is  highlighted by  the average number of oracle calls which is less than  2 in some cases that the method that the run time is close to that of margin rescaling.   

\end{abstract}
\end{comment}
\begin{comment}
   In this chapter, we focus on the cost-sensitive surrogate losses in structured prediction. We first review common surrogate losses and its properties. Then, we show its connection with the one of the most important issue in the structured prediction: hardness of the loss augmented inference. Finally, from the novel view of structured surrogate loss as a bi-criteria function, we present {\em ProbLoss}, a loss function directly derived from the target cost-sensitive loss.
\end{comment}
%\section{Problem Setting}

%In the empirical risk minimization framework, often we use the surrogate loss that is easily optimizable rather than the actual performance measure of interest, which is called task loss.  
In Structured SVM\cite{tsochantaridis2004support}, for each label, there is a margin  which describes how far the label is from the decision boundary. And the loss for each instance is given as the maximum value of the bivariate function of the structural loss and the margin over all the labels. The bivariate function is called  surrogate loss. Surrogate loss defines the objective function, and it is important for the performance.
We first visit two common surrogate losses in the literature: margin rescaling and slack rescaling.
We review pros and cons of the approach.
%From the generalization bound, we show that it is plausible that the surrogate loss between the two could attain the minimum generalization bound, and suggest new surrogate loss termed beta rescaling.

Another critical issue of the surrogate loss is that the loss augmented inference, the argmax operation over the label inside of the surrogate loss, requires to be efficient. Often it is the most challenging step in the optimization since the label space is often   exponentially large in the size of the labels in structured prediction. This restricts complexity of the surrogate loss. Therefore, the surrogate loss need be complicated to guarantee a low test loss, but at the same time, it needs to be simple enough that the argmax operation is efficient. We visit this issue with respect to the two surrogate losses.

To further investigate  the surrogate losses, we observe that the two common surrogate losses can be viewed as a bi-criteria function of margin loss and structural loss. For instance, the margin rescaling and slack rescaling are connected to the simplest bi-criteria: an arithmetic mean with an addition and a geometric mean with a multiplication. The bi-criteria function (or negative utility function) is often studied in economics, and  quasi-concavity is considered to be a natural property of negative utility function. We accept this view, and present a new group of surrogate loss called {\em bi-criteria surrogate loss}. 
  We propose two novel surrogate losses as examples: ProbLoss that directly models the risk under the normal noise assumption and Micro-F1 surrogate that directly upper bounds Micro-F1.

\section{Two Structured Surrogate Losses}
As in \eqref{eq:objective_function_general}, surrogate loss  $\tilde{L}(w,x_i,y_i)$ is used as a proxy to ease the optimization of task loss $L$%(\hat{y}_w,y_i)$
, the loss or the performance measure of interest. Two common convex surrogate losses in the literature\cite{tsochantaridis2004support,Taskar03,NIPS2016_6485} are margin rescaling (MR) and slack rescaling (SR). Both are tight upper bound of the task loss.% $L(\hat{y}_w,y_i)$.
\begin{comment}
Then, define two loss augmented potential functions as,
\begin{align}
\Phi^{MR}_i(y)&=
\dot{L}_{i}(y)+m_i(y) 
\label{eq:mr_potential}
\\
\Phi^{SR}_i(y)&=
\dot{L}_{i}(y)(m_i(y)+1) 
\label{eq:sr_potential}
\end{align}
\end{comment}

\begin{comment}
The surrogate losses $\tilde{L}(w,x_i,y_i)$ are,
\begin{align*}
%\text{Margin rescaling loss } 
%\left(
\text{Margin rescaling}:&& 
%\right)
&   
%\max_y&L(y,y_i)+f(x,y)-f(x,y_i)
%&=
\max_y  \tilde{\Phi}^{MR}(w,y,y_i,x_i)=\max_{y}
L(y,y_i)+\tilde{m}(w,y,y_i,x_i)
\\
%\text{Slack rescaling loss } \left(
\text{Slack rescaling}:&&
%\right)
&   
%\max_y&L(y,y_i)\left(f(x,y)-f(x,y_i) +1\right)
%&=
\max_y \tilde{\Phi}^{SR}(w,y,y_i,x_i)=\max_{y}L(y,y_i)(\tilde{m}(w,y,y_i,x_i)+1)
\end{align*} 
where $\tilde{m}(w,x,y,y_i)=f_{w}(y)-f_{w}(y_i)$ is the margin error.
\end{comment}
 As discussed in section \ref{sec:surrogate_loss_notation}, for the notational convenience, and to better emphasize the difference between margin and slack rescaling,
we focus on a single training instance $i$ and drop the dependence on $i$ and $w$, and  analyze the surrogate losses with respect to  function $\Phi$. %$m(y)=\tilde{m}(w,y,y_i,x_i)$,  $\Phi^{MR}(y)=\tilde{\Phi}^{MR}(w,y,y_i,x_i)$, and  $\Phi^{SR}(y)=\tilde{\Phi}^{SR}(w,y,y_i,x_i)$. 
   \begin{align*} %\textstyle
\tilde{L}(w,x_i,y_i)=\max_y \Phi(y)
\end{align*}

Then, the two common surrogate losses $\tilde{L}(w,x_i,y_i)$ are,
\begin{align}
%\text{Margin rescaling loss } 
%\left(
\text{Margin rescaling}:&& 
%\right)
&   
%\max_y&L(y,y_i)+f(x,y)-f(x,y_i)
%&=
\max_y  \Phi^{MR}(y)=\max_{y}
L(y,y_i)+m(y)
\label{eq:mr_loss}
\\
%\text{Slack rescaling loss } \left(
\text{Slack rescaling}:&&
%\right)
&   
%\max_y&L(y,y_i)\left(f(x,y)-f(x,y_i) +1\right)
%&=
\max_y \Phi^{SR}(y)=\max_{y}L(y,y_i)(m(y)+1)
\label{eq:sr_loss}
\end{align} 
where $m(y)=f_{w}(x_{i},y)-f_{w}(x_i,y_i)$ is the margin error.

 Slack rescaling (or multiplicative loss) \eqref{eq:sr_loss}  first appeared in \cite{tsochantaridis2004support} as an alternative to margin rescaling (or additive loss) \eqref{eq:mr_loss}. 
First define the functions:
$h(y)=m(y)$ and $g(y)=L(y,y_i)$. With these definitions, we
see that the maximization \eqref{eq:mr_loss} for margin
rescaling is $\max_{y\in\mathcal{Y}} h(y)+g(y)$, while the
maximization \eqref{eq:sr_loss} for slack rescaling is
$\max_{y\in\mathcal{Y}} g(y)(h(y)+1)$. Immediately, we can see that margin rescaling is the maximum over sum of two scores, and  slack rescaling is the multiplication of the two. We analyze this with two aspects: Required margin for zero error and sensitivity to the outlier.
\begin{align*}
 \mbox{Required margin for zero error }: M^{req}(y)= & \mbox{ min } |m(y)| \mbox{ s.t } \Phi(y)=0 \\
 \mbox{Sensitivity to the outlier }: \delta^{outlier}(y)= &\left. \dfrac{\partial \Phi(y)}{\partial m(y)} \right \rvert_{m(y)>0}    \end{align*}

  \begin{table}[h]
  \begin{center}
    \begin{tabular}{|l|c|c|}
    \hline
     & $M^{req}(y)$ & $\delta^{outlier}(y)$ \\
    \hline

Margin Rescaling &   $L(y,y_{i})$ & 1 \\
Slack Rescaling &   $1$ & $L(y,y_{i})$\\    
    \hline

    \end{tabular}\\
  \end{center}
\caption{Difference between margin rescaling and slack rescaling.  %$\beta$-scaling, Micro-F1 surrogate, Loss scaled log loss and ProbLoss are new surrogate losses.
} \label{table:difference_mr_sr} 
\end{table}
 The differences are summarized in Table \ref{table:difference_mr_sr}. The problem of margin rescaling is that it requires a large margin  even if the label is well separated. For a problem with a large task loss, this could be a problem. 

For instance, consider two labels $y_1$, $y_2$ with $m(y_1)=-10,$  $m(y_2)=1$, $L(y_1,y_i)=100$,  $L(y_2,y_i)=2$. Then, $\Phi^{MR}(y_1)=90>\Phi^{MR}(y_2)=3$. The loss is defined by $y_1$ which is on the correct side of the classification boundary. However, $y_2$ is on the wrong side of the boundary.
 
 The problem of the above example is that the additive property does not take the classification boundary into the account, the sign of the margin. That is, the overall loss can be significant even if the margin is negative.
 
 One way to resolve this issue is to use the slack rescaling surrogate loss \eqref{eq:sr_loss}. In the previous example, it is easy to  see that $\Phi^{SR}(y_1)=-900<\Phi^{SR}(y_2)=4$.  In slack rescaling formulation, if the margin error is smaller than $-1$, it indicates that the label is well-separated by the unit margin, and the potential is always negative. 
 
 However, the con of the slack rescaling is that the sensitivity of the outlier  $\delta^{outlier}(y)$ is large. This can be a problem for inseparable problems. If there exists an outlier with large  $\delta^{outlier}(y)$, the error is multiplied by $\delta^{outlier}(y)$ in the objective function. This also can be a problem with a large structured loss. For instance, if the loss is as large as $100$, its error will be multiplied by $100$. Furthermore, if there is an outlier or instances that cannot fit into the model, the phenomenon will be much intensified.

\section{Decomposability and Hardness of Inference }\label{sec:decomposability}
We will describe the other important difference between the two surrogate loss: hardness of the inference. We will also point out the significant merit of the margin rescaling that its additive property preserves the substructure, and argmax becomes one operation over the structure. This contrasts to the  slack rescaling inference, which involves interaction between two 
different scores over the labels, thus its inference is much complicated.

Due to the combinatorial size of the label space $\mathcal{Y}$,  often the loss augmented inference is hard and the most time consuming step, and the optimization requires the efficient method for the loss augmented inference.   Decomposability into substructure is one of the common properties that enables the efficient inference. This is when both surrogate task loss $L(y)$ and $m(y)$ decompose over the same substructure (or micro-labels) $S_k\subseteq[m]$, $k\in[K]$, $\bigcup_{k=1}^KS_k=[m]$  that $L(y,y_{i})=\sum_{k}L_{S_k}(y_{S_k})$ and $m(y)=\sum_{k}m_{S_k}(y_{S_k})$. Then, the margin rescaling potential function in \eqref{eq:mr_loss}  also decomposes as a sum over the substructure, $\Phi^{MR}(y)=\sum_{k}L_{S_k}(y_{S_k})+ m_{S_k}(y_{S_k})$.   This substructure can be utilized for efficient inference with many algorithms such as dynamic programming, linear programming, graph-cut, and etc. This results one argmax operation over the label set $\mathcal{Y}$ of function $m'_{S_k}(y_{S_k})=L_{S_k}(y_{S_k})+m_{S_k}(y_{S_k})$. However, the decomposability is maintained only for the margin rescaling. In slack rescaling \eqref{eq:sr_loss}, the potential function does not decompose over the substructures. That is, the interaction between the  task loss and the margin is global, i.e. argmax is attained when both losses are high. 

This also accounts for the popularity of the margin rescaling. In practice, optimizing with respect to the margin rescaling formulation is commonly considered  as the only option and slack rescaling formulation is not considered possible.

To sum up,  when the loss $L$ decompose with respect to the substructures or micro-labels, the inference of the margin rescaling also benefits from the decomposability. The problem of margin rescaling is that to have a zero error, the negative margin has to grow linearly respect to the error, which can be too large.  On the other side, in slack rescaling, the labels only need to be separated by the unit margin or $1$. However, the difficulty of slack rescaling lies in its inferences that multiplicative nature disallows decomposition.

\section{Bi-criteria Surrogate Loss } \label{sec:quasi_convex_surrogate}

Thus far,  we encountered the several problems. To maintain the efficiency of decomposability, 1) Task loss has to decompose respect to the substructures. This prohibits directly optimizing the task loss  that does not decompose. In such a case, rather than optimizing the task loss, a different but highly correlated decomposable loss is used in the objective function. For instance, Hamming loss is used instead of Micro-F1. 2) The choice of surrogate loss is limited to margin rescaling.

 Now we remove both restrictions. Specifically,  we introduce a new class of surrogate losses which later we show that it can be efficiently inferred. For instance, our algorithm extends to efficient inference for  micro-F1 and slack rescaling.

%In this section, we introduce new surrogate loss of interest, {\em quasi-convex structural surrogatse loss}.   
% We first denote $h(y)=m(y)$, and $g(y)=L(y)$.

First, we note that both surrogate losses discussed  can be written as one form  using a different bi-criteria function  $\psi:\mathbb{R}\times \mathbb{R}\mapsto \mathbb{R}$ which represents overall loss given two losses, i.e., 
\begin{align*}
\Phi(y)=\psi(m(y),L(y,y_{i}))
\end{align*}
For  margin rescaling, $\psi(a,b)=a+b$, and for slack rescaling, $\psi(a,b)=(a+1)\cdot  b$.

We generalize the surrogate loss to other function $\psi$. Specifically, we consider the surrogate loss of the following form.  
\begin{align*}
\tilde{L}(w,x_i,y_i)=\max_y \Phi(y)=\max_y\psi(h(y),g(y))
%\label{eq:surrogate_loss}
\end{align*}
The bi-criteria function takes two losses as an input, for instance, a margin loss and a structural loss, and outputs overall loss.  For instance, for  margin rescaling, $\psi$ is an addition, and for  slack rescaling, $\psi$ is a multiplication of two losses.  These are connected to the simplest bi-criteria, the arithmetic mean and the geometric mean. 
  This bi-criteria function expresses the negative preference of the combination of two variables. In economics, such functions that express the preference are extensively studied, and they are called the utility functions. One central assumption of the utility function is quasi-convexity. Therefore, we adopt this view of the utility function and consider quasi-concavity functions as a natural class of bi-criteria function of surrogate losses since bi-criteria is the negative utility.  On the other hand, in the later chapter, we show that quasi-concavity is an essential property  for the efficiency. 

We formally define the surrogate loss function of interest.

\begin{comment}
\begin{definition}[Bi-criteria Surrogate Loss]
    Consider a surrogate loss $\tilde{L}(w,x_i,y_i)$ with a following form,
\begin{align*}
\tilde{L}(w,x_i,y_i)=\max_{y\in\mathcal{Y}} \Phi(y)
\end{align*}
where $\Phi:\mathcal{Y}\mapsto \mathbb{R}$ is a potential function. 

Then,  
$\tilde{L}(w,x_i,y_i)$ is a  bi-criteria surrogate loss if following holds true for a   bi-criteria function $\psi:\mathbb{R}\times \mathbb{R}_+\mapsto \mathbb{R}$.
\begin{enumerate}
\item
 $\Phi(y)= \psi(m(y), L(y))$ 
\item 
 $\psi$ is a differentiable function.

\item  Let $K_\alpha=\{(a,b)\mid a\in\mathbb{R}, b\in\mathbb{R}_+,\psi(a,b)\ge \alpha\}$ be the $\alpha$ super level
set of $\psi$, and $\beta=\psi(0,0)$. $\psi$ is a quasi-concave function  in the domain $K_{\beta}$, i.e., $K_\beta$ is convex.   
\item  
   $\psi$ is a monotonically increasing function in both arguments that  for $\forall a\in \mathbb{R}, \forall b\in \mathbb{R}_+,\forall \epsilon>0,$ $\psi(a+\epsilon,b+\epsilon)>\psi(a,b)$.
\end{enumerate}
\end{definition}

 Note that for bi-criteria surrogate loss, the arguments for the bi-criteria function $\psi$ are fixed to $m(y)$ and $L(y)$. We can generalize bi-criteria surrogate loss that bi-criteria function can take other loss functions. 
\end{comment}

\begin{definition}[Bi-criteria Surrogate Loss] \label{def:gbicriteria}
    Consider a surrogate loss $\tilde{L}_i(w,x_i,y_i)$ with a following form,
\begin{align*}
\tilde{L}_{i}(w,x_i,y_i)=\max_{y\in\mathcal{Y}} \Phi_{i}(y)
\end{align*}
where $\Phi_i:\mathcal{Y}\mapsto \mathbb{R}$ is a potential function. 

Then,  
$\tilde{L}_{i}(w,x_i,y_i)$ is a   bi-criteria surrogate loss if following holds true for a bi-criteria functions $\psi_i:\mathbb{R}\times \mathbb{R}_\mapsto \mathbb{R}$,  two loss factor functions $h_i:\mathcal{Y}\mapsto \mathbb{R}$ and $g_i:\mathcal{Y}\mapsto \mathbb{R}$, 
 
\begin{enumerate}
\item
 $\Phi_{i}(y)= \psi_i(h_i(y), g_i(y))$ 
\item In domain $K_0$.
\begin{enumerate}
\item 
 $\psi_{i}$ is a differentiable function.
\item   $\psi_i$ is a quasi-concave function, i.e., $\forall\beta\ge0,  K_\beta$ is convex. 

\item  
   $\psi_i$ is a monotonically  increasing function for each arguments, and $\forall( a, b)\in K_0,\forall \epsilon>0,$  $\psi_{i}(a+\epsilon,b+\epsilon)>\psi_{i}(a,b)$.
\end{enumerate}
\item Function $h_i$ and $g_i$ decomposes into substructures that $\forall \lambda\ge0$, $\argmax_y \lambda g_i(y)+h_i(y)$ is efficiently solvable. 
%$\exists \rho:\mathbb{R_+}\mapsto\mathbb{R}_+$ such that $\forall \lambda\ge0$,  $\argmax_{y\in\mathcal{Y}} \lambda g(y)+ h(y)$=$\argmax_{y\in\mathcal{Y}} \rho(\lambda)L(y)+m(y)$.
 \end{enumerate}
where    $K_\alpha=\{(a,b)\mid a\in\mathbb{R}, b\in\mathbb{R}, \psi_i(a,b)\ge \alpha\}$ is $\alpha$ super level
set. 
\end{definition}

 Note that bi-criteria surrogate losses extend the bi-criteria function $\psi$ to a wide class of functions which was restricted to only addition or multiplication.  
%  A generalized bi-criteria surrogate loss, also extends $\psi$ function to take different arguments than $L(y)$ and $m(y)$.
Last property in the definition \ref{def:gbicriteria}  limits the choice of factor function $h_i$ and $g_i$ that  argmax operation of the  $h_i(y)+\lambda g_i(y)$  is efficient. This is very similar to the margin rescaling inference, $\argmax_{y\in\mathcal{Y}}  m(y)+ L(y,y_{i})$, with the choice of $h_i(y)=m(y)$ and $g_i(y)=L(y,y_{i})$,  but has an additional weighting term $\lambda$. $\argmax_{y\in\mathcal{Y}}  m(y)+ \lambda L(y,y_{i})$ is  called  $\lambda$-oracle, and will be extensively used in later chapters. $\lambda$-oracle is often efficient when margin rescaling inference is efficient. Therefore, for bi-criteria surrogate loss the choice of $h_i$ and $g_i$ is free with one condition that  the corresponding  $\lambda$-oracle with $h_i$ and $g_i$ is efficient. Usually $h_i(y)=m(y)$ and $g_i(y)=L(y,y_{i})$, however we show an example using other choices. 

\subsection{New Surrogate Losses} 
  Since our framework includes a variety of surrogate losses, we give some examples of useful surrogate losses. Such specific surrogate losses are summarized in Table \ref{table:surrogate}.
 Generalized scaling and $\beta$-scaling generalizes margin rescaling  (when $\alpha=1,\beta=0$) and slack rescaling (when $\alpha=1,\beta=1$) where $\alpha\le-1$ or $\alpha\ge0$,  $0\le \beta-\alpha\le 1$. The bound of $\alpha$ and $\beta$ is the sufficient condition for $\psi$ to be quasi-concave.  
Generalized scaling and $\beta$-scaling can control the size of the hypothesis class by controlling $\beta$ which directly affects the generalization term in \cite{NIPS2016_6485}. This can lead to a tighter generalization bound than that of \cite{NIPS2016_6485}.
Loss scaled log loss is a cost sensitive version of the log loss. 

  \begin{table}[H]
  \begin{center}
    \begin{tabular}{|l|c|}
    \hline
    Name & $\psi(h,g)$\\
    \hline

Margin Rescaling &   $h+g$ \\
Slack Rescaling &   $(h+1)g$ \\    
    \hline
Generalized scaling &   $hg^{\beta}+g^\alpha$ \\
$\beta$-scaling &   $hg^{\beta}+g$ \\
Loss Scaled Log Loss &   $g\log(\exp h +1)$\\ 
ProbLoss &   $2g \mathcal{N}_{CDF}(h,0,2g/\pi)$ \\    
\hline
Micro-F1 Surrogate   &   $h/(-g)$\ \\

    \hline

    \end{tabular}\\
  \end{center}
\caption{Bi-criteria surrogate losses: $h=m(y)$ and $g=L(y,y_{i})$ except Micro-F1 surrogate loss.  %$\beta$-scaling, Micro-F1 surrogate, Loss scalied log loss and ProbLoss are new surrogate losses.
} \label{table:surrogate} 
\end{table}

\subsection{ProbLoss} \label{sec:probloss}
 Margin rescaling has been widely used with success, however, it has been noted that it requires a high separating margin for the label with a high structural loss even if the label is well separated. To resolve this issue, we propose a new loss function that directly minimizes risk under a noise assumptions.   Let each label $y$ be a sequence of micro-labels of length $R$, and each micro-label has $M$ states, i.e., $y\in\mathcal{Y}=[M]^R$.
 Let $H(y,y_i)$ be  Hamming loss, which is the number of wrong micro-labels, i.e., $H(y,y_{i})=\sum_{r=1}^{R} {\bf 1}[y_r\neq [y_i]_r]$. We will only consider Hamming Loss for ProbLoss. The main idea is that rather than imposing  each micro-classifier to have  a constant safe margin of $1$ (which is equivalent  having a margin of $H(y,y_i)$ for the entire classifier), we impose a noise per micro-classifier, and consider the risk under the noise. A direct consequence of the approach is that the sensitivity (or the slope) of the error on the margin is a function of margin whereas it was constant $1$ for margin rescaling. See Figure \ref{fig:comparison_loss}. The surrogate loss works as an intermediate surrogate loss between margin rescaling and slack rescaling for Hamming loss. For instance, the maximum slope of the surrogate loss is $\sqrt{H(y,y_i)}$\footnote{ Differentiate \eqref{eq:Probloss} with $x$, and we can see that it is maximized at $x=0$.} whereas it is $1$ for margin rescaling, and $H(y,y_i)$ for slack rescaling. Another property is that the loss is never zero, thus solves the problem of returning no violating label discussed in \cite{belanger2016structured}. Thus, this loss can suitable for the deep networks since deep networks are  powerful models that can benefit from further enhancing the margin  larger than  $H(y,y_i)$. In margin rescaling, if all labels reaches margin of $H(y,y_i)$, it returns no violating label (or a zero gradient), and further optimization is not possible. 

Consider  the usual potential of a label $\phi$ that it is  a sum of potentials of  micro-classifiers $\phi_r$, and the margin is the difference of the potential between the predicted label and the correct label, i.e., $m(y)=\sum_{r=1}^{R} \phi_{r}(y_r)-\phi_{r}([y_i]_r)$ where $\phi_r(y_{r})$ is the potential for micro-label $y_{r}$ in a position $r$. Without the structured setting, the potential can be viewed as a sum of potential of $H(y,y_i)$ (not $R$) micro-classifiers due to the cancelation. Assume each micro-classifier has a Gaussian noise $\epsilon_{r}\sim\mathcal{N}(0,2/\pi)$. Let $h(y)=m(y)$ and $g(y)=H(y,y_i)$. Then, we use the loss for $y$ as its risk multiplied by a constant of two,  
\begin{align}
\psi(h(y),g(y))= 2H(y,y_{i})\mathbb{P}(m(y)\ge0) =2H(y,y_{i}) \mathcal{N}_{CDF}(m(y),0,2H(y,y_{i})/\pi)
\label{eq:Probloss}
\end{align}
where $ \mathcal{N}_{CDF}(x,\mu,\sigma^2)$ is a cumulative distribution function for a Gaussian distribution at $x$ with a mean $\mu$, and variance $\sigma^2.$  $2/\pi$ is used to set the slope to be $1$ when $g(y)=1$, which can be a tuning parameter. Multiplication of constant two makes ProbLoss a tight upper bound of structured $0-1$ loss, and matches other loss functions. Comparison with other losses are shown in Figure \ref{fig:comparison_loss}.

\begin{figure}[H]
\centering     
%\begin{subfigure}[b]{\linewidth}
\includegraphics[width=.7\linewidth]{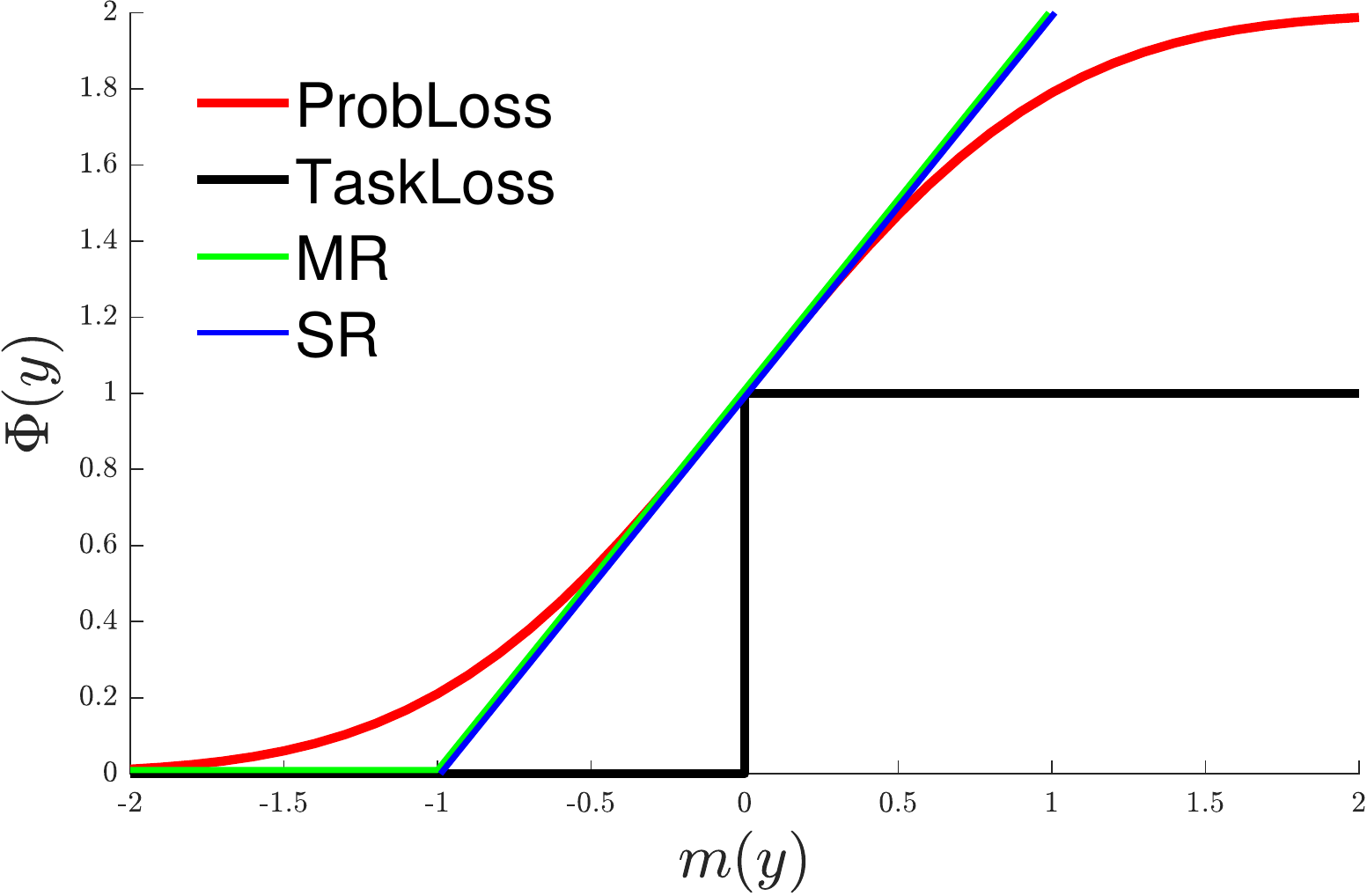}
\caption{$H(y,y_i)=1$}%\label{sfig:chs1}
\end{figure} 
%\end{subfigure}
%\begin{subfigure}[b]{\linewidth}
%\begin{figure}[H]
%\centering     
\begin{figure}[H]
\centering   
\includegraphics[width=.7\linewidth]{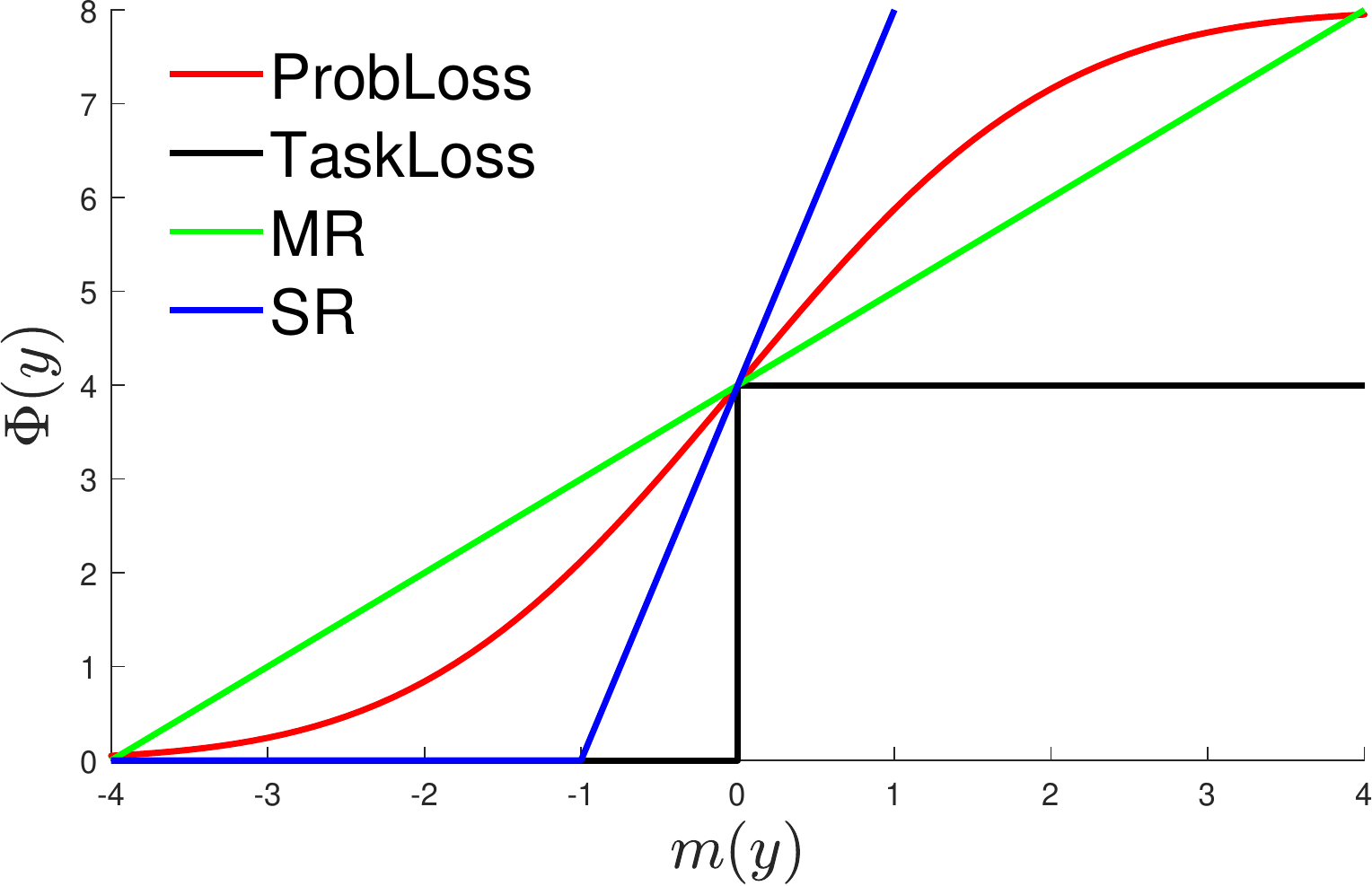}
\caption{$H(y,y_i)=4$}%\label{sfig:chs2}
%\end{subfigure}
\caption{Comparison of the Surrogate Losses.}\label{fig:comparison_loss}
\end{figure} 

We also propose to use a convex extension of the Proloss.
\begin{align*}
\psi(h(y),g(y))= \begin{cases}
2H(y,y_{i}) \mathcal{N}_{CDF}(m(y),0,2H(y,y_{i})/\pi) & \text{if } m(y)\le 0 \\
H(y,y_i)+\sqrt{H(y,y_i)}m(y) & \text{otherwise}
\end{cases}
\end{align*}

\subsection{Micro-F1 Surrogate} \label{sec:microf1}

We present our surrogate loss for the direct minimization of the instance-based Micro-F1 as  an example of a  bi-criteria surrogate loss.  
Micro-F1 is an average score of Micro-F1 score per instance. Since it is  a sum of losses of the instances (which is not the case for Macro-F1), the measure is suitable for the optimization method in current literature for a large scaled problems such as SGD, which leads to  efficient optimizations.

We consider a  multi-label problem of $M$ labels, $\mathcal{Y}\subseteq[M]$. For true label $y_{i}$, and prediction $y$, see that,
\begin{align*}
\text{Precision}(y,y_{i})&=\dfrac{|y\cap y_{i}|}{|y|}\\
\text{Recall}(y,y_{i})&=\dfrac{|y\cap y_{i}|}{|y_{i}|}
\\
\text{Micro-F1}(y,y_{i})&=1/\left(1/\text{Precision}(y,y_{i})+1/\text{Recall}(y,y_{i})\right)=\dfrac{2|y\cap y_{i}|}{|y|+|y_{i}|}
\end{align*}
Note that $|y_i|$ is a constant. Maximizing Micro-F1 score is same as minimizing 1-(Micro-F1 score). 
Then, our surrogate loss function, Micro-F1 surrogate, is 
\begin{align*}
L^{F1}=\max_{y} \left(1-\dfrac{2|y\cap y_i|}{|y|+|y_i|}\right)+\dfrac{m(y)}{|y|+|y_i|}=\max_y \dfrac{H(y,y_i)+m(y)}{|y|+|y_i|}
\end{align*}
where $H(y,y_i)$ is the Hamming loss. 
  
 The loss function is a tight upper bound of the actual F1 loss since it equals to the $1-\text{Micro-F1}(y,y_{i})$ when $m(y)=0$.
 Micro-F1 surrogate is a  bi-criteria surrogate with choice of $\psi_i(a,b)=a/(-b),$ $h_{i}(y)=H(y,y_i)+m(y)$, and $g_{i}(y)=-|y|-|y_i|$. 
Since $|y|=\sum_{j\in[m]}{\bf 1}[j\in y]$, $h_{i}$ and $g_{i}$ decomposes with respect to each label. 
\section{Conclusion}
In this chapter,  we investigated the problem of the choice of surrogate losses and its properties in the cost-sensitive learning problem. We first reviewed two common surrogate losses: margin rescaling and slack rescaling. Two surrogate losses are compared with respect to two properties:  required margin and  sensitivity to outliers. As for the margin rescaling, the required margin increases linearly with the task loss, and in turn, this results already well-separated label in the correct side of the classification boundary to incur a larger potential than labels on the wrong side of the classification boundary. On the other hand, slack rescaling formulation only requires a unit margin of $1$. Thus, the   labels with a safe margin larger than  $1$ has the negative potential, and never be selected as the argmax label. The problem of slack rescaling is that margin loss is multiplied by the structural loss, and this could be potentially large for the label with a large structural loss, especially, for the inseparable datasets or the datasets with outliers. 
 
  We introduced an important issue of the surrogate loss: Hardness of the inference. For the efficient inference, decomposability plays an important role. While the additive property of the margin rescaling maintains the decomposability and results in an efficient inference, the multiplicative property of slack rescaling disallows decomposability, and which in turn disables efficient inference utilizing the decomposability.

We extended the surrogate loss from the simple additive and multiplicative form to other formulations termed  bi-criteria surrogate. Two examples are given: ProbLoss that upper bounds the risk under the normal noise assumption and Micro-F1 surrogate that directly related to optimizing the  Micro-F1 score.

\chapter{Efficient Slack Rescaling Inference with $\lambda$-oracle}\label{cha:inference}

In this chapter, we investigate an important task; efficient inference methods for the slack rescaling formulation before visiting other more complicated surrogate losses. We assume that the model is decomposable, and efficient inference of the margin rescaling formulation is available.  Then, we ask following the questions. If we consider margin rescaling inference as an oracle, can we find the optimum for slack rescaling exactly? What is the minimum number of oracle call?   We present several practical inference algorithms as well as theoretical lower bounds.  The most important insight comes from the geometric interpretation of labels into the 2D plane. 

This chapter is mainly based on the previous publication \cite{choi2016fast}.

% !TeX root = slack_rescaling.tex
\section{Problem Formulation}
Recall  our objective function in \eqref{eq:objective_function_struct_svm},
  \begin{align*} %\textstyle
        \min_{w} \frac{C}{2}\|w\|^2_2
        +\frac{1}{n}\!\sum_i\max_{y\in\mathcal{Y}} \tilde{\Phi}(w,y,y_{i},x_{i})
\end{align*}
 and margin rescaling and slack rescaling formulation in \eqref{eq:mr_loss} and \eqref{eq:sr_loss},
\begin{align*}
%\text{Margin rescaling loss } 
%\left(
\text{Margin rescaling}:& 
%\right)
&   
%\max_y&L(y,y_i)+f(x,y)-f(x,y_i)
%&=
\max_y  \tilde{\Phi}^{MR}(w,y,y_i,x_i)&=\max_{y}
L(y,y_i)+m(w,x_{i},y,y_i)
\\
%\text{Slack rescaling loss } \left(
\text{Slack rescaling}:&
%\right)
&   
%\max_y&L(y,y_i)\left(f(x,y)-f(x,y_i) +1\right)
%&=
\max_y \tilde{\Phi}^{SR}(w,y,y_i,x_i)&=\max_{y}L(y,y_i)(m(w,x_{i},y,y_i)+1)
\end{align*} 
 As before, define the functions:
$h(y)=1+m(w,x_{i},y,y_i)$ and $g(y)=L(y,y_i)$. 
Then, with these definitions, loss augmented inference or finding the argmax label is,
\begin{align}
%\text{Margin rescaling loss } 
%\left(
\text{Margin rescaling}:&& 
%\right)
%\max_y&L(y,y_i)+f(x,y)-f(x,y_i)
%&=
\argmax_y&g(y)+h(y)
\label{eq:mr_argmax}
\\
%\text{Slack rescaling loss } \left(
\text{Slack rescaling}:&&
y^{*}=\argmax_y &g(y)h(y)
\label{eq:sr_argmax}
\end{align} 

%where $m_w(y,y_i,x_i)=f_w(y,x_i)-f_w(y_i,x_i)$ is the margin error. 

%\subsection{$\lambda$-oracle}\label{ch:lambda_oracle}
In this chapter, we only consider the problem of solving the loss augmented inference in \eqref{eq:sr_argmax} efficiently. The efficiency is measured by the number of calls to the $\lambda$-oracle, which is closely related to loss augmented inference of margin rescaling.  
That is, we assume that we have access to a procedure that can efficiently solve the problem:
\begin{equation}\label{eq:lambda_oracle}
y(\lambda)=\mathcal{O}(\lambda)
%=\argmax_{y\in\mathcal{Y}} h(y)+\lambda g(y)\\ 
=\argmax_{y\in\mathcal{Y}}\mathcal{L}_\lambda(y)
\end{equation}
where $\mathcal{L}_\lambda(y)=h(y)+\lambda g(y)$.   This problem is
just a rescaling of argmax operation of margin rescaling formulation.  E.g., for linear
responses it is obtained by scaling the weight vector by $1/\lambda$.
If we can handle margin rescaling efficiently, we can most likely
implement the $\lambda$-oracle efficiently.  This is also the oracle
used in \cite{sarawagi2008accurate}.
%

%We will answer the questions in order presenting algorithms as we develop the understanding of this approach.

We consider $\lambda$-oracle as a basic inference operation and call the oracle different $\lambda$ to obtain $y^*$.

Exploring the limitation of this approach, we propose an alternative procedure that can access a more
powerful oracle, which we call the \emph{constrained $\lambda$-oracle}:
\begin{align}\label{eq:constrained_oracle}
%  \max_{y\in\mathcal{Y},\; A_1{\tiny\left(\substack{ h(y) \\ g(y)}\right)} < 0,A_2 {\tiny \left(\substack{ h(y) \\ g(y)}\right)}\le 0} h(y)+\lambda g(y), 
%  \max_{y\in\mathcal{Y},\; \alpha h(y)\ge g(y),\; \beta h(y)< g(y)}\right)}} h(y)+\lambda g(y), 
  \mathcal{O}_{c}(\lambda,\alpha,\beta)=\max_{y\in\mathcal{Y},\; \alpha h(y)> g(y),\; \beta h(y)\le g(y)} \mathcal{L}_\lambda(y),  
%\\  \underline{y}_{\lambda,\alpha,\beta}=\underline{\mathcal{O}}_{c}(\lambda,\alpha,\beta)=\max_{y\in\mathcal{Y} \alpha h(y)> g(y),\; \beta h(y)\le g(y)} h(y)+\lambda g(y)  
\end{align}
where $\alpha,\beta \in\reals$.
This oracle is similar to the $\lambda$-oracle, but can additionally handle linear constraints on the values $h(y)$ and $g(y)$. In the sequel, we show that in many interesting cases this oracle is not more computationally expensive than the basic one.
For example, when the $\lambda$-oracle is implemented as a linear program (LP), the additional constraints are simply added to the LP formulation and do not complicate the problem significantly.
%

%optimization 
% !TeX root = slack_rescaling.tex
\section{Slack Rescaling Inference Algorithms based on Convexity}

 We first present the inference algorithms based on convexity of the function, which results in the form of a binary search.

\subsection{Binary Search for Cutting-plane Optimization}
In this section, we review 
framework in \cite{sarawagi2008accurate} use the formulation in \eqref{obj_slack} as a starting point. This is slightly different from our methods that to work with cutting-plane optimization.

 Rewrote the constraints as,
\begin{equation}
1 + f(y) - f(y_i) - \dfrac{\xi_i}{L(y,y_i)} \leq 0 \quad\forall i,y\ne y_i
\label{eq:scaled_constraint_1}
\end{equation}
%\cite{sarawagi2008accurate} use the first set of constraints \eqref{eq:scaled_constraint_1}, and hence solve the problem:
Hence, to find a violated constraint they attempt to maximize: solve the problem with substitution of:
\begin{align}
%&\argmax_{y\in \mathcal{Y}'} \left ( 1+f(y)-f(y_i)-\dfrac{\xi_i}{L(y_i,y)}\right) \nonumber \\
%=&
\argmax_{y\in \mathcal{Y}'} \left ( h(y)-\dfrac{\xi_i}{g(y)}\right)
\label{yhat_c}
\end{align}
where $\mathcal{Y}'=\{y|y\in
\mathcal{Y},h(y)>0, % 1+f(y)-f(y_i)>0,
y \neq y_i\}$. Note that   $h(y)=1 + f(y) - f(y_i)$ and $g(y)=L(y,y_i)$.
They suggest minimizing a convex upper bound of \eqref{yhat_c} which stems from the convex conjugate function of $\dfrac{\xi_i}{g(y)}$:
\begin{align}
&\max_{y\in \mathcal{Y}'} h(y)-\dfrac{\xi_i}{g(y)}%\nonumber\\
=\max_{y\in \mathcal{Y}'} \min_{\lambda\ge 0} \left ( h(y)+\lambda g(y)-2\sqrt{\xi_i \lambda}\right) \nonumber\\
\le &\min_{\lambda\ge 0} \underset{y\in \mathcal{Y}'}{\max} F'(\lambda,y) %\nonumber \\
=\min_{\lambda\ge 0} \underset{y\in \mathcal{Y}'}{\max}\; F'(\lambda,y) = \min_{\lambda\ge 0} F(\lambda) \label{eq:p1_obj_up}
\end{align}
 where $F(\lambda)=\underset{y\in \mathcal{Y}'}{\max}\; F'(\lambda,y) =\max_{y\in \mathcal{Y}'} h(y)+\lambda g(y)-2\sqrt{\xi_i \lambda}$.
 Equation (\ref{eq:p1_obj_up}) can be also derived from  following quadratic bound, 
 \begin{align*}
%         (\ref{obj_sa})=
         &\underset{y\in \mathcal{Y'}}{\max} h(y)+\min_{\lambda>0}\left
\{ \left
(  \sqrt{\dfrac{\xi_i}{g(y)}}-\sqrt{\lambda g(y)}\right)^2-\dfrac{\xi_i}{g(y)}\right
\} \\
=&\underset{y\in \mathcal{Y'}}{\max}\min_{\lambda\ge 0} \left ( h(y)+\lambda g(y)-2\sqrt{\xi_i \lambda}\right) 
 \end{align*}

Since $F(\lambda)$ is a convex function, \eqref{eq:p1_obj_up} can be solved by a simple search method such as golden search over $\lambda$ \cite{sarawagi2008accurate}.
% \in [\epsilon/L_{max}, \max_{y\ne y_i} h(y)-h(y_i)+1-\xi_i/L_{max}].$

Although this approach is suitable for the cutting plane algorithm,
unfortunately, it cannot be easily extended to other training
algorithms.  In particular, $F'(\lambda,y)$ is defined in terms of
$\xi_i$, which ties it to the constrained form \eqref{obj_slack}.

\subsection{Binary Search for SGD}\label{sec:binary_sgd}

In the previous section, the function inside argmax operation was in a specific form used in cutting plane optimization, from now on we investigate solving argmax from that can be used in other optimization, which is more plausible for a large size dataset. Specifically, we present the framework for
solving the maximization problem \eqref{eq:slack_argmax}, which we
write as:
\begin{align}\label{eq:slack_argmax}
\max_y \Phi(y) := \max_y h(y) g(y)
\end{align}
%In \secref{sec:alg_lambda_oracle}
We describe two new algorithms to solve this problem using access to the $\lambda$-oracle, which has several advantages over previous approaches.

We first present a binary search algorithm similar to the one proposed in \cite{sarawagi2008accurate}, but with one main difference.
As we stated before, our algorithm can be easily used with training methods that optimize the unconstrained objective \eqref{unconstrained_obj_slack}, and can, therefore, be used for SGD\cite{RatliffSubgradient}, SDCA\cite{shalev2013accelerated}, and FW\cite{lacoste2013block} since the algorithm minimizes a convex upper bound on $\Phi$ without slack variable $\xi_i$.
The algorithm is based on the following lemma.
\begin{lemma}\label{binary_upper}
Let $\bar{F}(\lambda)= \frac{1}{4}  \max_{y\in\mathcal{Y}^+}\left ( \frac{1}{\lambda}h(y)+\lambda g(y)  \right )^2$, then
\begin{align*} 
\max_{y\in\mathcal{Y}} \Phi(y)&
\le \min_{\lambda>0}\bar{F}(\lambda)
\end{align*}
and $\bar{F}(\lambda)$ is a convex function in $\lambda$. 
\label{app:quadratic_bound}

\proof
First, let $\mathcal{Y}^+=\{y|y\in \mathcal{Y},h(y)>0\}$, then $\max_{y\in\mathcal{Y}} \Phi(y) = \max_{y\in\mathcal{Y}^+} \Phi(y)$,
since any solution $y$ such that $h(y) < 0$ is dominated by $y_i$, which has zero loss.
Second, we prove the bound w.r.t. $y\in\mathcal{Y}^+$.
In the following proof we use a quadratic bound (for a similar bound see \cite{nguyen1995minimizing}).
\begin{align}
\max_{y\in\mathcal{Y}^+} \Phi(y)&
=\max_{y\in\mathcal{Y}^+} h(y)g(y)= \max_{y\in\mathcal{Y}^+} \frac{1}{4}\left (2\sqrt{h(y)g(y)} \right )^2\nonumber\\
%\max_y 2\sqrt{h(y)g(y)}&
&=\frac{1}{4}\left ( \max_{y\in\mathcal{Y}^+} \min_{\lambda>0} \left \{\frac{1}{\lambda}h(y)+\lambda g(y) \right \}\right )^2\nonumber
 \\ &\le \frac{1}{4} \left (\min_{\lambda>0}\ \max_{y\in\mathcal{Y}^+} \left \{\frac{1}{\lambda}h(y)+\lambda g(y) \right \} \right )^2 \label{upper2}
\end{align}
%Finally, it is easy to see that $F(\lambda)$ is convex by differentiating twice.
%\todo{Write the convexity proof.}
To see the convexity of $\bar{F}(\lambda)$, we differentiate twice to obtain:
%\begin{align*}
%\bar{F}(\lambda)=\frac{1}{4}  \max_{y\in\mathcal{Y}^+} \frac{1}{\lambda^2}h(y)^2+\lambda^2g(y)^2+2h(y)g(y)
%\end{align*}
%and by differentiating twice, we obtain:
\begin{align*}
\dfrac{\partial^2 \bar{F}(\lambda)}{\partial \lambda^2}=\frac{1}{4}  \max_{y\in\mathcal{Y}^+} 6\frac{1}{\lambda^4}h(y)^2+2g(y)^2>0
\end{align*}
% Recall that $g(y)=L(y,y_i)\ge 0$. Inside of the max function is
 %convex, which can be seen by differentiating twice, and then, $F(\lambda)$
 %is the sum of convex functions.
 \qed
\end{lemma} 
Similar to \cite{sarawagi2008accurate}, we obtain a convex upper bound
on our objective.
Evaluation of the upper bound $\bar{F}(\lambda)$ requires using only the $\lambda$-oracle.
Importantly, this alternative bound $\bar{F}(\lambda)$ does not depend on the slack variable $\xi_i$, so it can be used with algorithms that optimize the unconstrained formulation \eqref{unconstrained_obj_slack}. As in \cite{sarawagi2008accurate}, we minimize $\bar{F}(\lambda)$ using \emph{binary search} over $\lambda$.
The algorithm keeps track of $y_{\lambda_t}$, the label returned by the $\lambda$-oracle for intermediate values $\lambda_t$ encountered during the binary search, and returns the maximum label $\max_t \Phi(y_{\lambda_t})$.
This algorithm focuses on the upper bound $\min_{\lambda>0}\bar{F}(\lambda)$, and interacts with the target function $\Phi''$ only through evaluations $\Phi''(y_{\lambda_t})$ (similar to \cite{sarawagi2008accurate}).

\section{Slack Rescaling Inference with  Geometrical Interpretation}

Previously, we focused on algorithms based on binary search utilizing the convex property of the function. We next present an algorithm that aims to optimize $\Phi(y)$ in a more direct manner, using a geometrical interpretation of mapping labels into $\reals^2$.

\subsection{Geometrical Interpretation of $\lambda$-oracle search}\label{sec:geometrical_inter}

The main observation in \cite{choi2016fast} is that each label $y$ can be mapped as a point $[h(y),g(y)]$ in $2$-dimensional plane $P=\mathbb{R} \times \mathbb{R}_+$ where $h(y)$ and $g(y)$ are  the coordinates in  $X$-axis and $Y$-axis of $P$ correspondingly. Taking this view, from now on, we consider the space of $\mathcal{Y}$ as a point in $P$, i.e.,  $\mathcal{Y}\subseteq P$. 
We will use both  $[y]_1$ and $h(y)$  for $X$-coordinate and  $[y]_2$ and $g(y)$  for $Y$-coordinate of point $y$.

The contours of our objective function
$\Phi(y)=[y]_{1}\cdot[y]_{2}$ are then
hyperbolas.  We would like to maximize this function by repeatedly
finding points that maximize linear objectives of the form
$\mathcal{L}_\lambda(y)=h(y)+\lambda g(y)$,
whose contours form lines in the plane.  See Figure \ref{fig:contour}.

An example of mapping of label into $\reals^2$ is shown in Figure \ref{fig:all_points}.
\begin{figure}[H]
\centering     
\includegraphics[width=\linewidth]{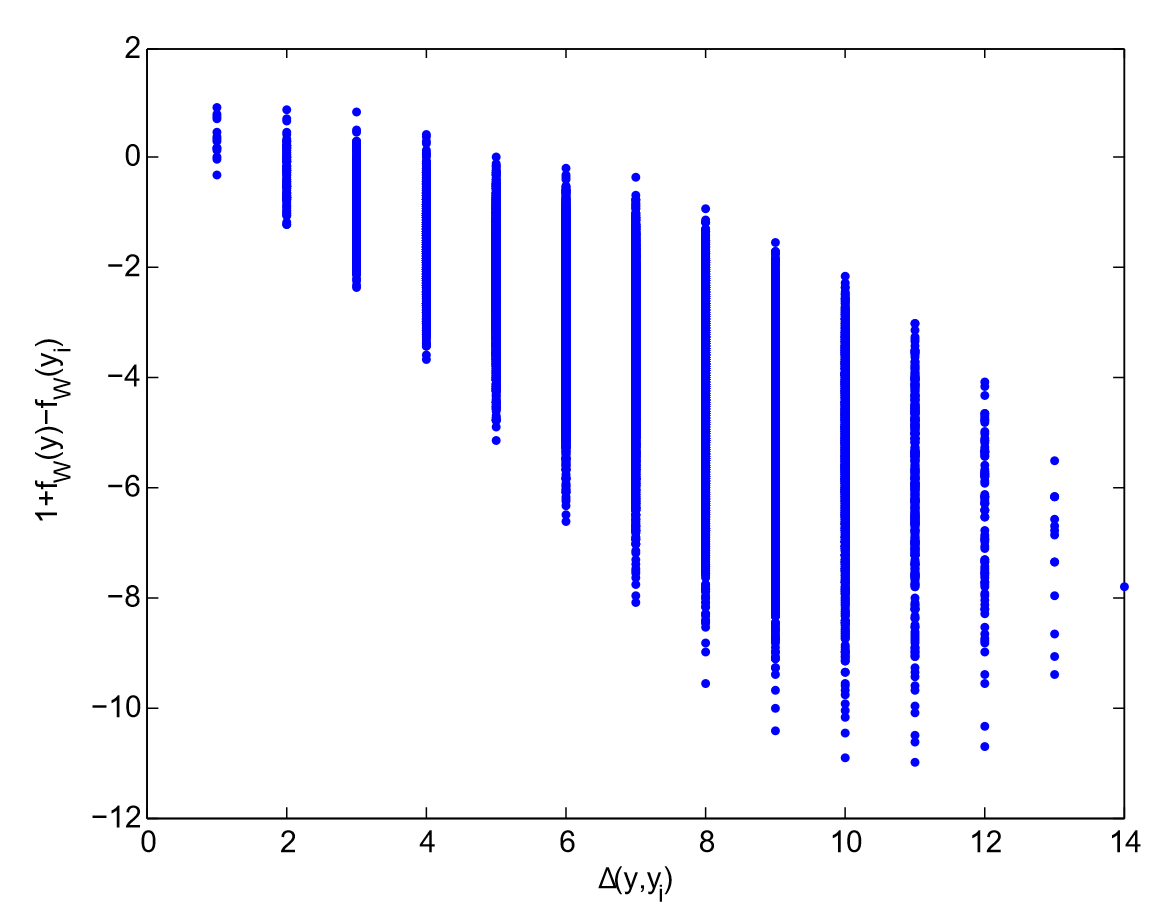}
% \begin{subfigure}[b]{.49\linewidth}
% \includegraphics[width=\linewidth]{all_points.pdf}
% \caption{Accuracy vs iterations}
% \end{subfigure}
\caption{A snapshot of labels during optimization with Yeast dataset.  Each $2^{14}-1$ labels is shown as a point in the figure \ref{fig:all_points}. X-axis is the $\triangle(y,y_i)$ and Y-axis is $1+f_w(y)-f_w(y_i)$.}\label{fig:all_points}
\end{figure}
\begin{figure}
%\begin{subfigure}
        \includegraphics[width=.8\linewidth]{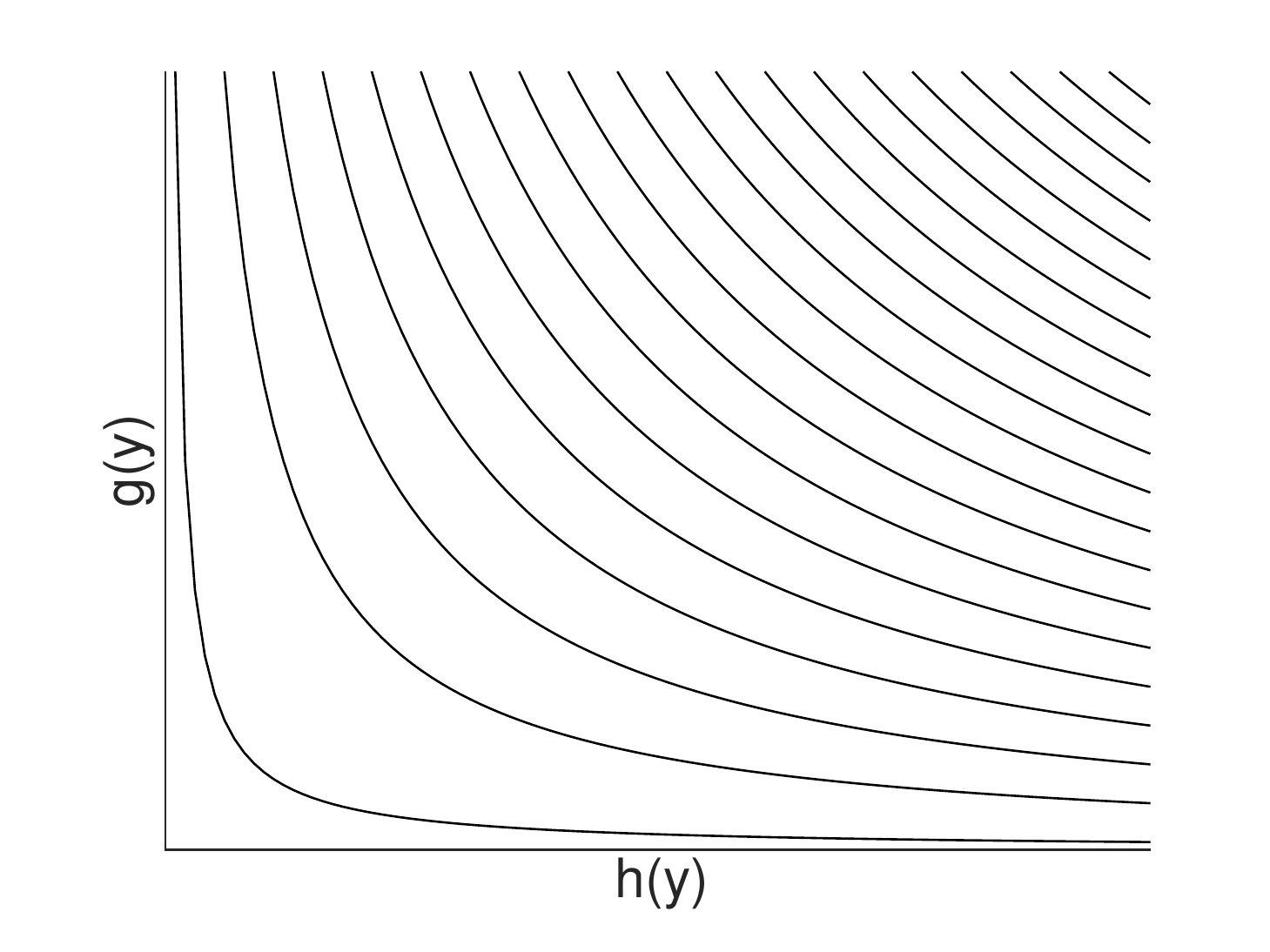}
        \label{fig:phi_contour}
        \caption{$\Phi$ contour.}
%    \end{subfigure}
%\begin{subfigure}
        \includegraphics[width=.8\linewidth]{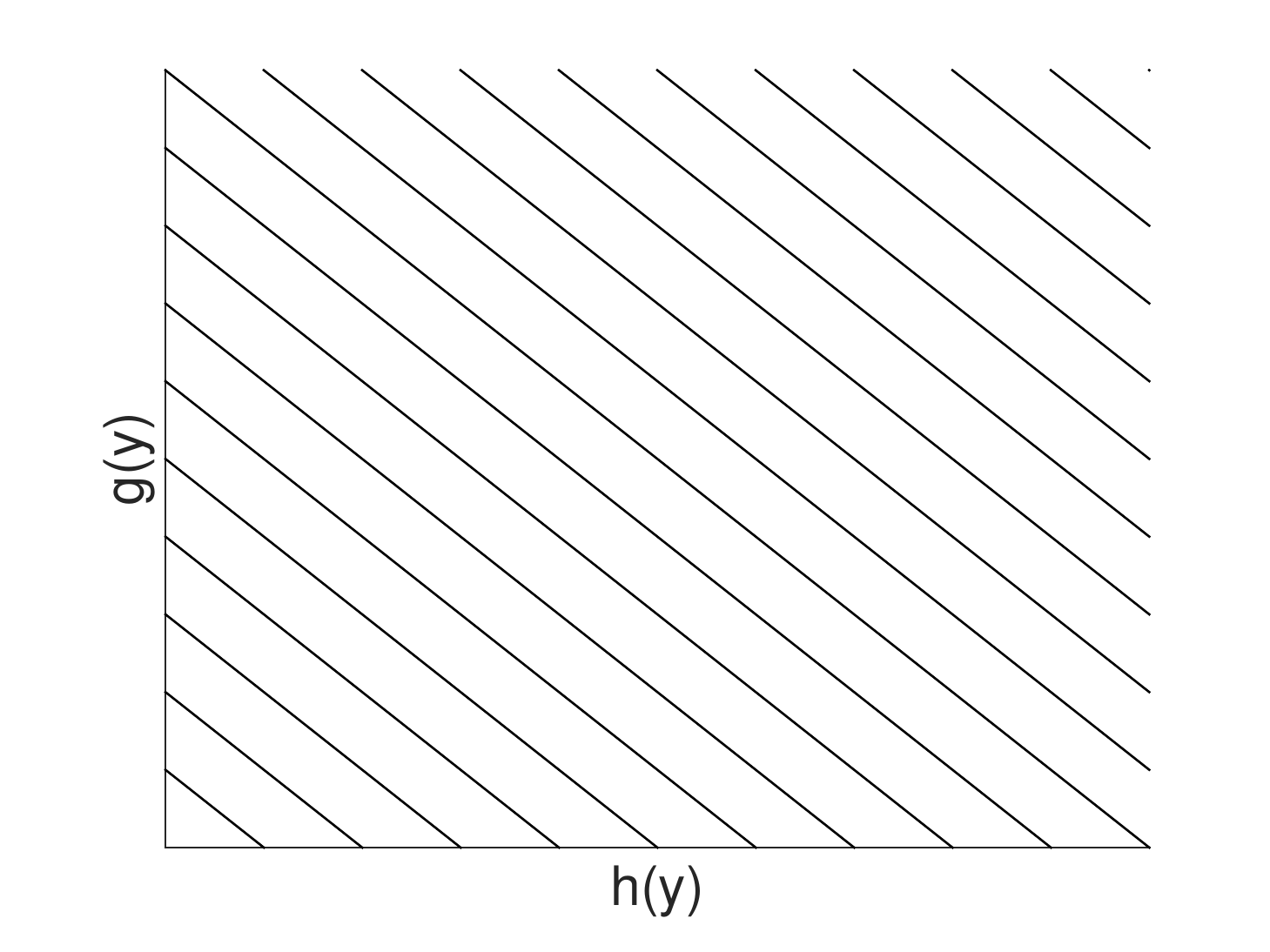}
        \label{fig:lambda_contour}
        \caption{$\lambda$ contour.} %$\mathcal{L}$ contour.}
%    \end{subfigure}    
\caption{Contour of the two functions considered in $\reals^2$. $\Phi$ contour is the contour of the objective function, and $\lambda$ contour is the contour used by the oracle. }\label{fig:contour}
\end{figure}
% Corresponding function to the objective function $\Phi(y)=h(y)g(y)$ is   $\Phi(y)=[y]_1\cdot [y]_2$ and  the contour of the function is hyperbolic, the multiplication of the two coordinates. On contrast, $\lambda$-oracle uses linear response $\mathcal{L}_\lambda(y)=[y]_1+\lambda [y]_2$. Our problem amounts to finding maximal $\Phi$ response point in $\mathcal{Y}\subset \reals^2$   using an oracle revealing maximal $\mathcal{L}_\lambda$ response point.

\begin{figure}[H]
\centering     %%% not \center
\includegraphics[width=\linewidth]
%{bisecting_.pdf}
{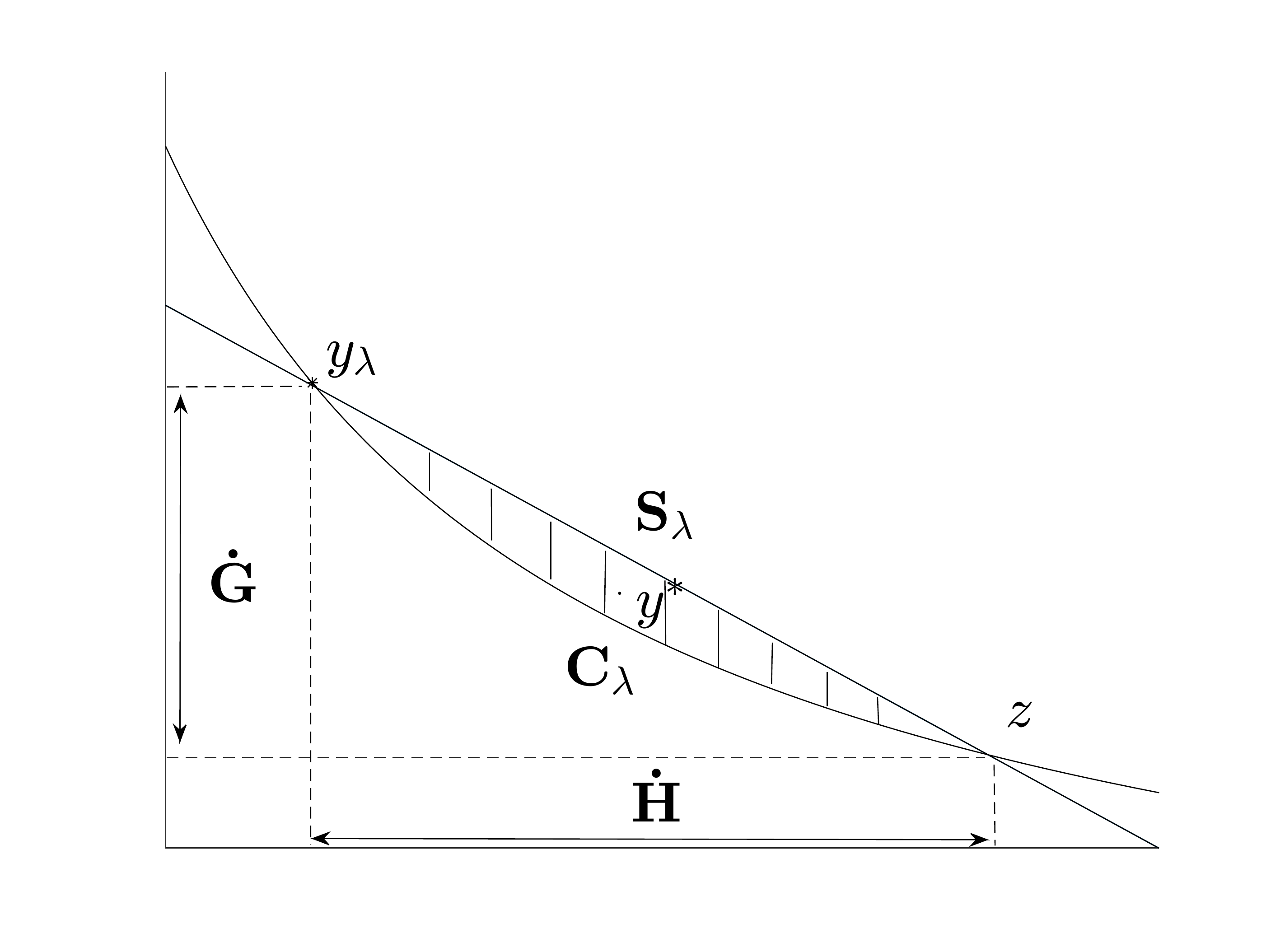}
\caption{Geometric interpretation of the $\lambda$-oracle: ${y^*}$ must reside between the upper bound $S_\lambda$ and the lower bound $C_\lambda$, the shaded area. It follows that $h(y^*)$ and $g(y^*)$ reside in a simple segment  $\dot{H}$ and $\dot{G}$ respectively.}\label{fig:geo}
\end{figure}
The importance of the $\reals^2$ mapping is that each $y(\lambda)$
revealed by the $\lambda$-oracle shows that $y^*$ can only reside in a
small slice of the plane.
See Figure \ref{fig:geo}.

\begin{lemma}\label{lem:restricted_space}
Let $S_\lambda$ be a line through $y(\lambda)$ and ${z}=[\lambda[y(\lambda)]_2,\frac{1}{\lambda}[y(\lambda)]_1]$, and   let $C_{\lambda}=\{y\in\reals^2|[y]_{1}\cdot [y]_{2}=\Phi(y(\lambda))\}$ be the hyperbola through $y(\lambda)$.
Then, 
${y^*}$  is on or below line $S_\lambda$, and ${y^*}$  is on or above hyperbola $C_\lambda$. 
\proof 
If there exists a $y\in\mathcal{Y}$ which is above $S_\lambda$, it contradicts the fact that $y(\lambda)$ is the argmax point for function $\mathcal{Y}_\lambda$.  And the second argument follows from ${y^*}$ being the argmax label w.r.t.~$\Phi$, and the area above $C_\lambda$ corresponds to points whose $\Phi$ value is greater than ${y(\lambda)}$. \qed
\end{lemma}
 It follows that $h(y^*)$ and $g(y^*)$ must each reside in a segment:
\begin{lemma} \label{cor:simple_segment}
Let $\dot{H}=[ \min([{\mathcal{O}(\lambda)}]_1,[{z}]_1), \max([{\mathcal{O}(\lambda)}]_1,[{z}]_1)]$ and $\dot{G}=[\min([{\mathcal{O}(\lambda)}]_2,[{z}]_2),$ $\max([{\mathcal{O}(\lambda)}]_2,[{z}]_2)]$.
Then, 
\begin{align*}
h(y^*)\in \dot{H}, &&g(y^*)\in \dot{G}
\end{align*} 
\proof This follows from the fact that $S_\lambda$ and $C_\lambda$ intersects at two points, ${\mathcal{O}(\lambda)}$ and ${z}$, and the boundaries, $S_\lambda$ and $C_\lambda$, are strictly decreasing functions. \qed
\end{lemma}

\subsection{Bisecting search}
\label{sec:bisecting_search}

%However, we also show that any algorithm which uses only the $\lambda$-oracle cannot always recover an optimal solution.
%Therefore, in \secref{sec:angular_w_constrained_oracle} we proposed an improved algorithm which requires access to an augmented $\lambda$-oracle that can also handle linear constraints. 

In this section, we propose a search algorithm which is based on the previous geometric interpretation.
Similar to the binary search, our method also relies on the basic $\lambda$-oracle. %We will demonstrate in the experiments section that this algorithm outperforms the simple binary search from the experiments section.

\begin{comment}
We begin by describing our new geometric interpretation of $\Phi(y)$.
Each label $y$ is represented as a point in $\reals^2$ with coordinates $(h(y),g(y))$.
Now, we think of $\Phi(y)$ as a function in this space with value $\Phi(y)=h(y)g(y)$.
Since it is hard to find the optimal label $y^*=\argmax_y \Phi(y)$ directly, we will make use of the $\lambda$-oracle to search this space for $y^*$.
First, from the linear oracle we get the constraint $h(y^*) + \lambda g(y^*) \le h(y(\lambda)) + \lambda g(y(\lambda))$.
Second, from the definition of $y^*$ we have that $h(y^*)g(y^*) \ge h(y(\lambda)) g(y(\lambda))$.
Taking the intersection of these conditions, we obtain a region in our two-dimensional space where $y^*$ can reside.
This is region, denoted as ${\cal{A}}_\lambda$, is illustrated in \figref{fig:bisect}.

%The level sets of $\Phi(y)=h(y)g(y)$ in this space are \todo{shown in \figref{fig:bisect}}.
%Since it is hard to find the maximizer of $\Phi(y)$ directly, we will make use of the $\lambda$-oracle to search this space.
%Specifically, each value of $\lambda$ yields a linear function in this space (\todo{lines in \figref{fig:bisect}}), and a call to the oracle returns the label $y(\lambda)$ which optimizes that function.
\end{comment}
We next give an overview of the algorithm. We maintain a set of  possible value ranges $\lambda^*=\argmax_{\lambda>0}\Phi(\mathcal{O}(\lambda))$, $h(\lambda^*)$, and $g(\lambda^*)$ as $L, H,$ and $G$, respectively; all initialized as  $\reals$.
First, for each $\mathcal{O}(\lambda)$ returned by the oracle,  we take an intersection of $G$ and $H$ with a segment of possible values of $h(y)$ and $g(y)$, respectively, using Lemmas \ref{lem:restricted_space} and \ref{cor:simple_segment}.
%\footnote{We can reduce $G$ and $H$ tighter when $\Phi({y(\lambda)})<\Phi^+$  where $\Phi^+$ is the best solution so far, e.g., $G\gets G\cap \{g|\left(h(y(\lambda))+\lambda g(y(\lambda)) \right)g-\lambda g^2-\Phi^+>0\}$}
%
Second, we reduce the space $L$ of potential $\lambda$'s based on the following Lemma.
%From Lemma \ref{cor:simple_segment}  we can conclude that $h(y^*_\lambda)>h(y(\lambda))$ or $h(y^*_\lambda)<h(y(\lambda))$.
%We next reduce the space of potential $\lambda$'s using the intuition that we are only interested in values $\lambda'$ such that $y^*$ can be the solution returned by the $\lambda$-oracle using $\lambda'$.
%only need to search for labels $y$ such that $\Phi(y)>\Phi(y(\lambda))$.
%These inequalities corresponds to $\lambda^*>\lambda$ or $\lambda^*<\lambda$ from the following lemma (proved in the Appendix).
\begin{lemma}\label{MAP_seg}
$h(\mathcal{O}(\lambda))$ is a non-increasing function of $\lambda$, and $g(\mathcal{O}(\lambda))$
is a non-decreasing function of $\lambda$.

\proof 
Let $g_1=g(y_{\lambda_1}),h_1=h(y_{\lambda_1}), g_2=g(y_{\lambda_2}),$
and $h_2=h(y_{\lambda_2}).$
%Assume $\lambda_1 \ge \lambda_2$. Then,
\begin{align*}
%h_1+\lambda_1g_1\ge h_2+\lambda_1 g_2, \quad 
%h_2+\lambda_2g_2\ge h_1+\lambda_2 g_1\ge h_1+\lambda_1g_1\\
%&\Rightarrow 
&h_1+\lambda_1g_1\ge h_2+\lambda_1 g_2, 
\quad h_2+\lambda_2g_2\ge h_1+\lambda_2 g_1 \\
&\Leftrightarrow
 h_1-h_2+\lambda_1(g_1-g_2) \ge 0,-h_1+h_2+\lambda_2(g_2-g_1) \ge 0 \\
&\Leftrightarrow
(g_2-g_1)(\lambda_2-\lambda_1)\ge 0
\end{align*}
For $h$, change the role of $g$ and $h$. \qed
\end{lemma}
Thus, we can discard $\{\lambda'|\lambda'>\lambda\}$ if $h(y^*_\lambda)>h(y(\lambda))$ or $\{\lambda'|\lambda'<\lambda\}$ otherwise from $L$.
Next, we pick $\lambda\in L$ in the middle, and query $y(\lambda)$.
The algorithm continues until at least one of $L, H,$ and $G$ is empty.

Similar to the binary search from the previous section, this algorithm can be used with training methods like SGD and SDCA, as well as the cutting-plane algorithm. %or with the modified formulation of the cutting-plane algorithm \eqref{eq:scaled_constraint_2}.
However, this approach has several advantages compared to the binary search.
First, the binary search needs explicit upper and lower bounds on $\lambda$, thus it has to search the entire $\lambda$ space \cite{sarawagi2008accurate}.
%can consist of the entire space of $\lambda$'s or needs to be initialized, which requires additional computation \cite{sarawagi2008accurate}.
However, the bisecting search can directly start from any $\lambda$ without an initial range, and for instance, this can be used to warm-start from the optimal $\lambda$ in the previous iteration. %, and increasing or decreasing $\lambda$ by some constant factor until upper or lower bound is found.
%Also, bisecting search reduces the search space with a rate of $\frac{1}{2}$, which is faster than the rate of golden search, which is approximately $0.618$.
%this is not really true%
 Furthermore, the search is guided by the target objective function $\Phi$ itself, whereas the binary search decreases the convex upper bound $\bar{F}$, which does not correspond to increasing $\Phi$, and $\Phi$ might be even decreasing.
Furthermore, we point out that since the search space of $h$ and $g$ is also bisected, the procedure can terminate early if either of them becomes empty. 
 
 We also propose two improvements that can be applied to either the binary search or the bisecting search. Specifically, we first provide a simple stopping criterion that can be used to terminate the search when the current solution $y_{\lambda_t}$ will not further improve. If $L=[\lambda_m,\lambda_M]$, and both endpoints have the same label, i.e.,  $y_{\lambda_m}=y_{\lambda_M}$, then we can terminate the binary search safely because from lemma \ref{MAP_seg}, 
it follows that the solution $\mathcal{O}(\lambda)$ will not change in this segment. 

 Second, we show how to obtain a bound on the suboptimality of the current solution, which can give some guarantee on its quality.
%The following lemma will be useful to determine when to terminate with the binary search.
%Convexity in lemma \ref{binary_upper} only holds for $\mathcal{Y}^+$. However, $h(y)<0$ implies $\Phi(y)<0$, and from lemma \ref{MAP_seg}, $h(y_{\lambda'})>h(y(\lambda))\iff \lambda'<\lambda$. Therefore,  instead of search in space of $\mathcal{Y}^+,$ we can  search the labels without the restriction, and ignore label $y$ such that  $h(y)>0$, and decrease $\lambda$. 
Let $K(\lambda)$ be the value of the $\lambda$-oracle. i.e., 
\begin{align}\label{loracle0}
K(\lambda)=\max_{y\in\mathcal{Y}} h(y)+\lambda g(y).
\end{align}
\begin{lemma}\label{lem3}
$\Phi^*$ is upper bounded by
\begin{align} \label{optimal_upper}
\Phi(y^*)\le \dfrac{K(\lambda)^2}{4\lambda}
\end{align}
\proof %Multiplying $g(y)$ in (\ref{loracle0}), we have
\begin{align*}
& h(y)+\lambda g(y)\le K(\lambda)\\ 
&\iff g(y)(h(y)+\lambda g(y))\le g(y)K(\lambda)\\ 
&\iff
 \Phi(y) \le g(y) K(\lambda) - \lambda g(y)^2\\&=-\lambda\left
( g(y)- \dfrac{K(\lambda)}{2\lambda}\right )^2 + \dfrac{K(\lambda)^2}{4\lambda}\le \dfrac{K(\lambda)^2}{4\lambda}
\end{align*}
\qed
\end{lemma}

\begin{algorithm}
  \caption{Bisecting search }
%In function {\bf MinNorm} and {\bf MinAlpha}, $N, E$ vectors are passed
%by reference and changes inside of the function are passed to outside of
%the function.
\label{alg:bisecting}
  \begin{algorithmic}[1]
    \Procedure{Bisecting}{$\lambda_0$}%, \alpha$}
    \INPUT {Initial $\lambda$ for the search $\lambda_0\in\mathbb{R}_+$
%$ initial step size $\alpha>1$.
}
    \OUTPUT{$\hat{y} \in \mathcal{Y}.$}
    \INIT  {$H=G=L=\mathbb{R}_+,\lambda=\lambda_0$, $\hat{\Phi}=0$.}
    \While {$H\ne\emptyset$ and $G\ne\emptyset$ }
     \State $y'\gets \mathcal{O}(\lambda)$
     \State $u\gets [h(y')\; \lambda g(y')],v \gets [g(y')\; \frac{1}{\lambda} h(y')]$
     \State $H\gets H \cap \{h'|\min u\le h' \le \max u\}$ \Comment{Update}
     \State $G\gets G \cap \{g'|\min v\le g' \le \max v\}$
     \If{$v_{1}\le v_2$}\Comment{Increase $\lambda$}
       \State $L\gets L \cap \{\lambda'\in\mathbb{R}|\lambda'\ge\lambda\}$
%     \If{$\max L=\infty$} \Comment{No upper bound of $\lambda$ found yet}
%       \State $\lambda\gets\alpha\lambda$
%      \Else
%       \State $\lambda \gets \frac{1}{2}(\min L+\max L)$
%      \EndIf
     \Else\Comment{Decrease $\lambda$}
       \State $L\gets L \cap \{\lambda'\in\mathbb{R}|\lambda'\le \lambda\}$
     \EndIf
    \State $\lambda \gets \frac{1}{2}(\min L+\max L)$
    \If {$h(y')g(y')\ge\hat{\Phi}$}
     \State $\hat{y}\gets y',$ $\hat{\Phi}\gets h(y')g(y')$.
    \EndIf
    \EndWhile
    \EndProcedure
  \end{algorithmic}
\end{algorithm}

So far we have used the $\lambda$-oracle as a basic subroutine in our search algorithms.
Unfortunately, as we show next, this approach is limited as we cannot guarantee to find the optimal solution $y^*$, even with the unlimited number of calls to the $\lambda$-oracle.
This is somewhat distressing since with unlimited computation we can find the optimum of \eqref{eq:slack_argmax} by enumerating all $y$'s.

%%%%%%%%%%%%%%%%%%%%%%%%%%%%%%%%%%%%%%%%%%%%%%%%%%%%%

%%%%%%%%%%%%%%%%%%%%%%%%%%%%%%%%%%%%%%%%%%%%%%%%%%%%%

\subsection{Limitation of the $\lambda$-oracle}
\label{sec:lambda_oracle_inexact}
Until now, we used only the $\lambda$-oracle to search for $y^*$ without
directly accessing the functions $h$ and $g$. We now show that this
approach, searching $y^*$ with only a $\lambda$-oracle, is very limited:
even with an unlimited number of queries, the search cannot be
exact and might return a trivial solution in the worst case.
\begin{theorem}\label{loracle_l1}
  Let $\hat{H}=\max_y h(y)$ and $\hat{G}=\max_y g(y)$. For any
  $\epsilon>0$, there exists a problem with 3 labels
  such that for any $\lambda\geq0$, $\Phi(\mathcal{O}(\lambda))=\min_{y\in \mathcal{Y}} \Phi(y)<\epsilon$, while
  $\Phi(y^*)=
  \dfrac{1}{4}\hat{H}\hat{G}$.

\proof

We will first prove the following lemma which will be used in the proof.

\begin{lemma}\label{limit_oracle}
 Let $A=[A_1\;A_2]\in \mathbb{R}^2,B=[B_1 \;B_2]\in \mathbb{R}^2$, and $C=[C_1 \;C_2]\in \mathbb{R}^2$, and $A_1<B_1<C_1$. If $B$ is under the
 line $\overline{AC}$, i.e.,$\exists t$,$0\le  t\le 1$,$D=tA+(1-t)C$, $D_1=B_1$, $D_2>B_2$. Then,   $\nexists \lambda\ge0$, $v=[1 \;\lambda]\in \mathbb{R}^2$,  such that
 \begin{align}
   v\cdot B>v\cdot A \;\text{  and }
  v\cdot B>v\cdot C \label{eq:argmaxv}
 \end{align}
 
 \proof Translate  vectors $A,B,$ and $C$ into coordinates of $[0,A_{2}],[a,b],$
 $[C_{1},0]$ by adding a vector $[-A_1,-C_2]$ to each vectors $A,B,$ and $C$,  since it does not change $B-A$ or $B-C$. Let $X=C_1$ and $Y=A_2$. 

If $0\le\lambda \le \dfrac{X}{Y}$, then $v\cdot A=\lambda Y\le X= v\cdot C$. $v\cdot (B-C)>0\iff (a-X)+\lambda b>0$ corresponds to  all the points above line  $\overline{AC}$. Similarly,
if $\lambda \ge \dfrac{X}{Y}, $ \eqref{eq:argmaxv} corresponds
to $a+\lambda (b-Y)>0$ is also all the points above  $\overline{AC}$.
\qed 
 \end{lemma}
 From lemma \ref{limit_oracle}, 
 if $y_1$,$y_2\in \mathcal{Y}$, then all the labels
 which lies under line $y_1$ and $y_2$ will not be found by $\lambda$-oracle. In the adversarial case, this holds when label lies on the line also. Therefore,
Theorem \ref{loracle_l1} holds when there exist three labels, for arbitrary small $\epsilon>0$,
$A=[\epsilon, \hat{G}],B=[\hat{H}, \epsilon]$, and $C=[\frac{1}{2}\hat{H},\frac{1}{2}\hat{G}]$, $\mathcal{Y}=\{A,B,C\}$. In this case for any $\lambda>0$, $\Phi(\mathcal{O}(\lambda))\approx0$.
\qed 
 \end{theorem}
 Theorem \ref{loracle_l1} shows that any search algorithm that can access
 the function only through $\lambda$-oracle, including the method in
 \cite{sarawagi2008accurate} and both methods presented above,
 cannot be guaranteed to find a label optimizing $\Phi(y)$, even
 approximately, and even with unlimited accesses to the oracle. This
 problem calls for a more powerful oracle.

\subsection{Angular search with the constrained-$\lambda$-oracle}
\label{sec:angular_w_constrained_oracle}

The \emph{constrained $\lambda$-oracle}   defined in
\eqref{eq:constrained_oracle} has two inequality constraints to
restrict the search space.  Using this modified algorithm, we can
present an algorithm that {\em is} guranteed to find the most
violating constraint, as captured by the following guarantee, proved
in the section \ref{angular_search_proof}:
\begin{theorem}\label{thm:angular_optimal} Angular search described in
  Algorithm \ref{alg:angular} finds the optimum $y^*=\argmax_{y\in
    \mathcal{Y}} \Phi(y) $ using at most $t=2M+1$ iteration where $M$ is the number of the labels.
\end{theorem}
This is already an improvement over the previous methods, as at least
we are guaranteed to return the actual most violating label.  However,
it is still disappointing since the number of iterations, and thus the number of oracle accesses might actually be larger than the number of
labels.  This defies the whole point, since we might as well just
enumerate over all $M$ possible labels.  Unfortunately, even with a
constrained oracle, this is almost the best we can hope for. In fact,
even if we allow additional linear constraints, we might still need
$M$ oracle accesses, as indicated by the following Theorem, which is proved in the section \ref{sec:limitation2}.

\begin{theorem}\label{loracle_l2}
Any search algorithm accessing labels only through a $\lambda$-oracle with any number of linear constraints cannot find $y^*$ using less than $M$ iterations in the worst case, where $M$ is the number of labelings.

\end{theorem}

Fortunately, even though we cannot guarantee optimizing $\Phi(y)$
exactly using a small number of oracle accesses, we can at least do so
approximately. This can be achieved by Algorithm \ref{alg:angular} (see section \ref{angular_search_proof} for proofs), as the next theorem states.

 \begin{algorithm}[H]
  \caption{Angular search }
%In function {\bf MinNorm} and {\bf MinAlpha}, $N, E$ vectors are passed
%by reference and changes inside of the function are passed to outside of
%the function.
\label{alg:angular}
  \begin{algorithmic}[1]
    \Procedure{AngularSearch}{$\lambda_0,T$}%, \alpha$}  
%    \INPUT {$\lambda\in\mathbb{R}_+$ }
%    \OUTPUT{$\hat{y} \in \mathcal{Y}.$}
%    \State \Return DoSplit$(\lambda,[0 \; \infty], [\infty \; 0],1)$
%    \hline
%    \EndProcedure
%    \Procedure{DoSplit}{$\lambda, U, L, d$}%, \alpha$}
    \INPUT {$\lambda_{0}\in\mathbb{R}_+$, 
%     linear upper bound $U:\reals^2$, linear lower bound $L:\reals^2$, 
       and maximum iteration $T\in \reals_{+}$
}   \OUTPUT{$\hat{y} \in \mathcal{Y}.$}
    \INIT  {$\alpha_0=\infty,\beta_0=0,$ Empty queue $\mathcal{Q}$, $\hat{y}=\emptyset$.$\lambda\gets \lambda_0$}
    \State $\text{ADD}(\mathcal{Q},(\alpha,\beta,0))$
    \While{$\mathcal{Q}\ne\emptyset$} 
    \State $(\alpha,\beta,s)\gets \text{Dequeue}(\mathcal{Q})$
    \If {$\beta\ne0$}
        \State $\lambda\gets \frac{1}{\sqrt{\alpha\beta}}$ 
    \EndIf
    \If {$s=0$}
        \State $y\gets\mathcal{O}_c(\lambda,\alpha,\beta)$   
    \Else
        \State $y\gets\underline{\mathcal{O}}_c(\lambda,\alpha,\beta)$   
    \EndIf
    \If {$\Phi(y)> \Phi(\hat{y})$}
    \State $\hat{y}\gets y$
    \EndIf
    \If {$y\ne\emptyset$}
    \State $z\gets [h(y)\; g(y)],z'\gets [\lambda g(y)\; \dfrac{1}{\lambda}h(y)]$
    \State $r\gets \left [\sqrt{\lambda h(y) g(y)}\; \sqrt{\frac{1}{\lambda} h(y) g(y)}\; \right ]$
    \If {$z_1=z'_1$}
    \State return $y$
    \ElsIf {$\partial(z)>\partial(z')$}
%    \State $P\gets z,Q\gets z'$
%     \State $K\gets(\partial(z),\partial(z'))$
     \State $K^{1}\gets(\partial(z),\partial(r),1)$
     \State $K^{2}\gets(\partial(r),\partial(z'),0)$
    \Else
%    \State $P\gets z',Q\gets z$
%     \State $K\gets(\partial(z'),\partial(z))$
     \State $K^{1}\gets(\partial(z'),\partial(r),1)$
     \State $K^{2}\gets(\partial(r),\partial(z),0)$
    \EndIf
%    \State $R\gets [\sqrt{a_1\cdot b_1}}\; \sqrt{a_2\cdot b_2}}]$
%    \State $Q\gets \text{DoSplit}($M_1\cdot M_2$,({\bot(\overline{PR})},\overline{OP},\overline{OR}))$
     \State $\text{ADD}(\mathcal{Q},K^1)$ $.\text{ADD}(\mathcal{Q},K^2)$ %\Comment{Add next search area}
%     \State  %    \State $y_{2}\gets \text{DoSplit}({\bot(\overline{RQ})},\overline{OR},\overline{OQ})$
%    \State \Return $\argmax_{y\in\{y,y_1,y_2\}}\Phi(y)$
    \EndIf
    \State $t\gets t+1$
    \If {$t=T$} \Comment{maximum iteration reached}
    \State \Return $\hat{y}$ 
    \EndIf
    \EndWhile    
    \EndProcedure
  \end{algorithmic}
\end{algorithm}

\begin{theorem}\label{thm:angular_suboptimality} In angular search, described in Algorithm \ref{alg:angular}, at iteration $t$, 
\begin{align*}
\dfrac{\Phi(y^*)}{\Phi(\hat{y}^t)}\le (v_1)^{\frac{4}{t+1}}
\end{align*}
where $\hat{y}^t = \argmax_t y^t$ is the optimum up to $t$, $v_1=\max \left \{\dfrac{\lambda_0}{\partial(y_1)}, \dfrac{\partial(y_1)}{\lambda_0}\right\}$, $\lambda_0$ is the initial $\lambda$ used, and $y_1$ is the first label returned by constrained $\lambda$-oracle.
\end{theorem}
We use $\partial(a)=\frac{a_2}{a_1}$ to denote the slope of a vector.

With proper initialization, we get the following runtime guarantee:
\begin{theorem}\label{thm:angular_running_time}  Assuming $\Phi(y^*)>\phi$, angular search described in algorithm \ref{alg:angular} with $\lambda_0=\dfrac{\hat{G}}{\hat{H}}, \alpha_0=\dfrac{\hat{G}^2}{\phi},\beta_0=\dfrac{\phi}{\hat{H}^2}$, finds an $\epsilon$-optimal solution, $\Phi(y)\ge(1-\epsilon)\Phi(y^*)$, in $T$ queries and $O(T)$ operations, where $T=4\log\left(\dfrac{\hat{G}\hat{H}}{\phi}\right)\cdot \dfrac{1}{\epsilon}$, and $\delta$-optimal solution, $\Phi(y)\ge\Phi(y^*)-\delta$,  in $T'$ queries and $O(T')$ operations, where $T'= 4\log\left(\dfrac{\hat{G}\hat{H}}{\phi}\right)\cdot \dfrac{\Phi(y^*)}{\delta}$.
\end{theorem}

 \begin{figure}[H]
\centering     %%% not \center
\includegraphics[trim = 60mm 0mm 30mm 0mm, clip, 
width=180mm]
{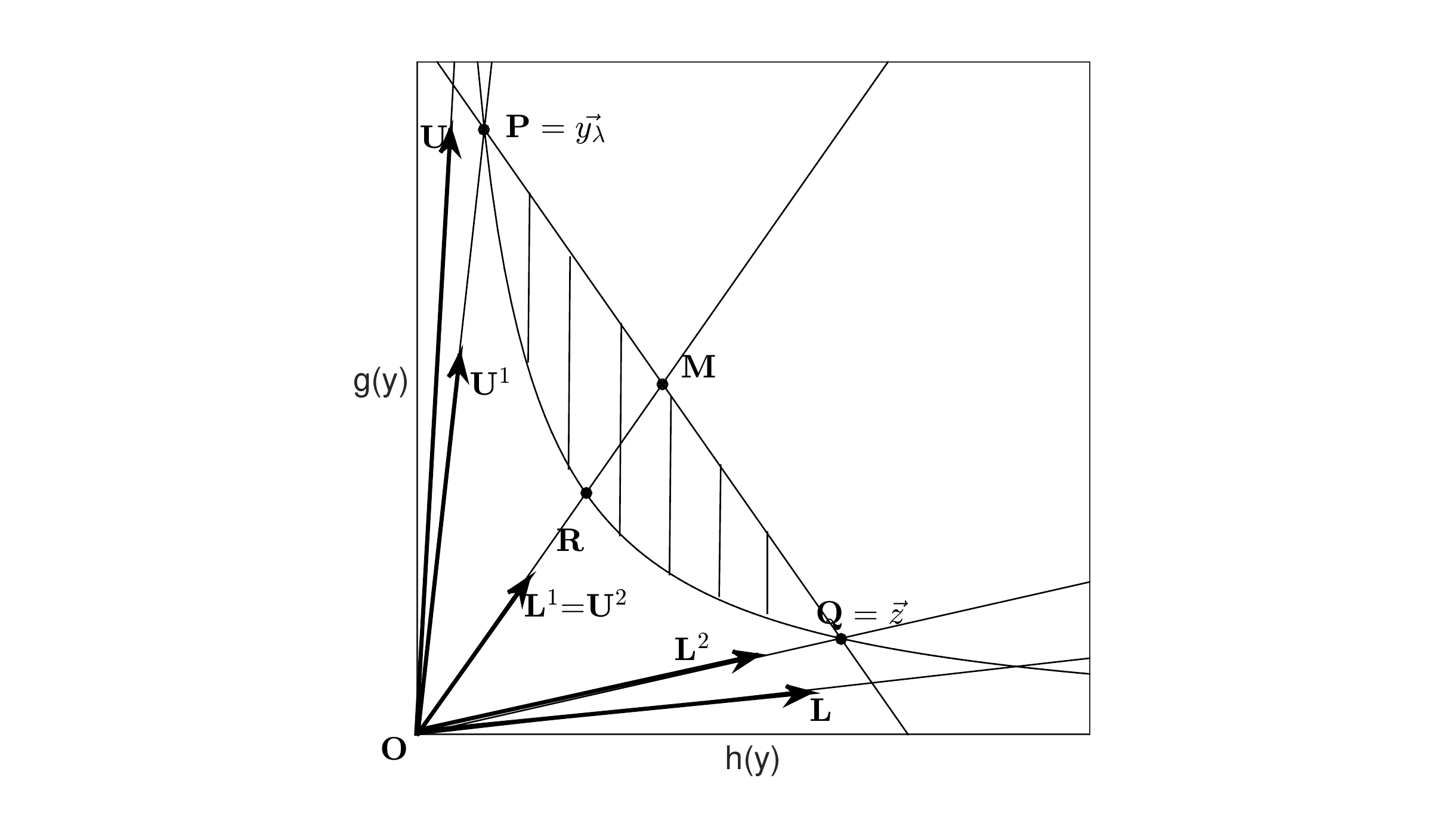}
\caption{Split procedure.}
\label{fig:split_cor}
\end{figure}

Here we give an overview of the algorithm with an illustration in Figure \ref{fig:split_cor}. The constrained $\lambda$-oracle restricts the search space, and this restriction can be illustrated as a linear upper bound $U$ and a lower bound $L$.
The search is initialized with the entire right angle: $U=[0\; \infty]$ and $L=[\infty\; 0]$, and maintains that ${y^*}$ is always between $U$ and $L$.
The constrained $\lambda$-oracle is used with $U,L$ and a certain $\lambda$ to reduce the potential area where ${y^*}$ can reside.
Specifically, the search space is reduced using an angle defined by $U=\overline{OP}$ and $L=\overline{OQ}$.
In the next iteration, the constrained $\lambda$-oracle is invoked with $U^1=\overline{OP}$ and $L^1=\overline{OM}$, and also with $U^2=\overline{OM}$ and $L^2=\overline{OQ}$.
Intuitively, each such query shrinks the search space, and as the search space shrinks, the suboptimality bound improves.
This process is continued until the remaining search space is empty.
The angular search algorithm defines the optimal $\lambda$ and values to be passed to the constrained $\lambda$-oracle.

In Algorithm \ref{alg:angular} each angle is dequeued, split, and enqueued recursively.
Each angle maintains its upper bound from the previous iterations and stops splitting itself and terminate if it is ensured that there exists no label with a larger $\Phi$ value within the angle. %For instance, for angle $A^1$, we are sure that there is no label above the upper bound $\overline{PM}$. 
When the oracle reveals a label with $\Phi(y(\lambda))=c$, we can safely discard all area corresponding to $\{y|\Phi(y)\le c\}$. % since $\Phi(y^*)\ge c$.
This works as a global constraint which shrinks the search space. Therefore, acquiring a label with a high $\Phi$ value in the early stages facilitate convergence. Thus, it is suggested to use a priority queue, and dequeue the angle with the highest upper bound on $\Phi$. % for the efficiency.

\paragraph{Label Cache}

A similar strategy is to have a {\em label cache}, the subset of previous most violated labels, denoted as $\mathcal{C}$.  With the label cache, we can discard a large part of the search space $\{y|\Phi(y)\le \max_{y'\in \mathcal{C}}\Phi({y'})\}$ immediately.
%Algorithm  \ref{alg:angular} also uses a variant of the constrained $\lambda$-oracle $\underline{y}_{\lambda,\alpha,\beta}=\underline{\mathcal{O}}_{c}(\lambda,\alpha,\beta)=\mathcal{O}_{c}(\lambda,-\frac{1}{\beta},-\frac{1}{\alpha})$  to prohibit the oracle from returning the same points by swapping the place of the equality in the two inequalities.
Algorithm \ref{alg:angular} also uses the constrained $\lambda$-oracle to avoid returning previously found labels.
Finally, for $\lambda_0$, we suggest to use $\lambda_0=\frac{\hat{H}}{\hat{G}}$, with $\hat{H}$ calculated from the current weights $w$. % or its the upper bound obtained from current $w$.

 %differs with $y_{\lambda,\alpha,\beta}$ only in the place of the equality in the inequality constraints,\begin{align*}
%\underline{y}_{\lambda,\alpha,\beta}=\underline{\mathcal{O}}_{c}(\lambda,\alpha,\beta)=\max_{y\in\mathcal{Y},\; \alpha h(y)> g(y),\; \beta h(y)\le g(y)} \mathcal{L}_\lambda(y)
%\end{align*}
%$\underline{y}_{\lambda,\alpha,\beta}$, which can be obtained using $y_{\lambda,\alpha,\beta}$ since   $\underline{y}_{\lambda,\alpha,\beta}=y_{\lambda,-\frac{1}{\beta},-\frac{1}{\alpha}}.$

% experiments
 
% !TeX root = slack_rescaling.tex
% 
 See section \ref{ap:illu} for the illustration of the angular search.
\section{Experiments}
\label{sec:experiments}
 In this section, we validate our contributions by comparing the different behaviors of the search algorithms on standard benchmark datasets, and its effect on the optimization. Specifically, we show that angular search with SGD is not only much faster than the other alternatives, but also enables much more precise optimization with slack rescaling formulation, outperforming margin rescaling. 

Unlike the simple structure used in \cite{bauer2014efficient}, we show applicability to complicated structure. We experiment with multi-label dataset modeled by a Markov Random Field with pair-wise potentials as in \cite{finley2008training}. Since the inference of margin scaling is NP-hard in this case, we rely on linear programming relaxation. Note that this complicates the problem, and  the number of labels becomes even larger (adds fractional solutions). %is  infinite, and the simplest solution of enumerating all the labels is not possible.    
Also notice that all of our results in previous sections apply with mild modification to this harder setting.
Two standard benchmark multi-label datasets, yeast\cite{elisseeff2001kernel} (14 labels)% enron\footnote{\url{<http://www.cs.cmu.edu/~enron/>}}
 and RCV1\cite{lewis2004rcv1}, are tested.
 For RCV1 we reduce the data to the $50$ most frequent labels.
 % Note we took most frequent 50 labels from  RCV1, and the enumerating $2^{50}$ label is impossible even without a linear relaxation. 
 For angular search, we stop the search whenever $\frac{\Phi(\hat{y})}{\Phi(y^*)}>0.999$ holds, to avoid numerical issues.
 \subsection{Comparison of the search algorithms}
 
  \begin{table}[h!]
  \begin{center}
    \begin{tabular}{| l  | c |c|c|}
    \hline
      &  Angular & Bisecting & Sarawagi \\
      \hline            \hline
 \multicolumn{4}{|c|}{Yeast (N=160)} \\      
    \hline
    Success &  22.4\%\ & 16.5\%\ & 16.4\% \\
    Queries per search & 3.8 & 10.3 & 43.2 \\
  Average time (ms) & 4.7 & 3.6 & 18.5 \\
      \hline            \hline
 \multicolumn{4}{|c|}{RCV1 (N=160)} \\      
    \hline
    Success &   25.6\%\  &      18.2\%\    &     18\%\\
    Queries per search &    4.8    &         12.7 \   &          49 \\
 Average time (ms) & 4.4 & 5.2 & 20.9 \\
    \hline    
    
    \end{tabular}\\
  \end{center}
  
  \caption{Comparison of the search algorithm. }\label{tabel_argmax}
\end{table} 
 Table \ref{tabel_argmax} compares the performance of the search in terms of the time spend, the number of queries, and the success percentage of finding the most violating label. %Specifically,   for several pass over  each instance $i$,\\ 
The cutting-plane algorithm calls the search algorithms to find the most violating label $\hat{y}$, and adds it the active set if the violation is larger than some margin $\epsilon$, i.e.,  $ \Delta(\hat{y},y_i)(1+f(\hat{y})-f(y_i))>\xi_i+\epsilon$.
%During cutting plain optimization with angular search, we try all three algorithms and compare result.
For cutting-plain optimization, we compare all three algorithms: Angular search, Bisecting search, and Sarawagi and Gupta's \cite{sarawagi2008accurate} (but just used Angular search for the update).
Success percentage is the percentage that the search algorithm finds such a violating label.
As expected from Theorem \ref{loracle_l1}, bisecting and Sarawagi's search miss the violating label in cases where angular search successfully finds one.
This is important for obtaining high accuracy solution. %, since bad approximation of gradient decrease the performance in the later iterations.
For RCV1 dataset, not only is angular search more accurate, but it also uses about 2.6 times less queries than bisecting and 10.1 times less queries than Sarawagi's search.
As for the timing, angular search is 1.18 times faster than bisecting search, and 4.7 times faster than Sarawagi's algorithm. 
 
 In figure \ref{fig:convergence}, we compare the convergence rate and the accuracy of the different optimization schemes using different  search algorithms. Additional plots showing convergence w.r.t.~the number of queries and iterations are in \ref{ap:additional_plot}.  These show that angular search with SGD converges order of magnitude faster.
%   Note that argmax that directly uses empirical loss (\ref{eq:slack_argmax}) optimizes the objective function more efficiently. Also, cutting-plane algorithm terminates when there are no $\epsilon$ violating constraints. Argmax in (\ref{yhat_c}) , denoted as Sarawgi, terminates in early iterations failing to find violating constraints with high objective value, whereas exact argmax (\ref{eq:slack_argmax}), continuously find the violating constraints and terminates last.  All of the bisecting argmax and exact argmax were violating constraints or equal to the most violating constraints. However, it was not the always the case for argmax of Sarawgi.  In figure \ref{m_f1_convergence}, faster convergence rate for SGD is shown.
 
 Table \ref{table:result} shows a performance comparison for the multi-label datasets.
 For RCV1 dataset it shows a slight performance gain, which shows that the benefit of slack rescaling formulation is greater when the label space is large. %Yeast has 14 labels, and RCV1 has 50 labels. The performance gain on the RCV1 dataset shows that the benefit of slack 
\begin{figure}
\centering     
\begin{subfigure}{\linewidth}
\includegraphics[width=.8\linewidth]{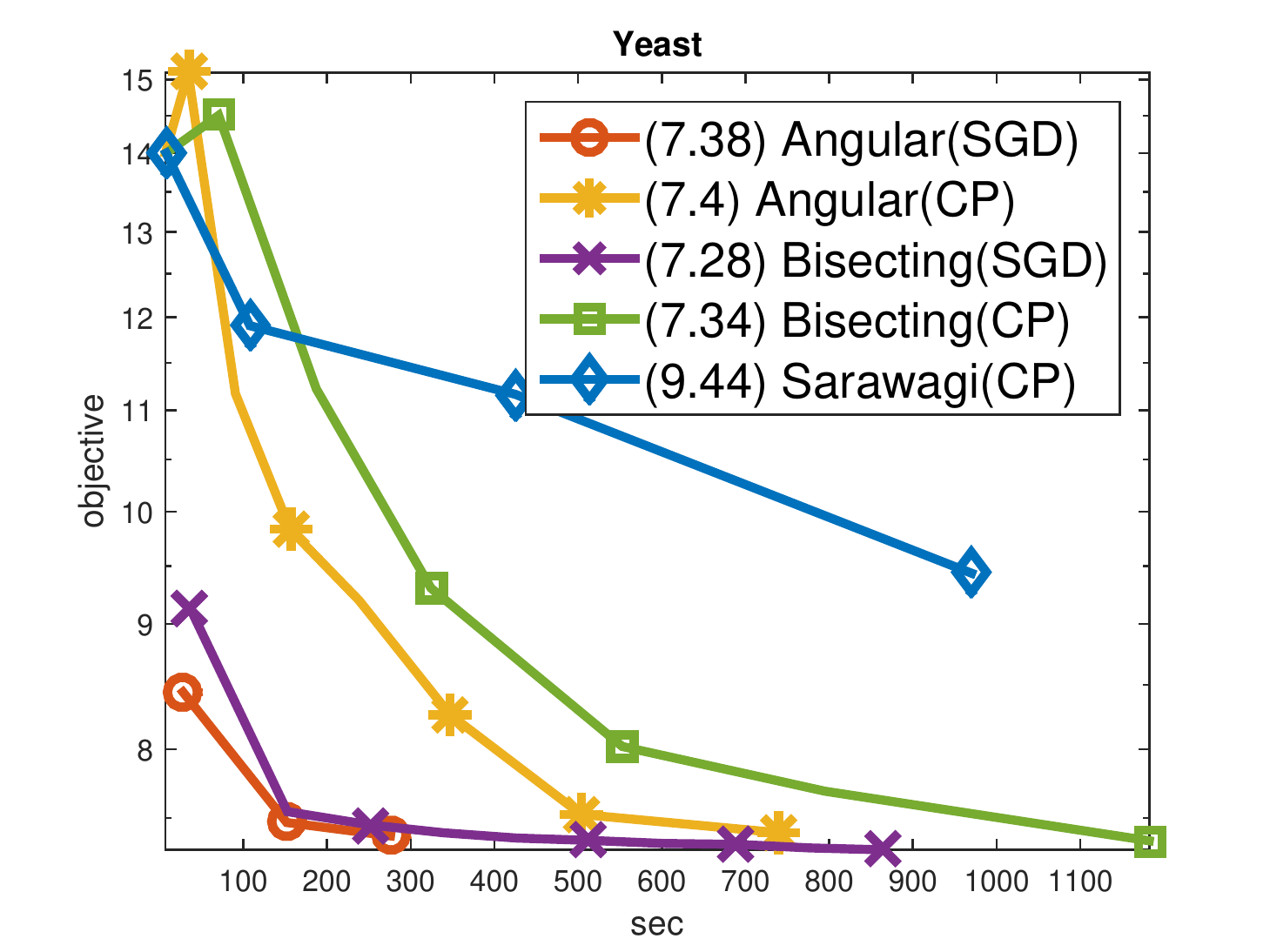}
\caption{Yeast Convergence rate}
\end{subfigure}
\begin{subfigure}{\linewidth}
\includegraphics[width=.8\linewidth]{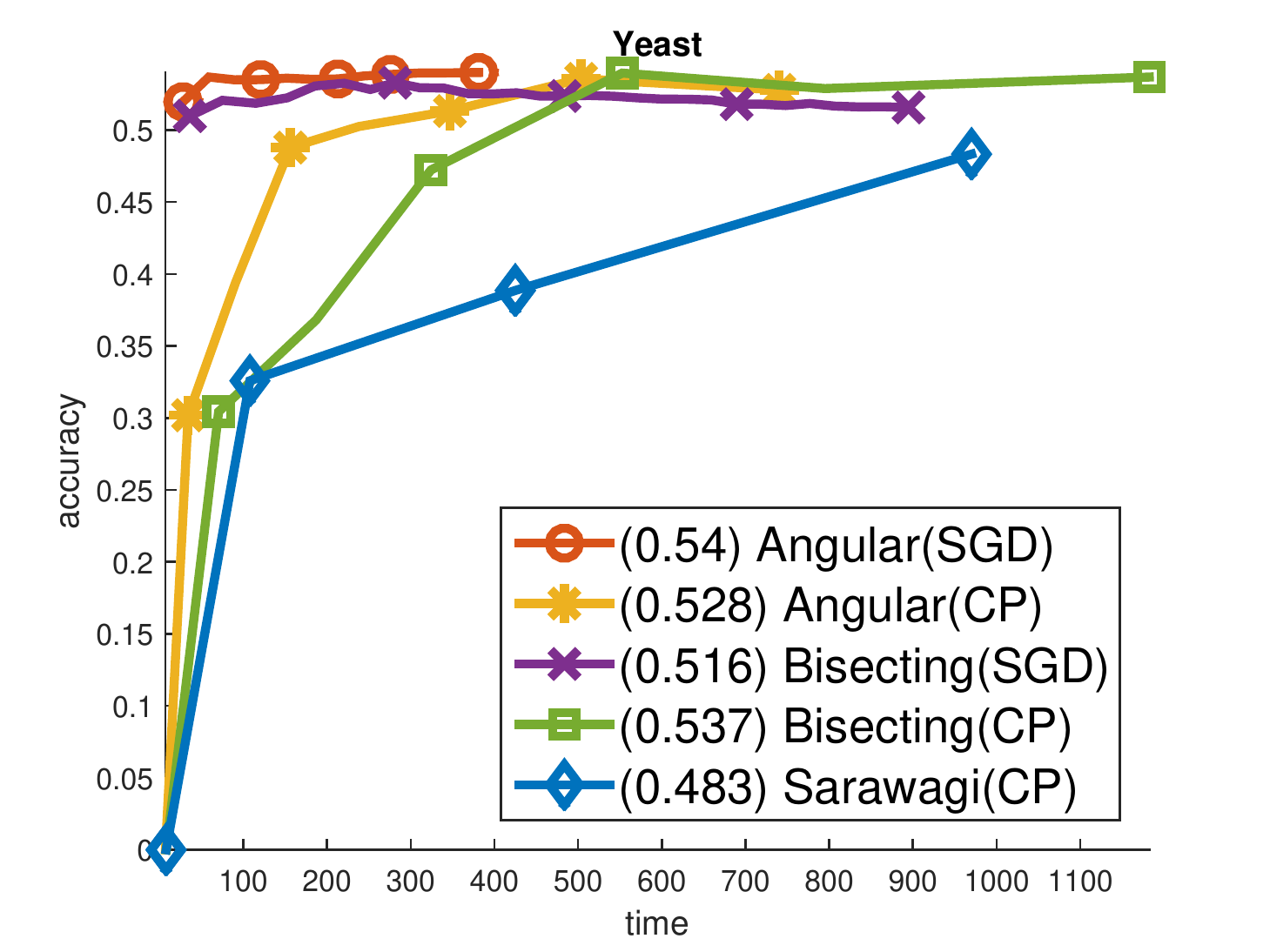}
\caption{Yeast Accuracy}
\end{subfigure}
\caption{Convergence rate and the accuracy. Angular search with SGD is significantly faster and performs  the others. }
\end{figure}
\begin{figure}
\begin{subfigure}{\linewidth}
\includegraphics[width=.8\linewidth]{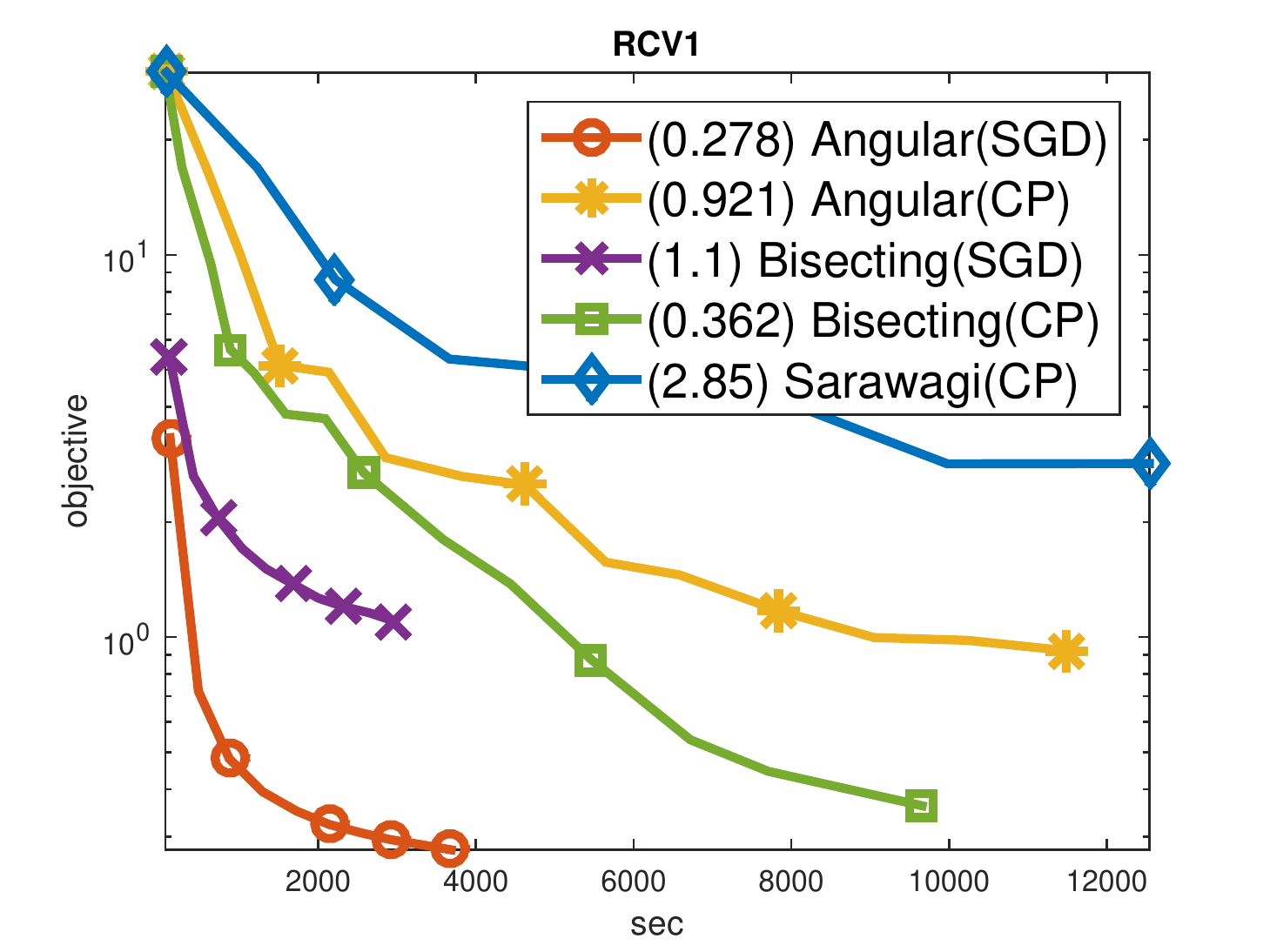}
\caption{RCV1 Convergence rate}
\end{subfigure}
\begin{subfigure}{\linewidth}
\includegraphics[width=.8\linewidth]{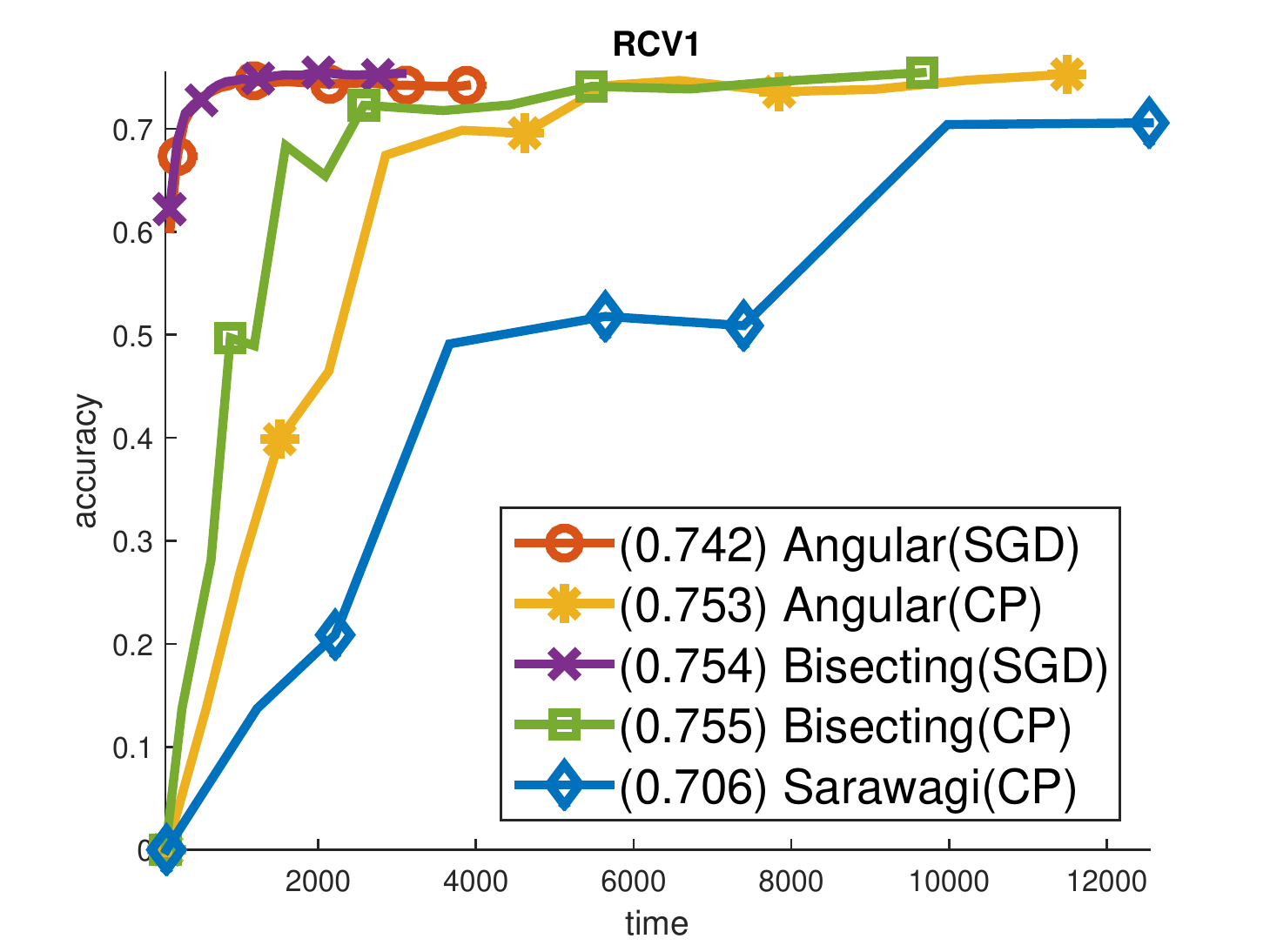}
\caption{RCV1 Accuracy}
\end{subfigure}
\caption{Convergence rate and the accuracy (continued). Angular search with SGD is significantly faster and performs  the others. }
\label{fig:convergence}
\end{figure}

\begin{table}
\centering
Yeast

    \begin{tabular}{| l  | c |c|c|c|}
    \hline
&  Acc & Label loss & MiF1 & MaF1\\
    \hline
    Slack  
      & .54  & .205 &  .661 &  .651\\
     Margin 
      &  .545& .204 &  .666 &  .654\\
     \hline
    \end{tabular}

    RCV1
    
        \begin{tabular}{| l  | c |c|c|c|}
    \hline
&  Acc & Label loss & MiF1 & MaF1\\
    \hline
    Slack  
      & .676  & .023 &  .755 &  .747\\
     Margin 
      &  .662 & .023 &  .753 &  .742\\
     \hline
    \end{tabular}
\caption{Results on Multi-label Dataset with Markov Random  Field.}\label{table:result}
\end{table}

\subsection{Hierarchical Multi-label Classification}

 We further experimented on problem of hierarchical multi-label classification \cite{cai2004hierarchical}. % \cite{choi2015normalized}.
 In hierarchical multi-label classification, each label $y$ is a leaf node in a given graph, and each label $y$ shares ancestor nodes. It can be described as a graphical model where a potential of a multi-label $Y=\{y_1,\dots,y_k\}$ is the sum of all potentials of its ancestors, i.e., $\Phi(Y)=\sum_{n\in \bigcup_{n\in Y} Anc(n)} \Phi(n)$. We extracted 
                    1500 instances with dimensionality 17944 with a graph structure of 156 nodes with 123 label from 
           SWIKI-2011. SWIKI-2011 is a multi-label dataset of wikipedia pages from 
LSHTC competition\footnote{http://lshtc.iit.demokritos.gr/}. We used 750 instances as training set, 250 instances as holdout set, and 750 instances as test set. The Hamming distance is used as label loss. We show that slack rescaling in such large label structure is tractable and outperforms margin rescaling.

  \begin{table}[h!]
\centering
    \begin{tabular}{| l  | c |c|c|c|}
    \hline
&  Acc & Label loss & MiF1 & MaF1\\
    \hline
    Slack  
      & .3798 & .0105 &  .3917 &  .3880\\
     Margin 
      &  .3327& .0110 &  .3394 &  .3378\\
     \hline
    \end{tabular}
  \caption{Result on hierarchical multi-label dataset}
\end{table}

\section{Proof of Limitation of Constrained $\lambda$-oracle}\label{sec:limitation2}
In this section, we give proof for the Theorem \ref{loracle_l2}.
\begin{reptheorem}{loracle_l2}
Any search algorithm accessing labels only through a $\lambda$-oracle with any number of linear constraints cannot find $y^*$ using less than $M$ iterations in the worst case, where $M$ is the number of labelings.
\proof
 We show this in the perspective of a game between a searcher and an oracle. At each iteration, the searcher queries the oracle with $\lambda$ and the search space denoted as $\mathcal{A}$, and the oracle reveals a label according to the query. And the claim is that with any choice of $M-1$ queries, for each query the oracle can either give an consistent label or indicate that there is no label in $\mathcal{A}$ such that after $M-1$ queries the oracle provides an unseen label $y^*$ which has bigger $\Phi$ than all previously revealed labels.%, and all $M-1$ queries are consistent given all the labels.

 Denote each query at iteration $t$ with $\lambda_t>0$ and a query closed and convex set $\mathcal{A}_t\subseteq\reals^2$% defined by some linear constraints given to the oracle
, and denote the revealed label at iteration $t$ as $y_t$. We will use   $y_t=\emptyset$ to denote that there is no label inside query space $\mathcal{A}_t$. Let $\mathcal{Y}_t=\{y_{t'}|t'<t\}$.

 Algorithm \ref{alg:M_iter} describes the pseudo code for  generating such $y_t$. The core of the algorithm is maintaining a rectangular area $\mathcal{R}_t$ for each iteration $t$ with following properties. Last two properties are for $y_t$. 

\begin{enumerate}
 \item $\forall t'<t,\forall y\in \mathcal{R}_t, \Phi(y)>\Phi(y_{t'})$.
 \item $\forall t'<t,\forall y\in \mathcal{R}_t\cap \mathcal{A}_{t'}, h(y_{t'})+\lambda_{t'} g(y_{t'})> h(y)+\lambda_{t'} g(y)$.
 \item  $\mathcal{R}_t\subseteq \mathcal{R}_{t-1}$.
 \item $\mathcal{R}_t$ is a non-empty open set.
 \item $y_t\in \mathcal{R}_t\cap \mathcal{A}_t$
 \item $y_t=\argmax_{y\in \mathcal{Y}_{t}\cap \mathcal{A}_t} h(y)+\lambda_{t} g(y)$.
 \end{enumerate}

Note that if these properties hold till iteration $M$, we can simply set $y^{*}$ as any label in $\mathcal{R}_M$ which proves the claim. 

First, we show that property 4 is true. $\mathcal{R}_0$ is a non-empty open set.  Consider iteration $t$, and assume $\mathcal{R}_{t-1}$ is a non-empty open set. Then $\tilde{R}$ is an open set since $\mathcal{R}_{t-1}$ is an open set.
There are two unknown functions,  $Shrink$ and $FindRect$. 
For open set $A\subseteq\reals^2,y\in \reals^2$, let $Shink(A,y,\lambda)=A-\{y'|\Phi(y')\le\Phi(y) $ or $h(y')+\lambda g(y')\ge h(y)+\lambda g(y)\}$. Note that $Shrink(A,y,\lambda)\subseteq A$, and $Shrink(A,y,\lambda)$  is an open set. Assume now that there exists a $y$ such that $Shrink(\mathcal{R}_{t-1},y,\lambda_t)\ne \emptyset$ and $FindPoint(\mathcal{R}_{t-1},\lambda_t)$  returns such $y$. Function $FindPoint$ will be given later. 
$FindRect(A)$ returns an open non-empty rectangle inside $A$. Note that $Rect(A)\subseteq\ A$, and since input to $Rect$ is always a non-empty open set, such rectangle exists. Since $\mathcal{R}_0$ is non-empty open set, $\forall t,\mathcal{R}_t$ is a non-empty open set.

Property 3 and 5 are easy to check. Property 1 and 2 follows from the fact that $\forall t\in \{t|y_t\ne \emptyset\},\forall t'>t, \mathcal{R}_{t'}\subseteq Shrink(\mathcal{R}_{t-1},y_t,\lambda_{t-1})$.

Property 6 follows from the facts that if $\mathcal{Y}_{t-1}\cap\mathcal{A}_t\neq \emptyset$, $\tilde{\mathcal{R}}=0\implies \mathcal{R}_{t-1} \subseteq \{y| h(y)+\lambda_{t} g(y)>  h(\tilde{y})+\lambda_{t} g(\tilde{y}) $ and $ y\in\mathcal{A}_t\}$, otherwise $\mathcal{Y}_{t-1}\cap\mathcal{A}_t= \emptyset$, and $\mathcal{R}_{t-1} \subseteq \mathcal{A}_t$.

\begin{algorithm}
  \caption{Construct a consistent label set $\mathcal{Y}$.   }
\label{alg:M_iter}
  \begin{algorithmic}[1]
    \INPUT {$\{ \lambda_t, \mathcal{A}_t\}_{t=1}^{M-1},  \lambda_t>0, \mathcal{A}_t\subseteq\reals^2, \mathcal{A}_t $is closed and convex region. 
\OUTPUT {$\{y_t\in\reals^2\}_{t=1}^{t=M-1},y^*\in \reals^2$} 
%and $\lambda$-oracle with a convex constraint oracle $\mathcal{O}:\mathbb{R}_{+}\times %\{a|a\subseteq \reals^2, a \text{ is convex}\}\rightarrow \reals^2$ such %that $\mathcal{O}(\lambda,\mathcal{A})=\argmax_{}
%$ initial step size $\alpha>1$.
}
%    \OUTPUT{$\hat{y} \in \mathcal{Y}.$}
    \INIT  {$\mathcal{R}_0=\{(a,b)|0< a , 0<
  b\}$},$\mathcal{Y}_{0}=\emptyset$.
    \For {$t=1,2,\dots,M-1$ }
%     \State $U_{t}=R_{t}\cap \mathcal{A}_t  $
%     \State $B_t=\mathcal{Y}\cap\mathcal{A}_t$
     \If{$ \mathcal{Y}_{t-1}\cap\mathcal{A}_t= \emptyset$}  % \Comment{If there is no label in the query area $\mathcal{A}_t$.}
\State $\tilde{y}=\argmax_{y\in \mathcal{Y}_t}  h(y)+\lambda_{t} g(y).$
\State $\tilde{\mathcal{R}}=\mathcal{R}_{t-1} \cap \{y| h(y)+\lambda_{t} g(y)<  h(\tilde{y})+\lambda_{t} g(\tilde{y}) $ or $ y\notin\mathcal{A}_t\}$.
     \Else 
     \State $\tilde{y}=\emptyset$, $\tilde{\mathcal{R}}=\mathcal{R}_{t-1}-\mathcal{A}_t$.
     \EndIf
     \If{$\tilde{\mathcal{R}}\ne\emptyset$}
     \State $y_t=\emptyset$. $\mathcal{R}_t=FindRect(\tilde{\mathcal{R}})$
     \Else 
%     \Comment{$\tilde{\mathcal{R}}=\emptyset$ and $\mathcal{R}_{t-1}\subseteq \mathcal{A}_t$ and $\forall y\in R_{t-1}, y'\in \mathcal{Y}_{t-1}, a$.}
     \State $y_t=FindPoint(\mathcal{R}_{t-1},\lambda_t)$.
     \State $\mathcal{R}_t=FindRect(Shrink(\mathcal{R}_{t-1},y_t,\lambda_t)).$
     \EndIf
     \If{$y_t\ne \emptyset$}
     \State $\mathcal{Y}=\mathcal{Y}\cup\{y_t\}$.
     \EndIf
    \EndFor
    \State Pick any $y^*\in \mathcal{R}_{M-1}$
  \end{algorithmic}
\end{algorithm}
 
 $FindPoint(A,\lambda)$ returns any $y\in A -\{y\in \reals^2|\lambda y_2=y_1\}$. Given input $A$ is always an non-empty open set, such $y$ exists. $Shrink(\mathcal{R}_{t-1},y,\lambda_t)\ne \emptyset$ is ensured from the fact that two boundaries, $c=\{y'|\Phi(y')=\Phi(y)\}$ and $d=\{h(y')+\lambda g(y')= h(y)+\lambda g(y)\}$ meets at $y.$ Since $c$ is a convex curve, $c$ is under $d$ on one side. Therefore the intersection of set above $c$ and below $d$ is non-empty and also open.
\qed
\end{reptheorem}

\section{Proof for Angular Search}\label{angular_search_proof}
We first introduce needed notations. 
$\partial^{\bot}(a)$ be the perpendicular  slope of $a$, i.e., $\partial^{\bot}(a)=-\frac{1}{\partial(a)}=-\frac{a_1}{a_2}$. For $\mathcal{A}\subseteq\reals^2$, let label set restricted to $A$ as $\mathcal{Y}_A=\mathcal{Y}\cap A,  $ and 
%\begin{align*}\label{eq:general_lambda_oracle}
$y_{\lambda,A}=\mathcal{O}(\lambda,A)=\argmax_{y\in\mathcal{Y}, y\in A} h(y)+\lambda g(y)$
$=\argmax_{y\in\mathcal{Y}_A} [y]_1+\lambda [y]_2$.
%\end{align*}
Note that if $A=\reals^2,$ $y_{\lambda,\reals^2}=y(\lambda)$.
For $P,Q\in \reals^2$, define $\Lambda(P,Q)$ to be the area below the line $\overline{PQ}$, i.e., $\Lambda(P,Q)=\{y\in\reals^2|[y]_{2}-[P]_2\le\partial^{\bot}(Q-P)([y]_{2}-[P]_2)\}$.
  $\Upsilon_\lambda=\{y\in\reals^2| \Phi(y)= [y]_{1}\cdot [y]_{2}\ge\Phi(y_{\lambda,A})\}$ be the area above $C_\lambda$, and $\underline{\Upsilon}_\lambda=\{y\in\reals^2| \Phi(y)= [y]_{1}\cdot [y]_{2}\le\Phi(y_{\lambda,A})\}$ be the area below $C_\lambda$.
 
Recall the \emph{constrained $\lambda$-oracle}  defined in \eqref{eq:constrained_oracle}:
\begin{align*}
  y_{\lambda,\alpha,\beta}=\mathcal{O}_{c}(\lambda,\alpha,\beta)=\max_{y\in\mathcal{Y},\; \alpha h(y)\ge g(y),\; \beta h(y)<g(y)} \mathcal{L}_\lambda(y)  
\end{align*}
where $\alpha,\beta\in\reals_+$ and $\alpha\ge\beta>0$. Let $A(\alpha,\beta)\subseteq\reals^2$ be the restricted search space, i.e., $A(\alpha,\beta)=\{a\in\reals^2|\beta<\partial(a)\le\alpha\}$. Constrained $\lambda$-oracle reveals maximal $\mathcal{L}_\lambda$ label within restricted area defined by $\alpha$ and $\beta$. The area is bounded by two lines whose slope is $\alpha$ and $\beta$. % $\alpha$ is the slope of the linear upper bound, $[t,t\alpha],t>0$, and $\beta$ is the slope of the  linear lower bound $[t',t'\beta],t'>0$. 
%Since $\alpha$ and $\beta$ can be interpreted as an angle in the $\reals^2$,
 Define a pair $(\alpha,\beta),\alpha ,\beta\in \reals_+,\alpha\ge\beta>0$ as an {\em angle}.  The angular search recursively divides an angle into two different angles, which we call the procedure as a {\em split}. For $\alpha\ge\beta\ge0$, let $\lambda=\dfrac{1}{\sqrt{\alpha\beta}}$, $z=y_{\lambda,\alpha,\beta}$ and $z'=[\lambda [z]_2,\frac{1}{\lambda}[z]_1]$.  Let $P$ be the point among $z$ and $z'$ which has the greater slope (any if two equal), and $Q$ be the other point, i.e., if $\partial(z)\ge \partial(z')$, $P=z$ and $Q=z'$, otherwise $P=z'$ and $Q=z$.  Let $R=\left [\sqrt{\lambda[z]_1\cdot [z]_2}\;\; \sqrt{\frac{1}{\lambda}[z]_1\cdot [z]_2}\;\right ]$. Define split$(\alpha,\beta)$ as a procedure divides $(\alpha,\beta)$ %into $(\alpha^+,\gamma^+)=(\min\{\partial(P),\partial(R)\},\max\{\partial(P),\partial(R)\})$
into two angles $(\alpha^+,\gamma^+)=(\partial(P),\partial(R))$
and  $(\gamma^+,\beta^+)=(\partial(R),\partial(Q))$.
% where $\lambda=\dfrac{1}{\sqrt{\alpha\beta}}$, $P=y_{\lambda,\alpha,\beta}$, $Q=[\lambda [P]_2\;\frac{1}{\lambda}[P]_1]$,  $R=\left [\sqrt{\lambda[P]_1\cdot [P]_2}\;\; \sqrt{\frac{1}{\lambda}[P]_1\cdot [P]_2}\;\right ]$. 

First, show that $\partial(P)$ and $\partial(Q)$ are in between $\alpha$ and $\beta$, and $\partial(R)$ is  between $\partial(P)$ and $\partial(Q)$. 
\begin{lemma} \label{lemma:angle_range}For each split$(\alpha,\beta)$,
\begin{align*}
\beta\le&\partial(Q)\le\partial(R)\le\partial(P)\le\alpha 
%\min\{\partial(P),\partial(Q)\}\le&\partial(R)\le\max\{\partial(P),\partial(Q)\}.
\end{align*}
\proof
$\beta\le\partial(z)\le\alpha$ follows from the definition of constrained $\lambda$-oracle in \eqref{eq:constrained_oracle}.

$\partial(z')=\dfrac{1}{\lambda^2\partial(z)}=\dfrac{\alpha\beta}{\partial(z)}\implies \beta\le\partial(z')\le\alpha\implies\beta\le\partial(Q)\le\partial(P)\le\alpha$.

$\partial(Q)\le\partial(R)\le\partial(P)\iff$
 $\min\left \{\partial(z),\dfrac{1}{\lambda^2\partial(z)}\right \}\le \dfrac{1}{\lambda}\le\max\left \{\partial(z),\dfrac{1}{\lambda^2\partial(z)}\right \}$ from $\forall a,b\in \reals_+,b\le a\implies  b\le\sqrt{ab}\le a$. \qed
\end{lemma}
After each split, the union of the divided angle ($\alpha^+,\gamma$) and ($\gamma,\beta^{+}$) can be smaller than angle $(\alpha,\beta)$. However, following lemma shows it is safe to use ($\alpha^+,\gamma$) and ($\gamma,\beta^{+}$) when our objective is to find $y^*$.
\begin{lemma} \label{lem:shrink_angel}
\begin{align*}
\forall a\in\mathcal{Y}_{A(\alpha,\beta)}, \Phi(a)> \Phi(y_{\lambda,\alpha,\beta}) \implies \beta^+<\partial(a)<\alpha^+ 
\end{align*}
\proof From lemma \ref{lem:restricted_space}, $\mathcal{Y}_{A(\alpha,\beta)}\subseteq\Lambda(P,Q)$. Let $U=\{a\in\reals^2|\partial(a)\ge \alpha_+=\partial(P)\}$,  $B=\{a\in\reals^2|\partial(a)\le \beta_+=\partial(Q)\}$,  and two contours of function $C=\{a\in\reals^2|\Phi(a)= \Phi(y_{\lambda,\alpha,\beta})\}$,  $S=\{a\in\reals^2|\mathcal{L}_\lambda(a)=\mathcal{L}_\lambda(y_{\lambda,\alpha,\beta})\}$. $S$ is the upper bound of $\Lambda(P,Q)$, and  $C$ is the upper bound of $\underline{C}=\{a\in \reals^2|\Phi(a)\le\Phi(y_{\lambda,\alpha,\beta})\}$. $P$ and $Q$ are the intersections of $C$ and $S$. For area of $U$ and $B$, $S$ is under $C$, therefore, $\Lambda(P,Q)\cap U\subseteq \underline{C} $, and $\Lambda(P,Q)\cap B\subseteq \underline{C}$. It implies that $\forall a\in (\Lambda(P,Q)\cap U )\cup (\Lambda(P,Q)\cap B) \implies \Phi(a)\le \Phi(y_{\lambda,\alpha,\beta})$. And the lemma follows from $A(\alpha,\beta)=U\cup B\cup \{a\in \reals^2| \beta^+<\partial(a)<\alpha^+\}$.\qed
\end{lemma}

We  associate a quantity we call a {\em capacity} of an angle, which is used to prove the suboptimality of the algorithm. For an angle $(\alpha,\beta)$, the capacity of an angle $v(\alpha,\beta)$ is
\begin{align*}
% v(\alpha,\beta):=\max \left \{\sqrt{\dfrac{\alpha}{\beta}}, \sqrt{\dfrac{\beta}{\alpha}}\right
  v(\alpha,\beta):=\sqrt{\frac{\alpha}{\beta}}
\end{align*}
Note that from the definition of an angle, $v(\alpha,\beta)\ge1$. First show that the capacity of angle decreases exponentially for each split. 

\begin{lemma} \label{lem_split_decrease}
% For any $s\in\mathcal{S}$, let $v(s)=\sqrt{\dfrac{\partial(U(s))}{\partial( L(s))}},\lambda=\dfrac{1}{\sqrt{\partial(U(s))\partial(L(s))}}, P=\mathcal{O}_c(\lambda, s)$, $Q=[\lambda [P]_2\;\frac{1}{\lambda}[P]_1]$, and $R=\left [\sqrt{\lambda[P]_1\cdot [P]_2}\;\; \sqrt{\frac{1}{\lambda}[P]_1\cdot [P]_2}\;\right ]$. Then,
 For any angle $(\alpha,\beta)$ and its split $(\alpha^+,\gamma^+)$ and  $(\gamma^+,\beta^+)$, 
\begin{align*}
 v(\alpha,\beta)\ge v (\alpha^+,\beta^+)=v(\alpha^+,\gamma^+)^{2}=v(\gamma^+,\beta^+)^{2}
\end{align*}
\proof Assume $\partial(P)\ge \partial(Q)$ (the other case is follows the same proof with changing the role of $P$ and $Q$), then $\alpha^+=\partial(P)$ and $\beta^+=\partial(Q)$. $\partial(Q)=\dfrac{1}{\lambda^2\partial(P)}=\dfrac{\alpha\beta}{\partial(P)},v$ $(\alpha^+,\beta^+)=v(\partial(P),\partial(Q))=\lambda \partial(P)=\dfrac{\partial(P)}{\sqrt{\alpha\beta}}.$ Since $\alpha$ is the upper bound and $\beta$ is the lower bound of $\partial(P)$,  $\sqrt{\dfrac{\beta}{\alpha}} \le v(\partial(P),\partial(Q))\le \sqrt{\dfrac{\alpha}{\beta}}$.  Last two equalities in the lemma are from $v(\partial(P),\partial(R))=v(\partial(R),\partial(Q))=\sqrt{\frac{\partial(P)}{\sqrt{\alpha\beta}}}$  by plugging in the coordinate of $R$.\qed
 \end{lemma}
 %There exists a the suboptimality bound of an angle $(\alpha,\beta)$ with $\lambda=\dfrac{1}{\sqrt{\alpha\beta}}$.
\begin{lemma}\label{lem:suboptimal__} Let $\mathcal{B}(a)=  \dfrac{1}{4}\left(a+\dfrac{1}{a}\right)^2$ . The suboptimality bound of an angle $(\alpha,\beta)$ with $\lambda=\dfrac{1}{\sqrt{\alpha\beta}}$ is 
\begin{align*} 
\dfrac{\max_{y\in \mathcal{Y}_{A(\alpha,\beta)}}\Phi(y)}{\Phi(y_{\lambda,\alpha,\beta})}
\le 
% \dfrac{1}{4}\left(v(\alpha,\beta)+\dfrac{1}{v(\alpha,\beta)}\right)^2=
 \mathcal{B}(v(\alpha,\beta)).
%\dfrac{\Phi(y)}{\Phi(y_{\lambda,\alpha,\beta})}\le 
\end{align*}
\proof From lemma \ref{lem:restricted_space}, $\mathcal{Y}_{A(\alpha,\beta)}\subseteq\Lambda(P,Q)=\Lambda(z,z').$  Let $\partial(z)=\gamma$. From \ref{lemma:angle_range}, $\beta\le\gamma\le \alpha$.
Let  $m=\argmax_{a\in \Lambda(z,z')}\Phi(a)$. $m$ is on line $\overline{zz'}$ otherwise we can move $m$ increasing direction of each axis till it meets the boundary $\overline{zz'}$ and  $\Phi$ only increases, thus $m=tz+(1-t)z'$. $\Phi(m)=\max_{t}\Phi(tz+(1-t)z')$. $\dfrac{\partial\Phi(tz+(1-t)z')}{\partial t}=0 \implies t=\dfrac{1}{2}$. $m=\frac{1}{2}[z_1+\lambda z_2\;\; z_2+\frac{z_1}{\lambda}]$.
\begin{align*}%\label{eq:in_suboptimal}
\dfrac{\max_{y\in \mathcal{Y}_{A(\alpha,\beta)}}\Phi(y)}{\Phi(y_{\lambda,\alpha,\beta})}=\dfrac{1}{4}\left (\sqrt{\dfrac{z_1}{\lambda z_2}}+\sqrt{\dfrac{\lambda z_2}{z_1}}\right)^2\\=\dfrac{1}{4}\left (\sqrt{\dfrac{\sqrt{\alpha\beta} }{\gamma}}+\sqrt{\dfrac{\gamma}{\sqrt{\alpha\beta} }}\right)^2
\end{align*} 
Since $v(a)=v\left(\frac{1}{a}\right)$ and $v(a)$ increases monotonically for $a\ge1$,
\begin{align*}
\mathcal{B}(a)\le \mathcal{B}(b)\iff \max\left \{a,\frac{1}{a}\right\} \le \max\left \{b,\frac{1}{b}\right\}
\end{align*}
If $\dfrac{\sqrt{\alpha\beta} }{\gamma}\ge \dfrac{\gamma}{\sqrt{\alpha\beta} }$, then $\dfrac{\sqrt{\alpha\beta} }{\gamma}\le \sqrt{\dfrac{\alpha}{\beta}}$ since $\gamma\ge \beta$. If $\dfrac{\gamma}{\sqrt{\alpha\beta}}\ge \dfrac{\sqrt{\alpha\beta} }{\gamma}$, then $\dfrac{\gamma}{\sqrt{\alpha\beta}}\le \sqrt{\dfrac{\alpha}{\beta}}$ since $\gamma\le \alpha$. Therefore, $\dfrac{\max_{y\in \mathcal{Y}_{A(\alpha,\beta)}}\Phi(y)}{\Phi(y_{\lambda,\alpha,\beta})}
=\mathcal{B}\left (\dfrac{\sqrt{\alpha\beta} }{\gamma}\right)\le \mathcal{B}(v(\alpha,\beta))$.\qed
\end{lemma}

 Now we can prove the theorems.

\begin{reptheorem}{thm:angular_optimal} Angular search described in algorithm \ref{alg:angular} finds optimum $y^*=\argmax_{y\in \mathcal{Y}} \Phi(y) $ at most $t=2M+1$ iteration where $M$ is the number of the labels.
\proof
 Denote $y_{t}, \alpha_t, \beta_t, z_t,z_t', K_t^1,$ and $K_t^2$ for $y,\alpha,\beta,z,z',K^1,$ and $K^2$  at iteration $t$ respectively.  $\mathcal{A}(\alpha_t,\beta_t)$ is the search space at each iteration $t$. At the first iteration $t=1$, the search space contains all the labels with positive $\Phi$, i.e., $\{y|\Phi(y)\ge 0 \}\subseteq\mathcal{A}(\infty,0)$. At iteration $t>1$, firstly, when $y_t=\emptyset$,  the search area $\mathcal{A}(\alpha_t,\beta_t)$ is removed from the search since  $y_t=\emptyset$ implies there is no label inside $\mathcal{A}(\alpha_t,\beta_t)$. Secondly, when $y_t\neq\emptyset$, $\mathcal{A}(\alpha_t,\beta_t)$ is dequeued, and $K_t^1$ and $K_t^2$ is enqueued. %Note that $\mathcal{A}(K_t)=\mathcal{A}(K_t^1)\cup \mathcal{A}(K_t^2)$.
% Therefore, the search area $\mathcal{A}(D_t)-\mathcal{A}(K_t)$ is removed since $\mathcal{A}(K_t)=\mathcal{A}(K_t^1)\cup \mathcal{A}(K_t^2)$. See the illustration in Figure \ref{fig:split_cor}. Let $P'$ and $Q'$ be the intersection of the line $\overline{PQ}$ with the $U$ and $L$ respectively. Since there are no labels above $\overline{PQ}$, we are effectively only removing label inside triangle $\overline{OP'P}$ and triangle $\overline{OQ'Q}$. However, all labels in two triangles has less $\Phi$ then $y_t$ since both triangle is below contour $C=\{y'|\Phi(y')=\Phi(y_t)\}$. %  is the upper bound of the labels inside $\mathcal{A}(D_t)$ the searching ends without splitting only  when there is no label within $U$ and $L$. Initially all labels are within $U$ and $L$. At each split, at line 17 and 18, the area above $\overline{OP}$ and below $\overline{OQ}$ does not contain any labels $\{y'|\Phi(y')\ge\Phi(y)\}$ since there exists no labels above $\overline{PQ}$, and all the points above $\overline{OP}$ and $Q$ are intersection of the upper bound and contour of points with $\{y'|\Phi(y')=\Phi(y)\}$. 
From lemma \ref{lem:shrink_angel}, at every step, we are ensured that do not loose $y^*$. % such that $\Phi(y')\ge\Phi(y)$. 
  By using strict inequalities in the constrained oracle with valuable $s$, we can ensure $y_t$ which oracle returns is an unseen label. Note that split only happens if a label is found, i.e., $y_t\ne\emptyset$. Therefore, there can be only $M$ splits, and each split can be viewed as a branch in the binary tree, and the number of queries are the number of nodes. Maximum number of the nodes with $M$ branches are $2M+1$. \qed
\end{reptheorem}

\begin{reptheorem}{thm:angular_suboptimality} In angular search described in algorithm \ref{alg:angular}, at iteration $t$, 
\begin{align*}
\dfrac{\Phi(y^*)}{\Phi(\hat{y})}\le (v_1)^{\frac{4}{t+1}}
\end{align*}
 where $v_1=\max \left \{\dfrac{\lambda_0}{\partial(y_1)}, \dfrac{\partial(y_1)}{\lambda_0}\right\}$, $\lambda_0$ is the initial $\lambda$ used, and $y_1$ is the first label returned by constrained $\lambda$-oracle.
\proof
After $t\ge2^{r}-1$ iteration as in algorithm \ref{alg:angular} where $r$ is an integer, for all the angle $(\alpha,\beta)$ in the queue $Q$, $v(\alpha,\beta)\le (v_1)^{2^{1-r}}$. 
%Let $K_r,\lambda_{r}$ be a $K,\lambda$ in the algorithm \ref{alg:angular} at depth $r$. 
This follows from the fact that since the algorithm uses the depth first search, after $2^{r}-1$ iterations all the nodes at the search is at least $r$. At each iteration, for a angle, the capacity is square rooted from the lemma \ref{lem_split_decrease}, and the depth is increased by one. 
\begin{comment} 
as for each split.   We will denote $D_r,K_r,K^1_r,$ and $K^2_r$ to denote $D,K,K^1$, and $K^2$ at depth $r$ correspondingly. At $r=1$,  holds true since $K=(\lambda_{},\overline{OP},\overline{OQ})$ is a standard split with $v(K)=\sqrt{\dfrac{\lambda g(y_{y_{\lambda_0,[0\; \infty],[\infty \; 0]}})}{h(y_{\lambda_0,[0\; \infty],[\infty \; 0]})}}=v_1$. 
%Let $K^{1}_1=(\bot(\overline{PR}),\overline{OP},\overline{OR})$ and $K^{2}_1=(\bot(\overline{RQ)},\overline{OR},\overline{OQ})$ %be the two split of $K_1$. And plugging in the coordinates, $v(K^{1}_1)=v(K^{2}_1)=\sqrt{v_0}$. %Therefore, theorem holds true for $r=1$.
 Assume for $r=k$, the statement holds.   $D_{r},$ is a standard split since we only enqueue normal splits except at the root. \todo{write down lemma} From lemma, $v(D_r)\ge v(K_r)$. And by plugging in the coordinates,  $v(K^{1}_r)=v(K^{2}_r)=\sqrt{v(K_r)}$, and $v(D_{r+1})=v(K^{1}_r)=v(K^{2}_r)$ since $V(K^1_r)$ or $V(K^2_r)$ will be dequeue at depth $r+1$. Therefore, $v(S_r)\le v_1^{\dfrac{1}{2^{r-1}}}$. \end{comment}
And the theorem follows from the fact that after $t\ge 2^{r}-1$ iterations, all splits are at depth $r'\ge r$, and at least one of the split contains the optimum with suboptimality bound with lemma \ref{lem:suboptimal__}. Thus,
\begin{align*}
\dfrac{\Phi(y^*)}{\Phi(\hat{y})}\le \mathcal{B}\left ((v_1)^{2^{1-r}} \right ) < (v_1)^{2^{2-r}} \le(v_1)^{\frac{4}{t+1}}
\end{align*}
\qed
\end{reptheorem}
\begin{reptheorem}{thm:angular_running_time}  Assuming $\Phi(y^*)>\phi$, angular search described in algorithm \ref{alg:angular} with $\lambda_0=\dfrac{\hat{G}}{\hat{H}}, \alpha_0=\dfrac{\hat{G}^2}{\phi},\beta_0=\dfrac{\phi}{\hat{H}^2}$, finds $\epsilon$-optimal solution, $\Phi(y)\ge(1-\epsilon)\Phi(y^*)$, in $T$ queries and $O(T)$ operations where $T=4\log\left(\dfrac{\hat{G}\hat{H}}{\phi}\right)\cdot \dfrac{1}{\epsilon}$, and $\delta$-optimal solution, $\Phi(y)\ge\Phi(y^*)-\delta$,  in $T'$ queries and $O(T')$ operations where $T'= 4\log\left(\dfrac{\hat{G}\hat{H}}{\phi}\right)\cdot \dfrac{\Phi(y^*)}{\delta}$.
\proof $\Phi(y^*)>\phi\Leftrightarrow  \dfrac{\phi}{\hat{H}^2}  <\dfrac{g(y^*)}{h(y^*)}=\partial({y^*})< \dfrac{\hat{G}^2}{\phi}$. $v_1=\max \left \{\dfrac{\lambda_0}{\partial(y_1)}, \dfrac{\partial(y_1)}{\lambda_0}\right\}$ from Theorem \ref{thm:angular_suboptimality}. Algorithm finds $y^*$ if $\beta\le \partial({y^*})\le \alpha$, thus set  $\alpha= \dfrac{\hat{G}^2}{\phi}$ and $\beta=\dfrac{\phi}{\hat{H}^2}$. Also from the definition of constrained $\lambda$-oracle, $\beta=\dfrac{\phi}{\hat{H}^2}\le\partial(y_1)\le\alpha= \dfrac{\hat{G}^2}{\phi}$. Therefore, $v_1\le\max \left \{\dfrac{\lambda_0}{\partial(y_1)}, \dfrac{\partial(y_1)}{\lambda_0}\right\}$. And the upper bound of two terms equal when $\lambda_0=\dfrac{\hat{G}}{\hat{H}}$, then $v_1\le  \dfrac{\hat{G}\hat{H}}{\phi}$. $\delta$ bound follows plugging in the upper bound of $v_1$, and $\epsilon=\dfrac{\delta}{\Phi(y^*)}$. \qed
% $\dfrac{\lambda_0\phi}{\hat{H}^2}=\dfrac{\phi}{\lambda_0\hat{G}^2}$,
\end{reptheorem}

%\section{Angular Search Algorithm}

\section{Illustration of the angular search}\label{ap:illu}
Following Figure \ref{fig:angular_illu} illustrates Angular search.
Block dots are the labels from Figure \ref{fig:all_points}. Blue X denotes the new label returned by the oracle. Red X is the maximum point. Two straights lines are the upper bound and the lower bound used by the constrained oracle. Constrained oracle returns a blue dot between the upper and lower bounds. We can draw a line that passes blue X that no label can be above the line.  Then, split the angle into half. This process continues until the $y^*$ is found.   
\begin{figure}[H]
\centering     
\begin{subfigure}{\linewidth}
\includegraphics[width=.7\linewidth]{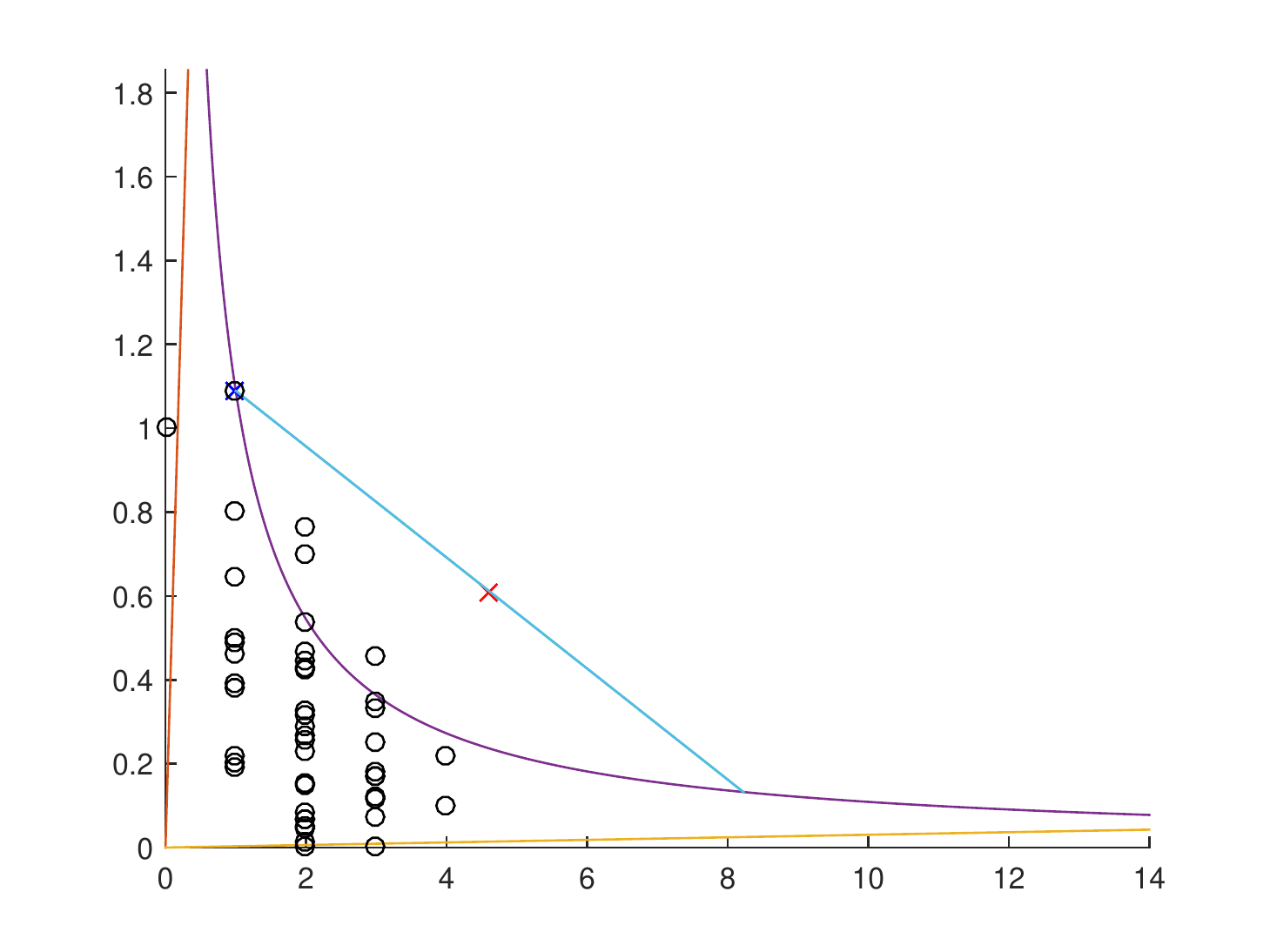}
\caption{Iteration 1}
\end{subfigure}
\begin{subfigure}{\linewidth}
\includegraphics[width=.7\linewidth]{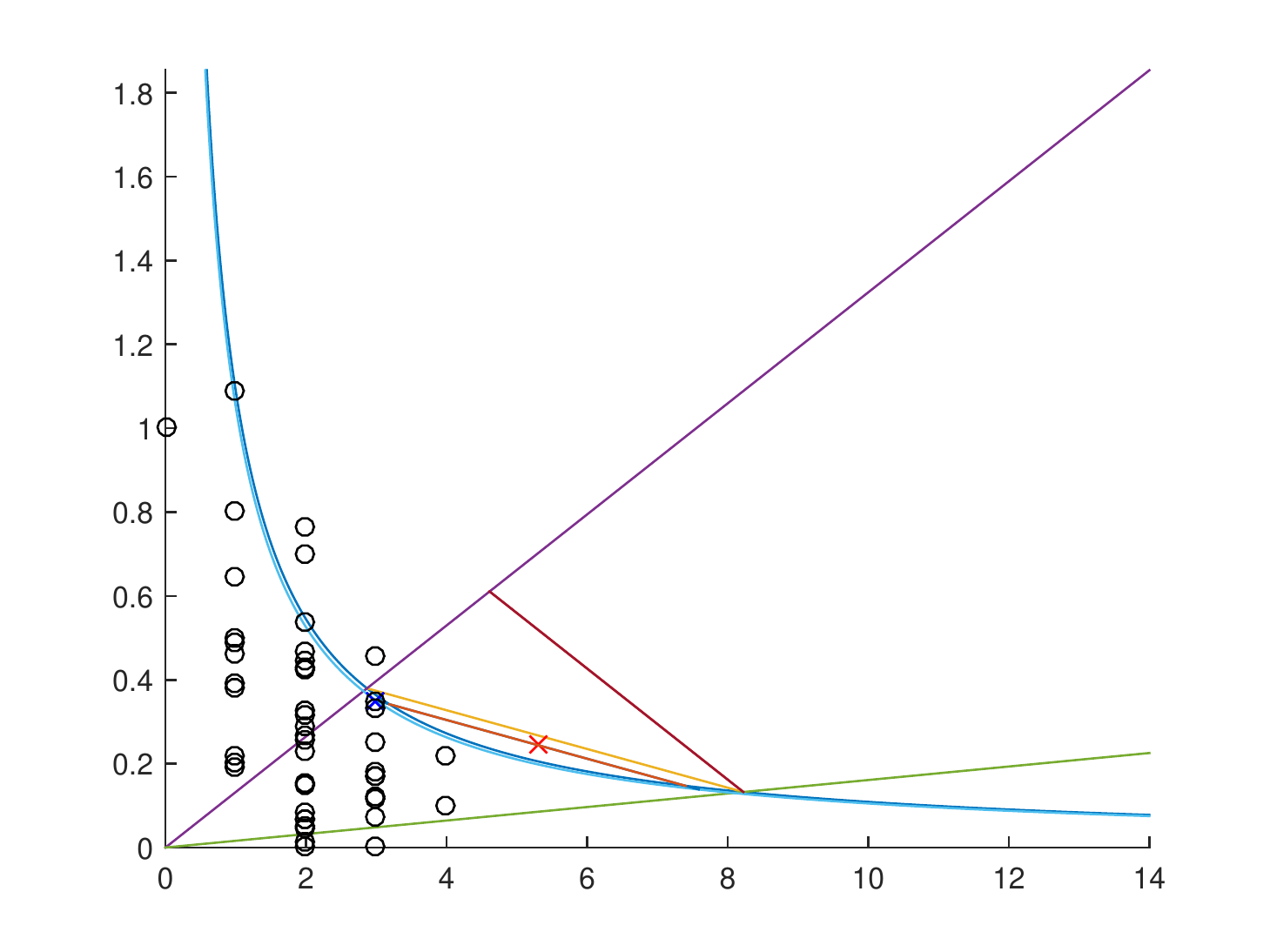}
\caption{Iteration 2}
\end{subfigure}
\caption{Illustration of the Angular search.}\label{fig:angular_illu}
\end{figure}
\begin{figure}
\begin{subfigure}{\linewidth}
\includegraphics[width=.7\linewidth]{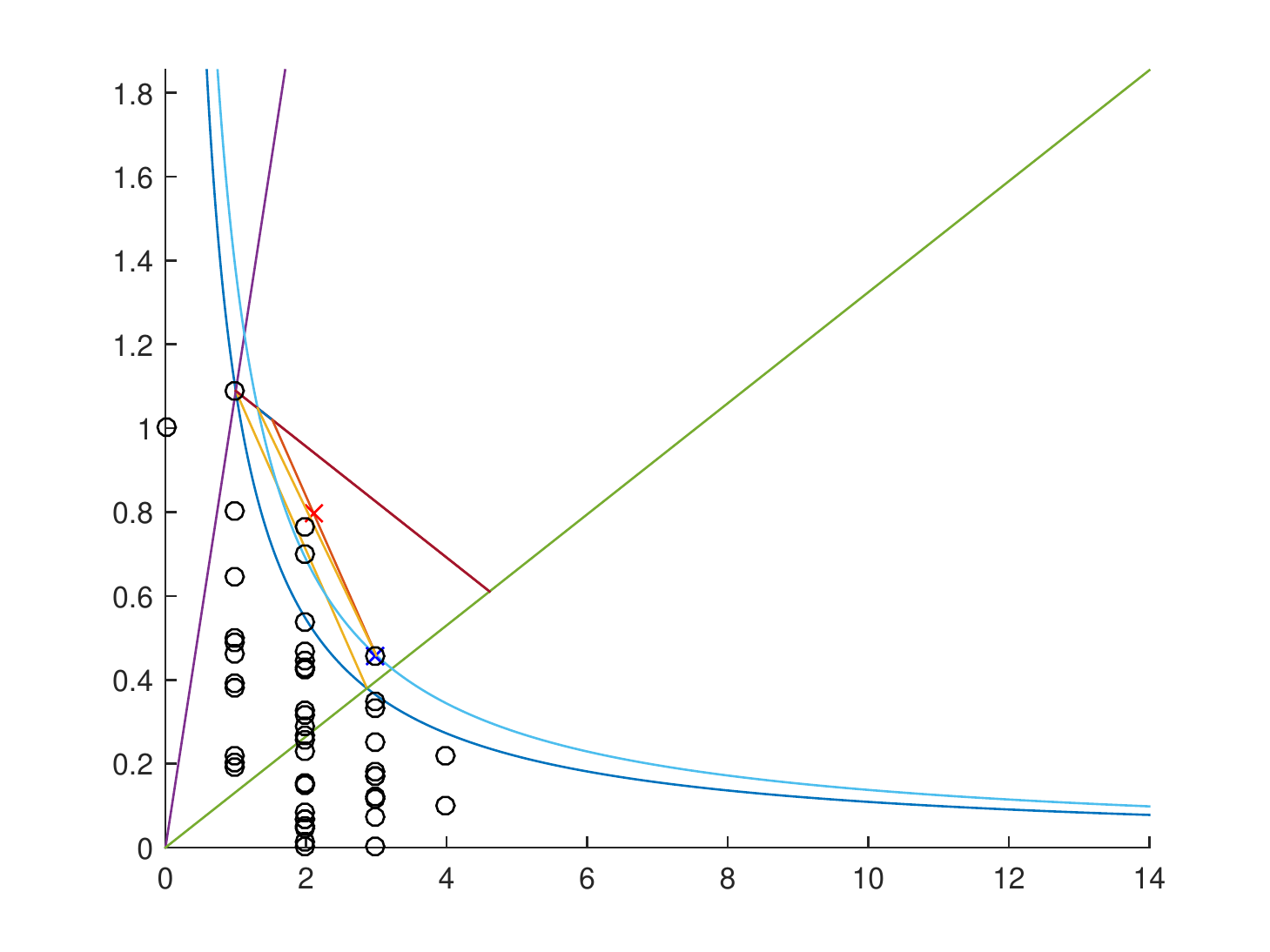}
\caption{Iteration 3}
\end{subfigure}
% \begin{subfigure}[b]{.49\linewidth}
% \includegraphics[width=\linewidth]{RCV1_f5.pdf}
% \caption{Accuracy vs iterations}
% \end{subfigure}
\caption{Illustration of the Angular search (Continued).}\label{fig:angular_illu}
\end{figure}

\section{Additional Plots from the Experiments}\label{ap:additional_plot}

\begin{figure}[H]
\centering     
%\begin{subfigure}[b]{\linewidth}
\includegraphics[width=.8\linewidth]{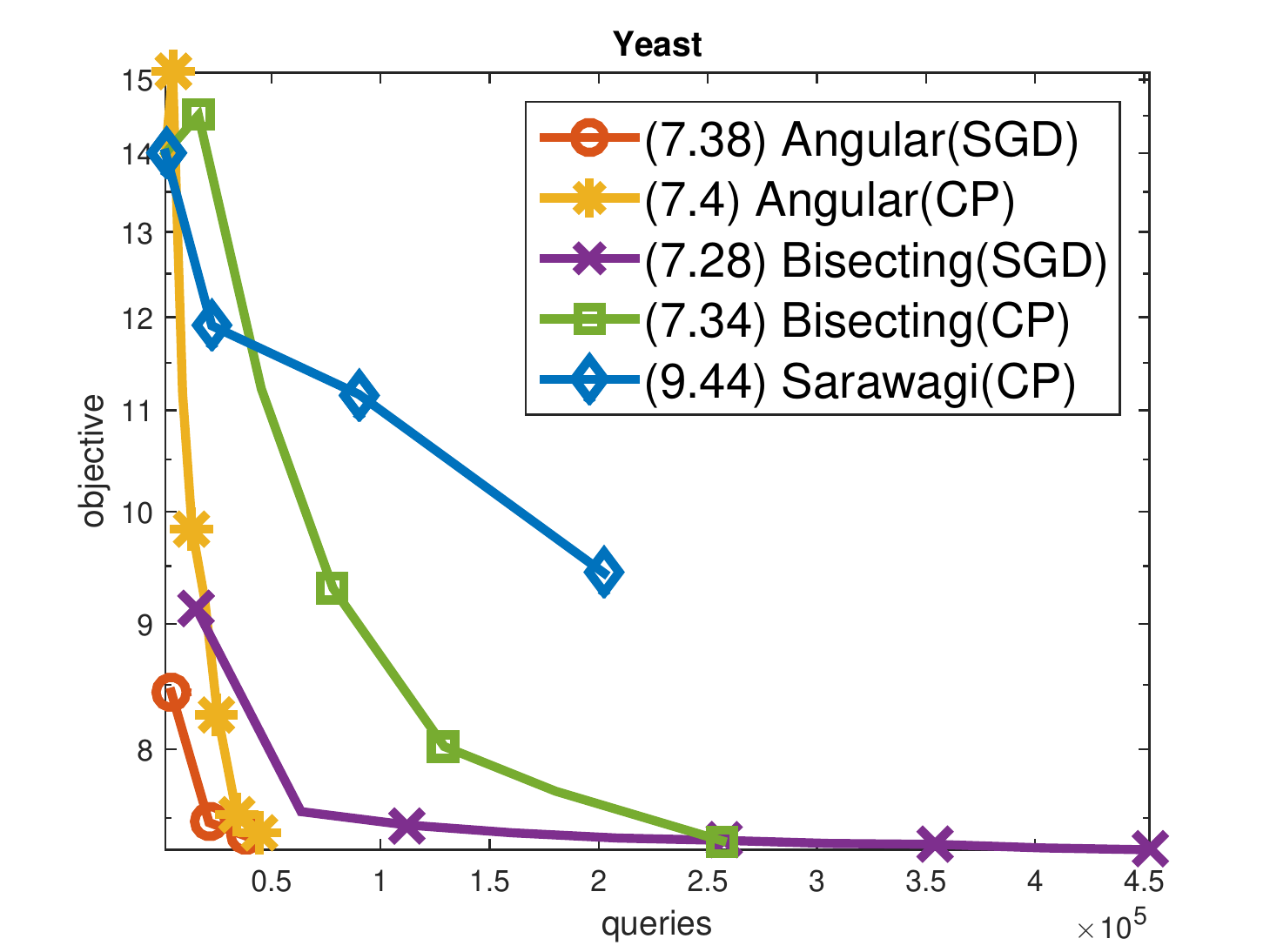}
\caption{Objective vs  queries}
\end{figure}
\begin{figure}[H]%\begin{subfigure}[b]{\linewidth}
\centering    
\includegraphics[width=.8\linewidth]{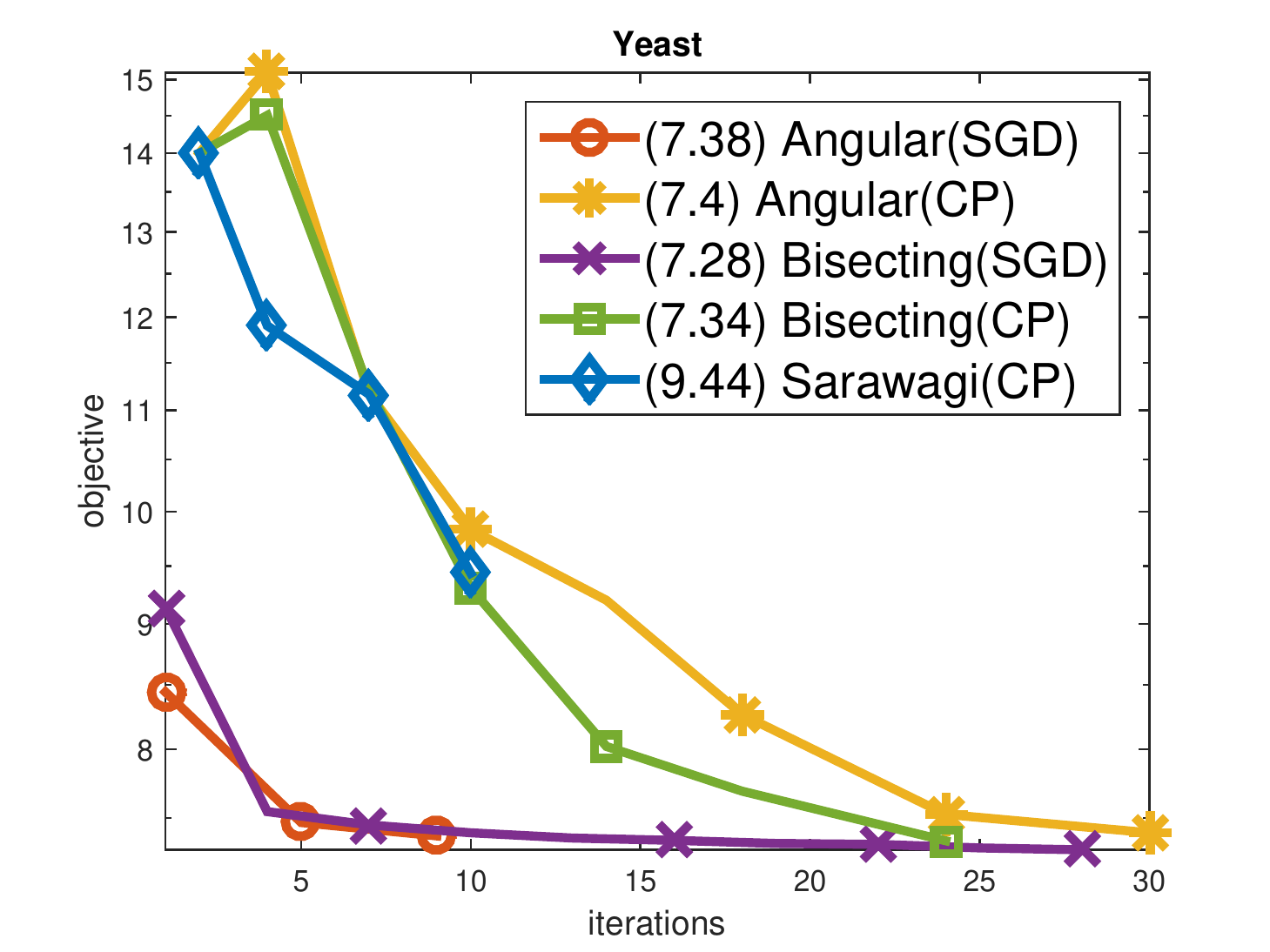}
\caption{Objective vs iterations}
\end{figure}
\begin{figure}[H]%\begin{subfigure}[b]{\linewidth}
\centering    
\includegraphics[width=.8\linewidth]{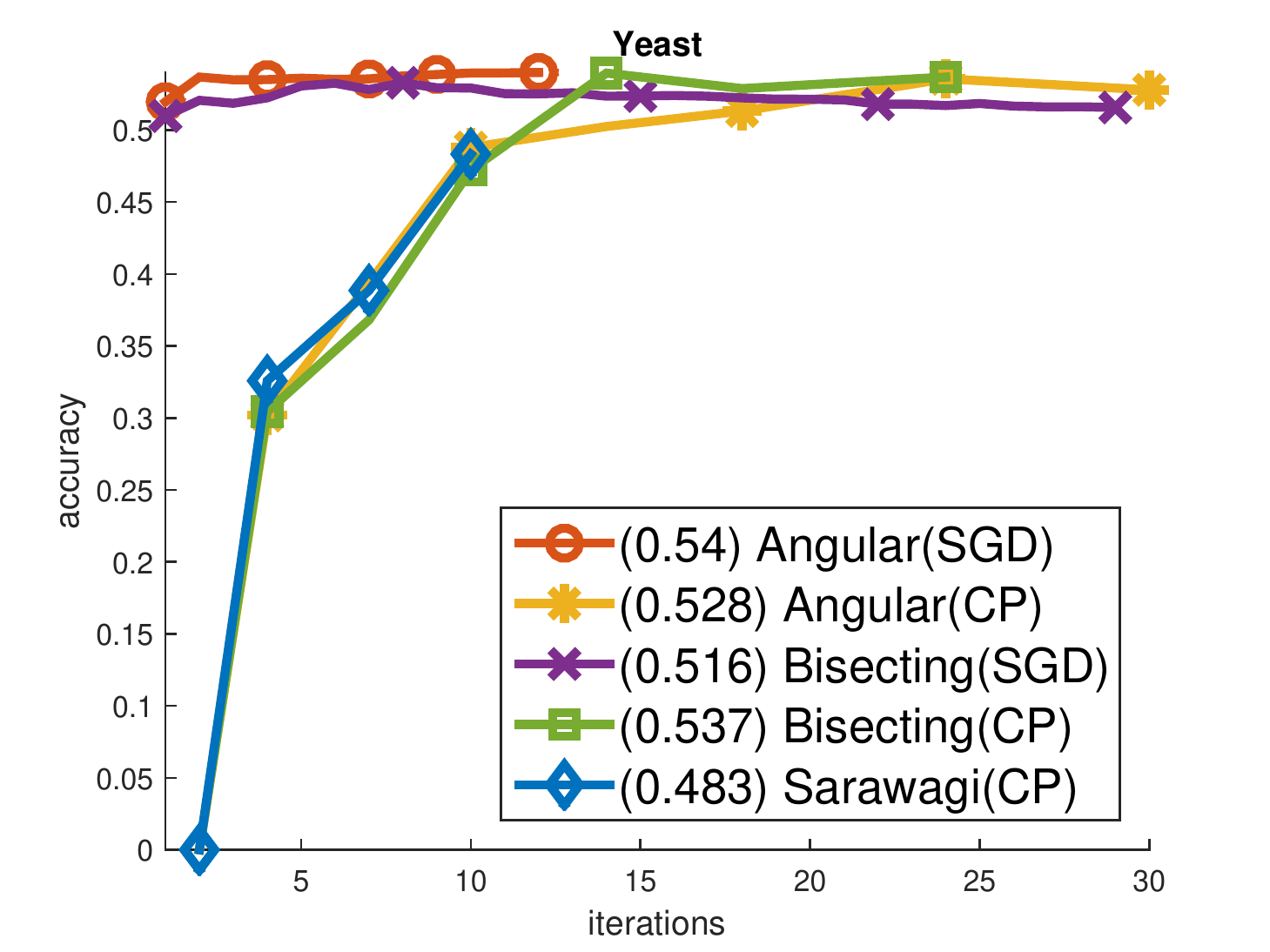}
\caption{Accuracy vs iterations}
\end{figure}
%\begin{subfigure}[b]{.49\linewidth}
%\includegraphics[width=\linewidth]{Yeast_f5.pdf}
%\caption{Accuracy vs iterations}
%\end{subfigure}
%\caption{Additional experiment plot (Yeast)}
%\end{figure}
%
\begin{figure}[H]
\centering     
\begin{subfigure}[b]{\linewidth}
\includegraphics[width=.8\linewidth]{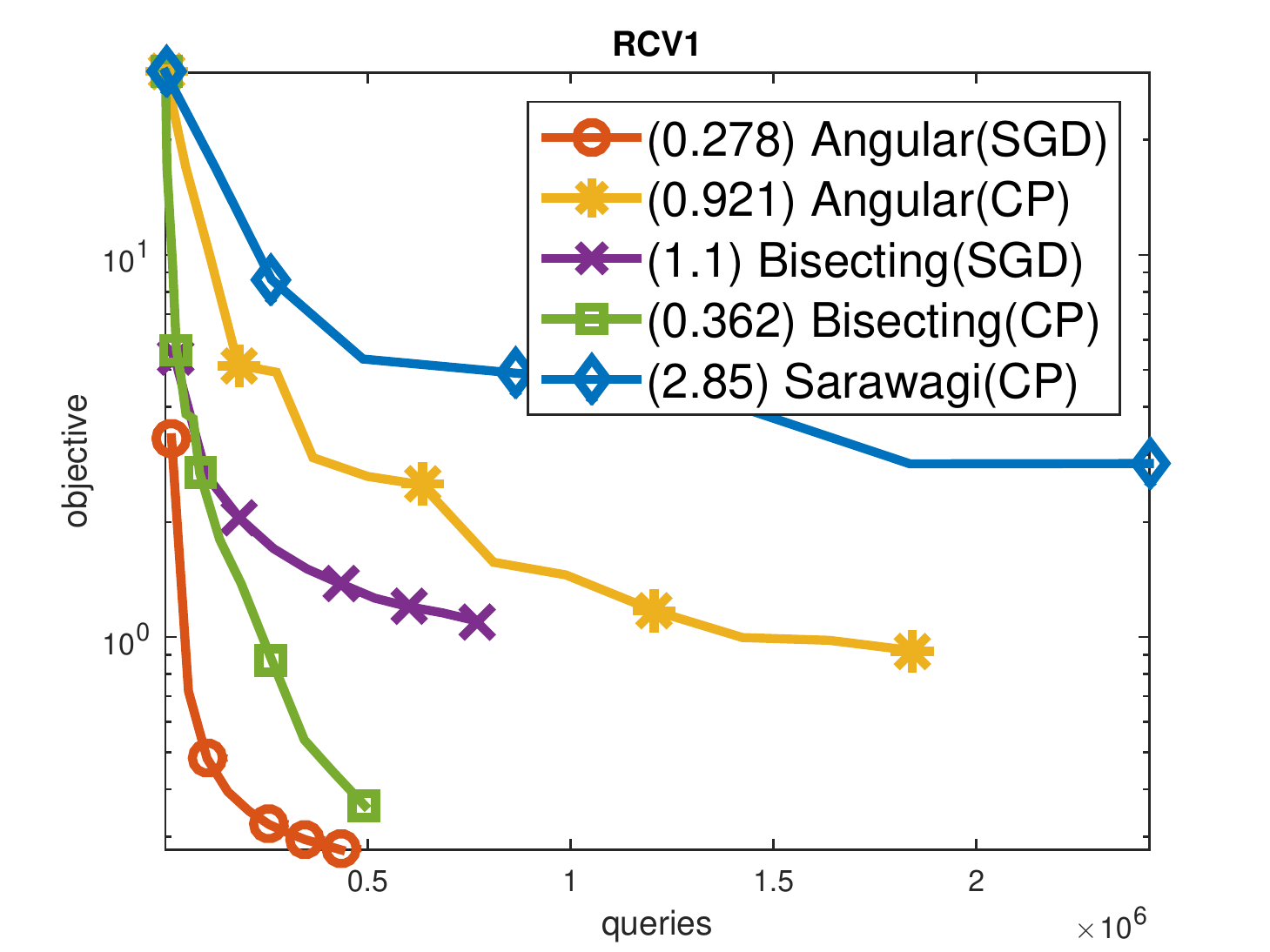}
\caption{Objective vs  queries}
\end{subfigure}
\begin{subfigure}[b]{\linewidth}
%\begin{figure}[H]
%\centering     
\includegraphics[width=.8\linewidth]{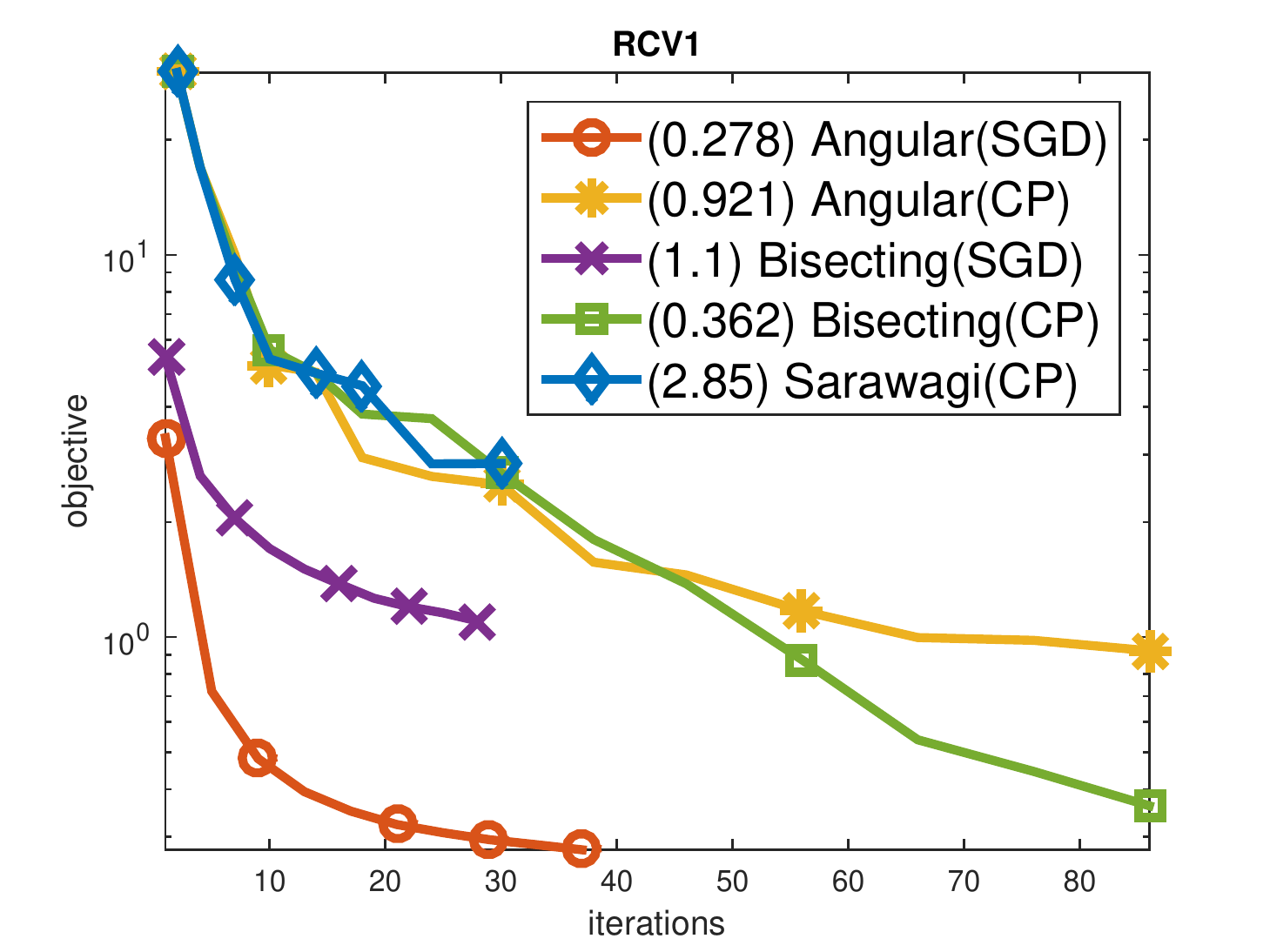}
\caption{Objective vs iterations}
\end{subfigure}
\caption{Additional experiment plot (RCV)}
\end{figure}
\begin{figure}[H]
\centering     
\begin{subfigure}[b]{\linewidth}
\includegraphics[width=.8\linewidth]{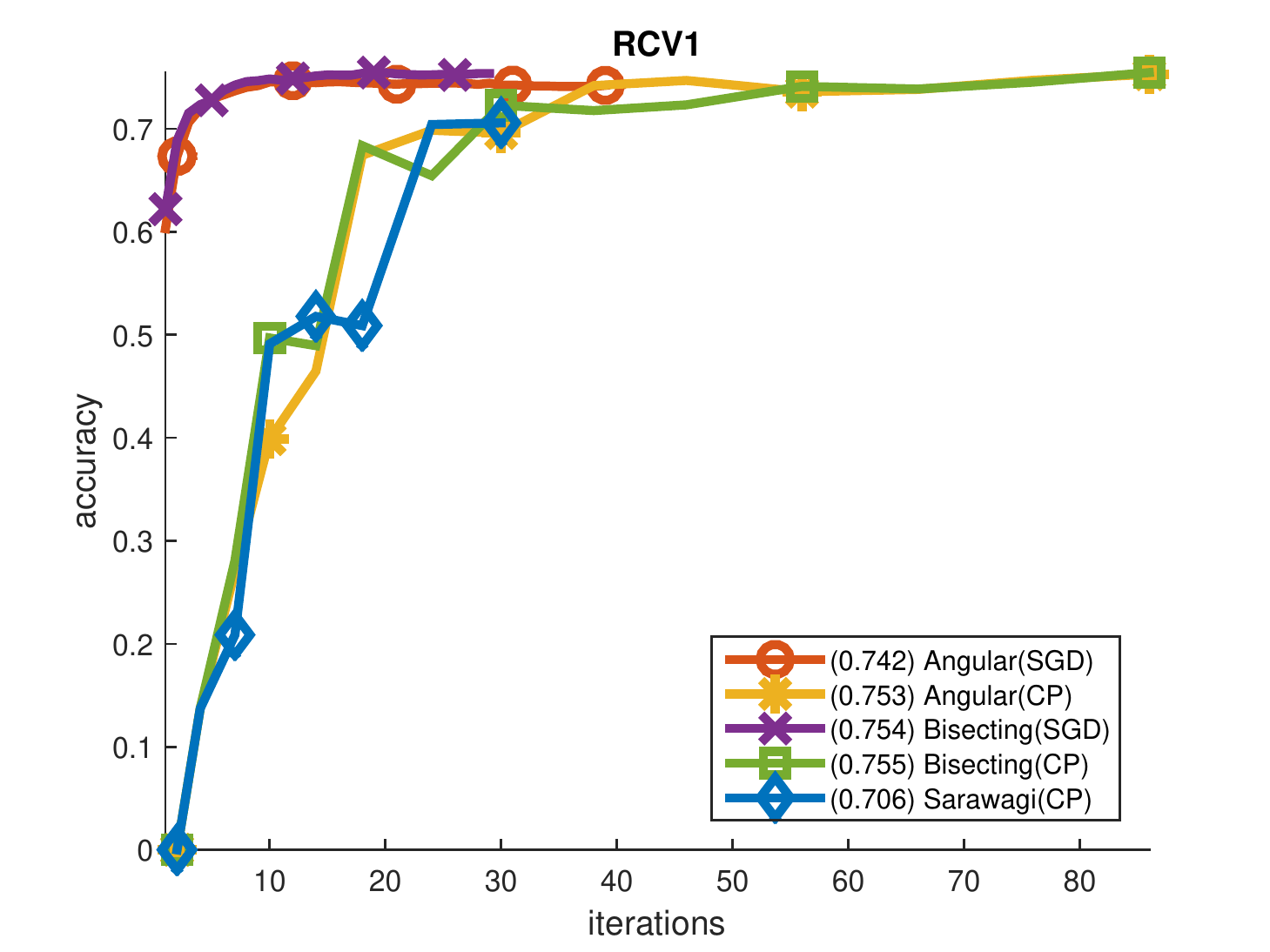}
\caption{Accuracy vs iterations}
\end{subfigure}
% \begin{subfigure}[b]{.49\linewidth}
% \includegraphics[width=\linewidth]{RCV1_f5.pdf}
% \caption{Accuracy vs iterations}
% \end{subfigure}
\caption{Additional experiment plot (RCV) (continued)}
\end{figure}

\begin{figure}[H]
\centering     
\label{f_QP}\includegraphics[width=.8\linewidth]{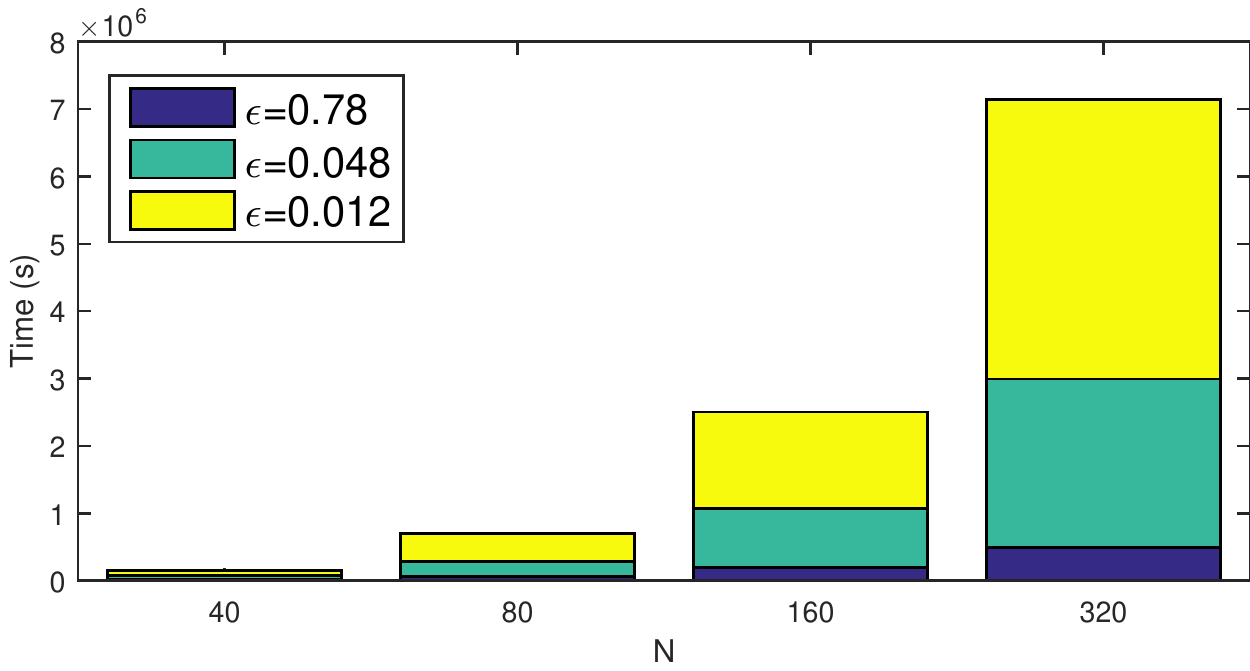}
%\subfigure[Iteration 4, final]{\label{fig:ei4}\includegraphics[width=60mm]{time.png}}
\caption{Total time spent at QP to reach $\epsilon$ in cutting-plane method for different number of instances. Time spend at QP growth super linearly. }
\end{figure}

%\begin{figure}
%\centering     
%\subfigure[Performance of different argmax for cutting-plain approach]{\label{m_f1_iteration}\includegraphics[width=90mm]{iteration.png}}
%\subfigure[Convergence rates]{\label{m_f1_convergence}\includegraphics[width=90mm]{convergence.png}}
%%\subfigure[Iteration 4, final]{\label{fig:ei4}\includegraphics[width=60mm]{time.png}}
%\caption{Comparsion of the algorithms: \ref{m_f1_iteration}  shows effect of different argmax in cutting-plane approach. Exact argmax converges global %minimum of slack rescaling objective. \ref{m_f1_convergence} shows different convergence rate for different algorithms. SGD shows faster convergence.}
%\end{figure}

\begin{figure}[H]
\centering     
\begin{subfigure}{\linewidth}
\centering  
\includegraphics[width=.5\linewidth]{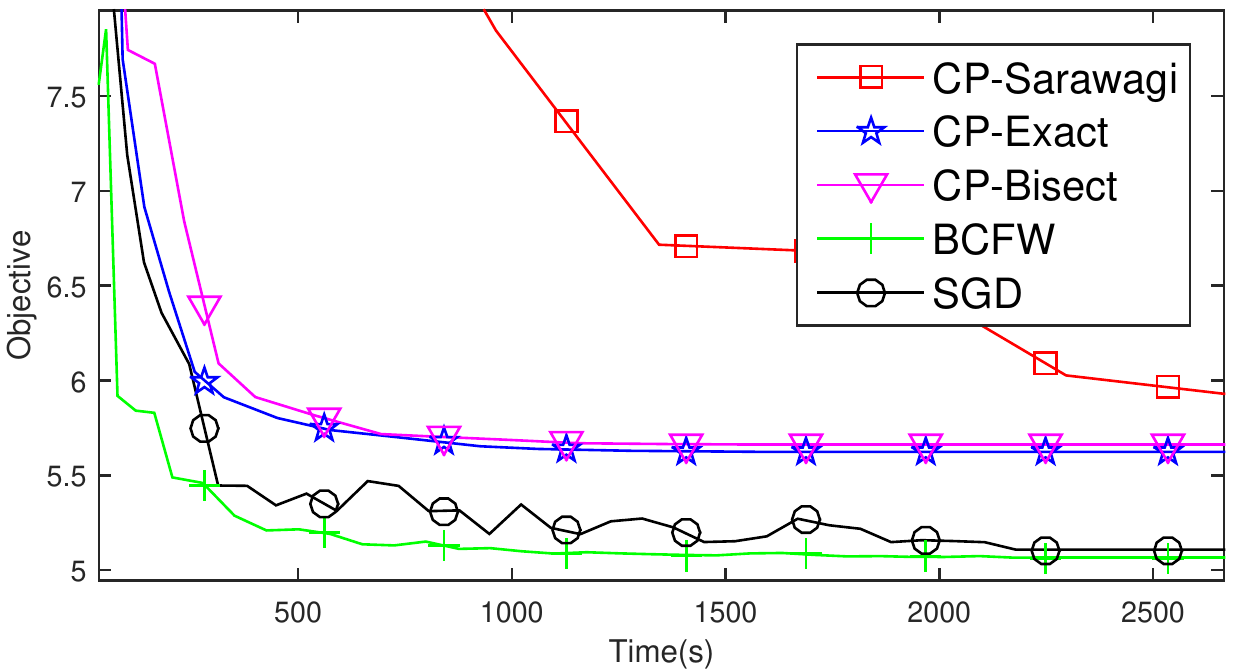}
\caption{Yeast. N=320.}\label{cf_yeast_320}
\end{subfigure}
\begin{subfigure}{\linewidth}
\centering  
\includegraphics[width=.5\linewidth]{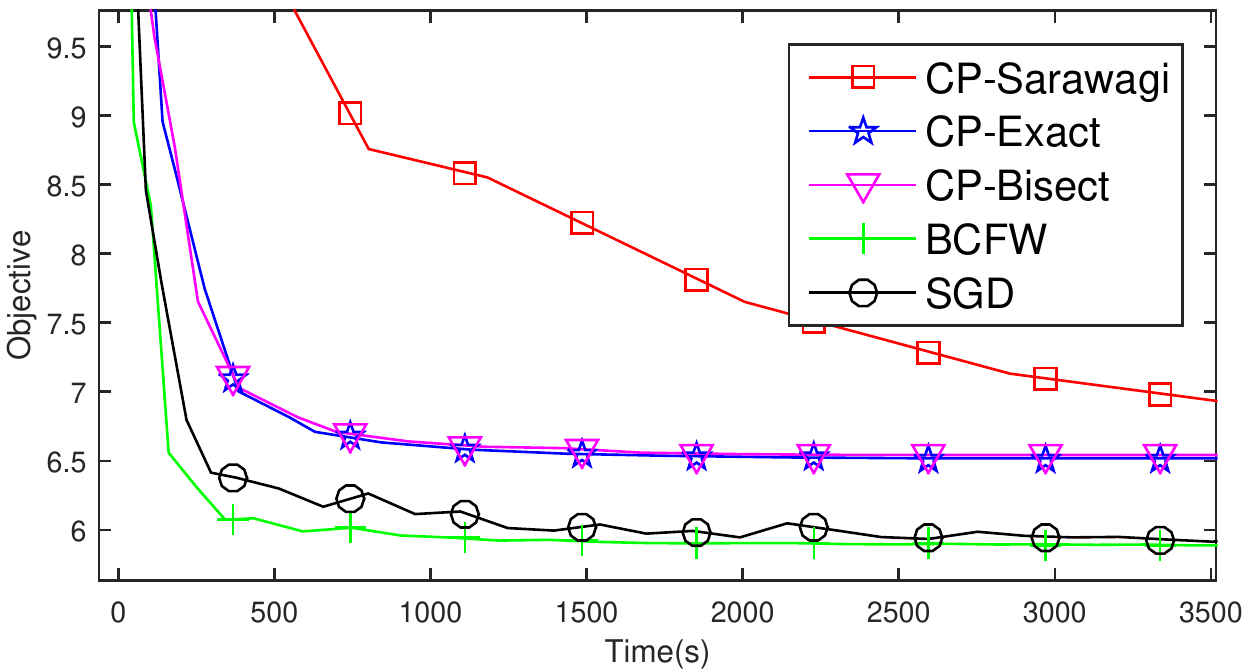}
\caption{Yeast. N=640.}\label{cf_yeast_640}
\end{subfigure}
\begin{subfigure}{\linewidth}
\centering  
\includegraphics[width=.5\linewidth]{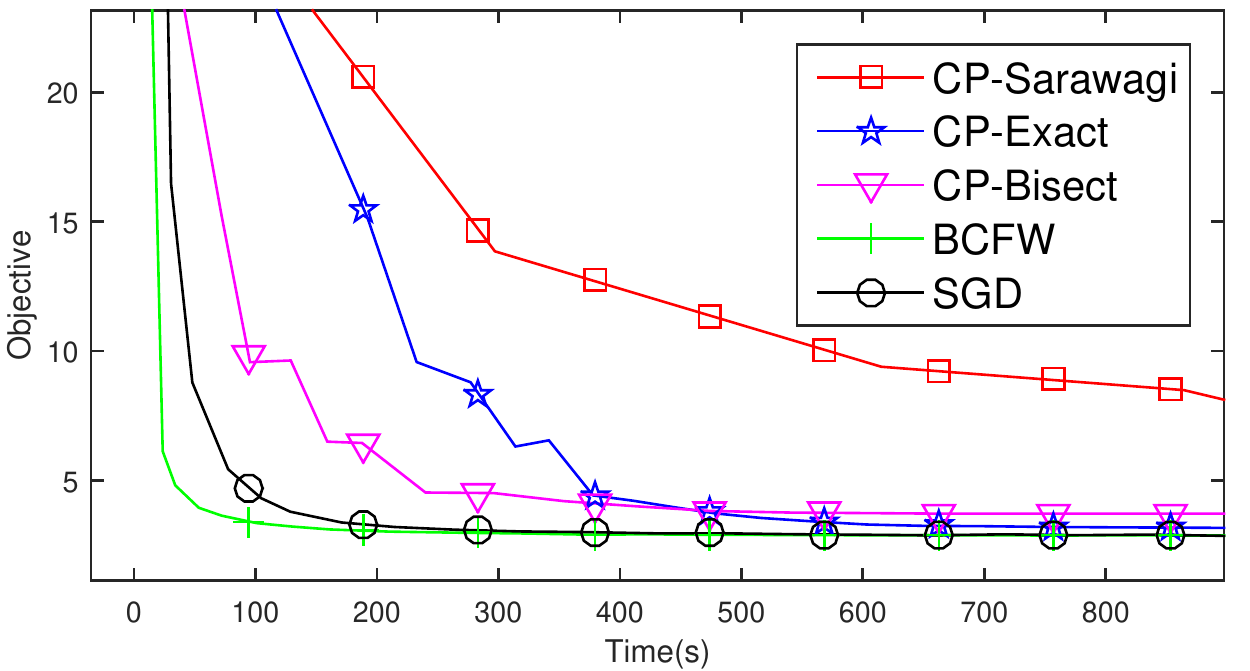}
\caption{Enron. N=80.}\label{cf_enron_80}
\end{subfigure}
\begin{subfigure}{\linewidth}
\centering  \includegraphics[width=.5\linewidth]{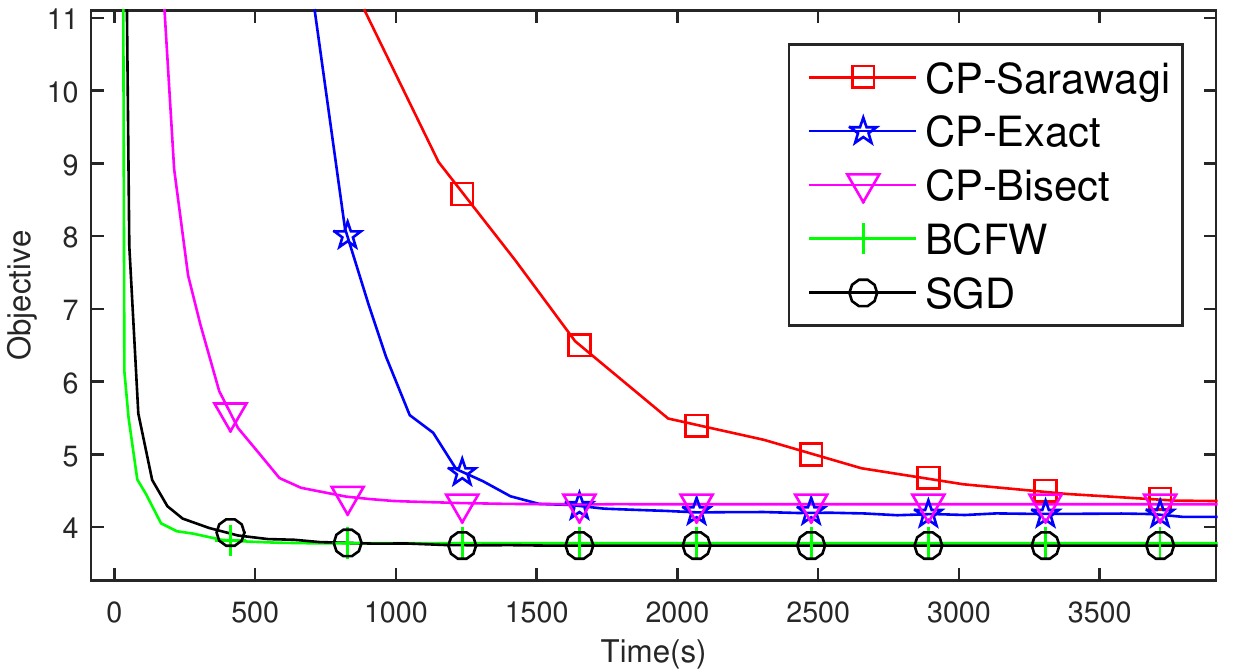}
\caption{Enron. N=160}\label{cf_enron_160}
\end{subfigure}
%\subfigure[Yeast N=640]{
%\includegraphics[width=80mm]{YeastN=640.pdf}\label{cf_yeast_640}}
%\subfigure[Enron N=80]{
%\includegraphics[width=80mm]{EnronN=80.pdf}\label{cf_enron_80}}
%\subfigure[Enron N=160]{
%\includegraphics[width=80mm]{EnronN=160.pdf}\label{cf_enron_160}}
%\subfigure[Iteration 4, final]{\label{fig:ei4}\includegraphics[width=60mm]{time.png}}
\caption{Convergence rate for subset of the data: Yeast (\ref{cf_yeast_320} N=320 and \ref{cf_yeast_640} N=640)   and Enron (\ref{cf_enron_80}  N=80 and \ref{cf_enron_160} N=160). BCFW performs the best. As further optimization performances of cutting plane drops due to the large time spent in QP. }
\label{f_convergence}
\end{figure}
\section{Conclusion}

\begin{comment}
As we saw in our experiments, and has also been previously noted,
slack rescaling is often beneficial compared to margin rescaling in
terms of predictive performance.  However, the margin-rescaled argmax
is often much easier computationally due to its
additive form. 
\end{comment}
 Margin rescaling is  much more frequently used in
practice.  Here, we show how an oracle for solving an argmax of the
margin rescalied form, or perhaps a slightly modified form (the
constrained-$\lambda$ oracle), is sufficient for also obtaining exact
solutions to the slack-rescale argmax \eqref{eq:slack_argmax}.  This
allows us to train slack-rescaled SVMs using SGD, obtaining better
predictive performance than using margin rescaling.  Prior work in
this direction \cite{sarawagi2008accurate} was only approximate, and
more significantly, only enabled using cutting-plane methods, {\em
  not} SGD, and was thus not appropriate for large scale problems.
More recently, \cite{bauer2014efficient} proposed an efficient
dynamic programming approach for solving the slack-rescaled argmax
\eqref{eq:slack_argmax}, but their approach is only valid for sequence
problems\footnote{The approach can also be generalized to tree-structured
  problems.} and only when using hamming errors, not for more general
structured prediction problems.  Here, we provide a generic method
relying on a simple explicitly specified oracle that is guaranteed to
be exact and efficient even when the number of labels are infinite and allows using SGD and thus working with large scale
problems.

%&latex
\chapter{Efficient  Inference Method for Bi-criteria Surrogate Loss with $\lambda$-oracle}\label{cha:inference2}

In the previous chapter, we investigated one alternative form of the surrogate loss, slack rescaling,  although its inference was much complicated than margin rescaling, we presented algorithms to do efficient inference with caveats: We showed that the exact inference is not possible with only access to the oracle, then the optimization becomes the only heuristic, and on the other hand, a more powerful oracle was required for the exact inference that is not usually available.

We study the method for efficient inference in the same setting that we access the labels through the $\lambda$-oracle, but in this chapter, we consider  much wider class of surrogate losses:  bi-criteria surrogate loss. We briefly discussed quasi-concavity as an important property of functions that it  is assumed to be natural property in economics that expresses a preference. However, we will show that quasi-concavity is a crucial property that enables an efficient inference with $\lambda$-oracle. For the bi-criteria surrogate loss, an exact optimal search is possible with an LP or convex relaxation over the labels. We prove some results showing the optimality of our inference method as well as empirical improvements in several important real-world tasks. 
 
% !TeX root = slack_rescaling.tex
\section{Problem Formulation}
Problem formulation is similar to the previous chapter. But for concreteness, we review our setting. Recall  our objective function in \eqref{eq:objective_function_struct_svm},
  \begin{align*} %\textstyle 
        \min_{w} \frac{C}{2}\|w\|^2_2
        +\frac{1}{n}\!\sum_i\max_{y\in\mathcal{Y}} \tilde{\Phi}(w,y,y_{i},x_{i})
\end{align*}

 We fix $i$ and $w$ and only consider an argmax operation for an instance. Thus, let    
$\Phi(y)=\tilde{\Phi}(w,y,y_{i},x_{i})$. In the previous chapter, we considered slack rescaling formulation. Then, $\Phi$ was in the form of,
\begin{align*}
\Phi(y)=\Phi_{SR}(y)=h(y)\cdot g(y)
\end{align*}
where $g(y)=L(y,y_i)$ and $h(y)=m(w,x_{i},y,y_i)+1$. 

Now, we investigate the methods for more general form, which is discussed in chapter \ref{sec:quasi_convex_surrogate}.
  We reiterate  the surrogate loss discussed in the chapter for the concreteness.

\begin{comment}
\begin{definition}[Bi-criteria Surrogate Loss]
    Consider a surrogate loss $\tilde{L}(w,x_i,y_i)$ with a following form,
\begin{align*}
\tilde{L}(w,x_i,y_i)=\max_{y\in\mathcal{Y}} \Phi(y)
\end{align*}
where $\Phi:\mathcal{Y}\mapsto \mathbb{R}$ is a potential function. 

Then,  
$\tilde{L}(w,x_i,y_i)$ is a  bi-criteria surrogate loss if following holds true for a   bi-criteria function $\psi:\mathbb{R}\times \mathbb{R}_+\mapsto \mathbb{R}$.
\begin{enumerate}
\item
 $\Phi(y)= \psi(m(y), L(y))$ 
\item 
 $\psi$ is a differentiable function.

\item  Let $K_\alpha=\{(a,b)\mid a\in\mathbb{R}, b\in\mathbb{R}_+,\psi(a,b)\ge \alpha\}$ be the $\alpha$ super level
set of $\psi$, and $\beta=\psi(0,0)$. $\psi$ is a quasi-concave function  in the domain $K_{\beta}$, i.e., $K_\beta$ is convex.   
\item  
   $\psi$ is a monotonically increasing function in both arguments that  for $\forall a\in \mathbb{R}, \forall b\in \mathbb{R}_+,\forall \epsilon>0,$ $\psi(a+\epsilon,b+\epsilon)>\psi(a,b)$.
\end{enumerate}
\end{definition}

 Note that for bi-criteria surrogate loss, the arguments for the bi-criteria function $\psi$ are fixed to $m(y)$ and $L(y)$. We can generalize bi-criteria surrogate loss that bi-criteria function can take other loss functions. 
\end{comment}

\begin{repdefinition}{def:gbicriteria}[Bi-criteria Surrogate Loss]
    Consider a surrogate loss $\tilde{L}_i(w,x_i,y_i)$ with a following form,
\begin{align*}
\tilde{L}_{i}(w,x_i,y_i)=\max_{y\in\mathcal{Y}} \Phi_{i}(y)
\end{align*}
where $\Phi_i:\mathcal{Y}\mapsto \mathbb{R}$ is a potential function. 

Then,  
$\tilde{L}_{i}(w,x_i,y_i)$ is a   bi-criteria surrogate loss if following holds true for a bi-criteria functions $\psi_i:\mathbb{R}\times \mathbb{R}_\mapsto \mathbb{R}$,  two loss factor functions $h_i:\mathcal{Y}\mapsto \mathbb{R}$ and $g_i:\mathcal{Y}\mapsto \mathbb{R}$, 
 
\begin{enumerate}
\item
 $\Phi_{i}(y)= \psi_i(h_i(y), g_i(y))$ 
\item In domain $K_0$.
\begin{enumerate}
\item 
 $\psi_{i}$ is a differentiable function.
\item   $\psi_i$ is a quasi-concave function, i.e., $\forall\beta\ge0,  K_\beta$ is convex. 

\item  
   $\psi_i$ is a monotonically  increasing function for each arguments, and $\forall( a, b)\in K_0,\forall \epsilon>0,$  $\psi_{i}(a+\epsilon,b+\epsilon)>\psi_{i}(a,b)$.
\end{enumerate}
\item Function $h_i$ and $g_i$ decomposes into substructures that $\forall \lambda\ge0$, $\argmax_y \lambda g_i(y)+h_i(y)$ is efficiently solvable. 
%$\exists \rho:\mathbb{R_+}\mapsto\mathbb{R}_+$ such that $\forall \lambda\ge0$,  $\argmax_{y\in\mathcal{Y}} \lambda g(y)+ h(y)$=$\argmax_{y\in\mathcal{Y}} \rho(\lambda)L(y)+m(y)$.
 \end{enumerate}
where    $K_\alpha=\{(a,b)\mid a\in\mathbb{R}, b\in\mathbb{R}, \psi_i(a,b)\ge \alpha\}$ is $\alpha$ super level
set. 
\end{repdefinition}

%\begin{restatable}[Goldbach's conjecture]{thmm}{goldbach}
%\label{thmm:goldbach}
%Every even integer greater than 2 can be expressed as the sum of two primes.
%\end{restatable}

%\section{Second}

%We recall \cref{thmm:goldbach}:

%\goldbach

 The natural  objective of interest in the chapter is to solve,
\begin{align}
 y^{*}=\argmax_{y\in\mathcal{Y}} \Phi_{i}(y)=\argmax_{y\in\mathcal{Y}} \psi_{i}(h_{i}(y), g_{i}(y)) \label{eq:in2_obj}
\end{align} 
 with the minimum call to the  $\lambda$-oracle with the choice of $h$ and $g$.  

 Later, we show that  objective in \eqref{eq:in2_obj} is too hard to solve, and we focus on the  main objective to  find the optimum in relaxed to the space of the labels $\mathcal{Y}$,
\begin{align}
 \tilde{y}^{*}=\argmax_{y\in\tilde{\mathcal{Y}}} \Phi(y) \label{eq:in2_obj_f}
\end{align} 
where 
 $\tilde{\mathcal{Y}}$ is the convex hull of labels, .i.e,
\begin{align*}
\tilde{\mathcal{Y}}=\left \{ \sum_j a_j y_j  \mid   y_j\in\mathcal{Y},  \sum a_j=1, a_j\ge 0  \right \}
\end{align*}
 Relaxing $\mathcal{Y}$ such way is called LP relaxation or over approximation\cite{joachims2009cutting}, and    $\tilde{y}^{*}$ is widely used in structured prediction.
 
 However, later on, we discuss, how to obtain $y^*$ in the experiment section.

\section{ Convex hull search}  \label{sec:convex_hull_search}
In this section,  we present our inference method, {\em convex hull search}. It is easy to implement since it does not require modification to the oracle. This can be directly applied to a system that uses margin rescaling to use other suggested surrogate losses with a few lines of the code added. In the experiment section, we show that it boosts the performance. Also since it uses the optimal choice of $\lambda$ each iteration (departing from the binary search of the Bisecting search in  \cite{choi2016fast}), it is very fast. Although it is shown that in the previous chapter that there can be no approximation guarantee for this approach, we show that exact optimum can be found in LP relaxed label space, convex space of labels. LP relaxation is extensively used in structured prediction when exact inference is not tractable. It meets the practical need for fast inference and good performance \cite{finley2008training,pmlr-v48-meshi16}.

Although we mentioned that functions $h$ and $g$ can be freely chosen as long as the properties are met, in this chapter,  for conciseness, we only consider the general case when   $h(y)=m(w, x_i,y_i,y),$  the loss regarding the margin, and $g(y)=L(y,y_i)$, the loss regarding the structural loss. In this case, $\lambda$-oracle exactly matches $\lambda$-oracle used in previous chapter \ref{cha:inference}.
 
 As in the previous chapter, since we consider solving \eqref{eq:in2_obj} for fixed $w$. Then, for each label $y$, $h(y)$ and $g(y)$ are fixed. We can consider each label $y\in\mathcal{Y}$ is a  point $[h(y),g(y)]$ in $2$-dimensional plane $\mathcal{Y}\subseteq P=\mathbb{R} \times \mathbb{R}_+$, and $h(y)$ and $g(y)$ are  the coordinates in  $X$-axis and $Y$-axis of $P$ correspondingly. We also use an alternative  notation that $[y]_1=h(y)$ as $X$-coordinate and $[y]_2=g(y)$ as $Y$-coordinate.  
Let $CH=Conv(\mathcal{Y})$ be the convex hull of the labels.
%We denote point with largest $h(y)$ as the top point, and that with largest $g(y)$ as the rightmost point.  
Let $V$ be the vertices of the convex hull $CH$. %right of $H$, and $V^*\subseteq V$ the right upper vertices of the convex hull, $\forall y\in V^*,\forall \epsilon,\epsilon'>0,$ $[h(y)+\epsilon,g(y)+\epsilon']\notin H$. % the vertices along the contour of the convex hull starting from top most point to the right most point. 
  Let $B$ be the edges on the boundary of $CH$ connecting  $V$.  We assume that there are no three labels lie exactly in $B$, which is highly unlikely and can be removed with a small perturbation of $w$ since in our case $h(y)$ is real-valued function respect to $w$.
Let $\partial(l)$ be the slope of line $l\subset P$, i.e., $\partial l=(h(y)-h(y'))/(g(y)-g(y')),$   $y\neq y'\in l$. For a line $l\subset P$ with $\partial(l)<0$, and a set of points $U\subseteq P$, we denote that {\em $U$ is above the line $l$}, if $\forall u\in U$,  $\exists p\in l,$ $[u]_1=[p]_{1},[u]_2>[p]_{1}$. For $a,b\in P$, denote $\overline{ab}\subset P$ be the  linear line segment in $P$ that ends at $a$ and $b$, and also denote $\overleftrightarrow{ab}\subset P$ be the  linear line  in $P$ that extends $\overline{ab}$ in both ends.

%For the opposite relationship, we denote $U$ is below the line $l$.      
 %We will show that $y^*$ is on one edge of $E$, let $E^*$ be the edge. 

 In the previous chapter, we showed that with the appropriate $\lambda$,  each call to the oracle shrinks the potential space of   $y^*_{SR}$, and an exact or approximation of $y^*_{SR}$ can be efficiently found. Exact inference requires the modification to the oracle, which not only slows the inference considerably but also not always available. An approximation method, Bisecting search, does not require any modification. We extend the approach. Since the only input to the oracle is $\lambda$, choice of $\lambda$ is the main focus of the algorithm.   Bisecting search depends on the binary search over the possible range of $\lambda$. However, the number of oracle calls is unbounded without an approximation guarantee. We extend the Bisecting search for a bi-criteria surrogate loss function discussed in previously and improve the efficiency by providing the optimum choice of $\lambda$, with a guarantee of the number of calls to the oracle to find the optimum in relaxed space of labels.

\subsection{Method}
Since it is shown in the chapter  \ref{sec:lambda_oracle_inexact} that with only access to the $\lambda$-oracle, the exact solution cannot be obtained. 
Define the set of labels which are attainable with  $\lambda$ oracle as $\mathcal{Y}_\lambda$ and its maximum as $y^*_\lambda$, i.e.,
\begin{align*}
y^*_\lambda=\max_{y\in \mathcal{Y}_\lambda} \Phi(y) && \text{where }  \mathcal{Y}_\lambda=\left \{\mathcal{O}(\lambda)  \mid  \forall \lambda \in \mathbb{R}  \right \} 
\end{align*}
We focus on the objective with  the LP relaxed or fractional space of the labels,
\begin{align*}
\tilde{y}^*=\argmax_{y\in  \tilde{\mathcal{Y}}} \Phi(y)
\end{align*}
where $\tilde{\mathcal{Y}}$ is the convex hull of labels, .i.e,
\begin{align*}
\tilde{\mathcal{Y}}=\left \{ \sum_j a_j y_j  \mid   y_j\in\mathcal{Y},  \sum a_j=1, a_j\ge 0  \right \}
\end{align*}

Eventually, we will show that the domain of $\tilde{y}^*$ is restricted by showing that $\tilde{y}^*\in \tilde{\mathcal{Y}}_2$, which is the the set of the convex combination of two labels in  $\mathcal{Y}$, i.e.,
\begin{align*}
\tilde{\mathcal{Y}}_2=\left \{ a y_1+(1-a)y_2 \left|   y_1,y_{2}\in\mathcal{Y},  0\le a\le 1 \right. \right \}
\end{align*}
We will discuss obtaining $y^*$ in later the experiment section. 
$h$ and $g$ are also defined for the fractional labels, i.e. for $y=  a y_1+(1-a)y_2 \in \tilde{\mathcal{Y}}$,
$h(y)=a h(y_1)+(1-a)h(y_2),$ $g(y)=ag(y_1)+(1-a)g(y_{2})$.
Note that optimization respect to the $\tilde{\mathcal{Y}}$ does not change the optimization much. (e.g., SGD update is in similar form. The update requires two updates;  One update for each fractional label.)

Convex hull search is shown in Algorithm \ref{alg:convexhull_search}. 
%$S$ is the set of the vertices of current convex hull where the elements are indexed with increasing $g(y)$, i.e. $S=\{y_1,y_2, \dots \}$, and $g(y_a)<g(y_b)$ if $b>a$. 
AddToSortedList($S,y$)  adds a vertex $y$ into the list $S$ sorted respect to  $g$. Denote $y_t$ for $t$-th label  in $S$ sorted respect to $g$. GoldenSearch($s$) finds the maximum $\Phi$ point in a line segment $s$ via binary search. Note that during GoldenSearch no oracle call is needed.  Note that for a system uses margin rescaling, this is easy to implement since only a few lines need to be added to query a $\lambda$ and update, and use the existing code. 

 The correctness and the running time is based on the unique optimal $\lambda^*$ which is obtained from  the  the unique linear lower bound of  the candidate space of $y^*_\lambda$, which is described in Theorem \ref{thm:optimal_lambda}. At each iteration, convex hull search finds the unique linear lower bound and obtains $\lambda^*$. $\lambda^*$ also serves as a termination criteria that if $\mathcal{O}(\lambda^*)$ is  a label which has been previously found than we can conclude that the optimal label $\lambda^*_\lambda\in S$,  and we can terminate the search since it is already previously been found. To prove Theorem \ref{thm:optimal_lambda}, we first prove several lemmas. 

%Let $\lambda_t$ be the value of the $\lambda$ used at iteration $t$ in the while loop. 

We start by first stating  following Theorem which describes the main result of the correctness and the running time (in the number of oracle calls).

\begin{theorem} \label{thm:chs_correctness}
Convex hull search in  Algorithm \ref{alg:convexhull_search} returns $\tilde{y}^*$  in $|V|$ calls to the oracle where $V$ is the   vertices of the convex hull of the labels in $2$ dimensional plane.%, i.e.,  $V'=ConvexHull([h(y),g(y)|y\in\mathcal{Y})$.
%\begin{enumerate}
%\item  $y_{\Phi}^*$ can be obtained in $O(|V'|)$ call to the oracle where $V'$ is set of the vertices of convex hull of $\tilde{Y}$, i.e.,  $V'=ConvexHull([h(y),g(y)|y\in\mathcal{Y})$ %\item  If $y_{\lambda_t}=y_{\lambda_{t'}}$ for $t'<t$ at iteration $t$, the algorithm terminates.
%\item Optimality of a label can be checked at most two oracle call. 
%\item  If candidate argmax labels exist, those can be used to reduce the number of oracle calls further.    
%\end{enumerate}
\end{theorem}

\begin{algorithm}[H]
  \caption{Convex Hull Search  %See the text for details.
  } \label{alg:convexhull_search}
  \begin{algorithmic}
\Procedure{ConvexHullSearch}{}
  \State $S\gets \emptyset$ . 
  \State $\lambda\gets\infty$
 \While {true}
 \State $y\gets \mathcal{O}(\lambda)$%=\argmax_y h(y)+\lambda g(y)$
  \If {$y\in S$}
\Return GetMaxFract($S$)
 \EndIf
 \State  $S\gets$ AddToSortedList$(S,y)$
 \State $\lambda\gets$ Get-$\lambda (S)$
 \EndWhile
\EndProcedure   
 \hrulefill 
\Function{Get-$\lambda$}{$S$}
 \State  $t\gets\argmax_t \Phi(y_t)$
 \State  $T'\gets$  the tangent line of $\Phi$-contour at $y_t$.
 \If {$y_{t-1}$ exists and $y_{t-1}$ is below $T'$ }  
  $T \gets \overline{y_{t-1}y_t}$
 \EndIf
 \If {$y_{t+1}$ exists and $y_{t+1}$ is below $T'$ }
  $T \gets \overline{y_{t}y_{t+1}}$
 \EndIf
  \Return $-\partial(T)$
\EndFunction    
 \hrulefill 
\Function{GetMaxFract}{$S$}
 \State  $t\gets\argmax_t \Phi(y_t)$
 \State $y^*_1\gets$GoldenSearch$(\overline{y_{t-1}y_t})$
 \State $y^*_2\gets$GoldenSearch$(\overline{y_ty_{t+1}})$
 \If {$\Phi(y^*_1)>\Phi(y^*_1)$}
 \Return $y^*_1$.
 \EndIf
 \Return $y^*_2$. 
\EndFunction    
  \end{algorithmic}
\end{algorithm}

Running time depends on the geometrical structure of the labels in the plane. In the worst case, all the labels could be the vertices of the convex hull, however, this is unlikely, and especially in the common case when the support of $g$ is discreet, the convex hull has only a few vertices, which will be shown later.  This is much improved over the binary search of the bisecting search in \cite{choi2016fast}. In bisecting search, there is  no guarantee of the number of oracle calls to obtain next label, i.e., it might require an indefinite number of oracle calls.  On the contrary, in convex hull search,  it is guaranteed to obtain a different label for each oracle call. We will show that this property is achieved by the unique choice $\lambda$s. We will also discuss how to evaluate the result of  Theorem \ref{thm:chs_correctness} later. 
We will prove the Theorem follow by a series of lemmas and a Theorem.

%\subsection{Convex hull search}
%

The following lemma shows that for the bi-criteria surrogate loss, the optimal label is always on the boundary
of the convex hull from the monotonicity of the $\psi$.

\begin{lemma}[Existence of optimum on the boundary] \label{lem:optimum_on_edge}
 The fractional optimal label  $\tilde{y}^{*}$ is on the boundary of the convex hull CH, i.e.,
\begin{align*} 
\tilde{y}^{*}\in B
\end{align*}
and is a convex combination of the two label,
i.e.,
\begin{align*}  \tilde{y}^{*}\in \tilde{\mathcal{Y}}_2
\end{align*}
 \proof 
%This follows from our assumption that our objective potential $\psi$ that it is strictly increasing at least in one argument, and both labels are the argmax label of $\psi$.

We prove by contradiction, assume the first part of the lemma is false. i.e., $\tilde{y}^{*}\notin B$. Then,   $\exists  \epsilon>0$ s.t. $y=\tilde{y}^* +[\epsilon,\; \epsilon]\in H,\Phi(y^*)<\Phi(y)$ from the definition of bi-criteria surrogate loss. It contradicts 
the definition of $\tilde{y}^*$. 
 
 The second part follows from the fact that in 2 dimensions, all the points in $y\in B\implies y\in \tilde{\mathcal{Y}}_2$  with our assumption that there exists no 3 labels lie on an edge.
\qed
\end{lemma}

From the argument similar to the previous lemma, following lemma shows that the label returned by the $\lambda$-oracle only returns the label on the vertices of the convex hull. 

\begin{lemma}[Limitation of the $\lambda$-oracle] \label{lem:limitation_of_linear_oracle}
For  $\lambda>0$, $\lambda$-oracle only returns a vertex of the convex hull, i,e, for any $\lambda>0$,
\begin{align*} 
 \mathcal{O}(\lambda) \in V 
\end{align*}
\proof
 From the definition of  $\mathcal{O}(\lambda)$, there exists no label above the line $\{y'|\lambda g(y')+h(y')=\lambda g(\mathcal{O}(\lambda))+h(\mathcal{O}(\lambda))\}$. Then,  $\nexists \epsilon>0$ such that $\mathcal{O}(\lambda)+ [\epsilon,\; \lambda\epsilon]\in CH$. This shows that   $\mathcal{O}(\lambda)\in B$. 
  $\mathcal{O}(\lambda) \in V$ follows from  $\mathcal{O}(\lambda)\in B$ and  our assumption that there exists no 3 labels lie on an edge.
\qed
\end{lemma}
Now, we focus on  the efficiency of our algorithm, which is directly related to the Theorem \ref{thm:chs_correctness}. As  in Algorithm \ref{alg:convexhull_search}, let $S=\{y_1,y_2, ...\}\subseteq \mathcal{Y}_\mathcal{\lambda}$ be the set of labels that returned by the $\lambda$-oracle sorted in increasing order of $g(y)$ until the current iteration. Following Theorem shows  that there exists an unique $\lambda$  that can check the sufficient condition for optimality of the current maximum label $\hat{y}=\argmax_{y\in S}\Phi(y)$, i.e., $\hat{y}=y_{\lambda}^*$, which  also serves as a termination condition. 

%Let $S^+=\{y|\Phi(y)>\Phi(\hat{y})\}$ be the area that has a higher $\Phi$ values than current candidate, and $V^*= S^{+}\cap V$ be the vertices of the convex hull that has higher $\Phi$ values. 
\begin{theorem}[Uniqueness of Optimal $\lambda$]\label{thm:optimal_lambda} There exists an unique $\lambda^{*}$ that following holds true. 
%\begin{align}
%\mathcal{T}: 
%\end{align}
 \begin{align*}
 \exists!\lambda^*\in\mathbb{R} \;\text{ s.t.}\;
%y_{\lambda^*} = \hat{y} \implies  y_{\lambda }^*=y^*_I
\mathcal{O}(\lambda^*) \in S \implies  y_{\lambda}^*\in S 
\end{align*}
 That is, if the $\lambda$ oracle with the $\lambda^*$ returns a label which was previously returned, we can terminate the algorithm since $y_{\lambda}^*$ is already found previously.
\proof
 We give a constructive proof for $\lambda^*$, which also describes the algorithm    \ref{alg:convexhull_search}. At each iteration, we use $\lambda^*=-\partial(T)$. $T$ is the separating hyperplane of $U$ and $S'$ where  $U=K-S'$ is the candidate space of $\tilde{y}^*$,   $K=\{y| \Phi(y)\ge\Phi(\hat{y}), y\in P\}$ is the super level set of $\Phi$, and  $S'=interior(Conv(S))$ is the interior of the convex hull of previously returned labels $S$.
 $S'$ is removed from  $U$ since $\tilde{y}^*$ is on the edge from lemma \ref{lem:optimum_on_edge}.
For the illustration, see Figure \ref{sfig:chs}.

\begin{figure}[H]
\centering     
\begin{subfigure}[b]{\linewidth}
\includegraphics[width=\linewidth]{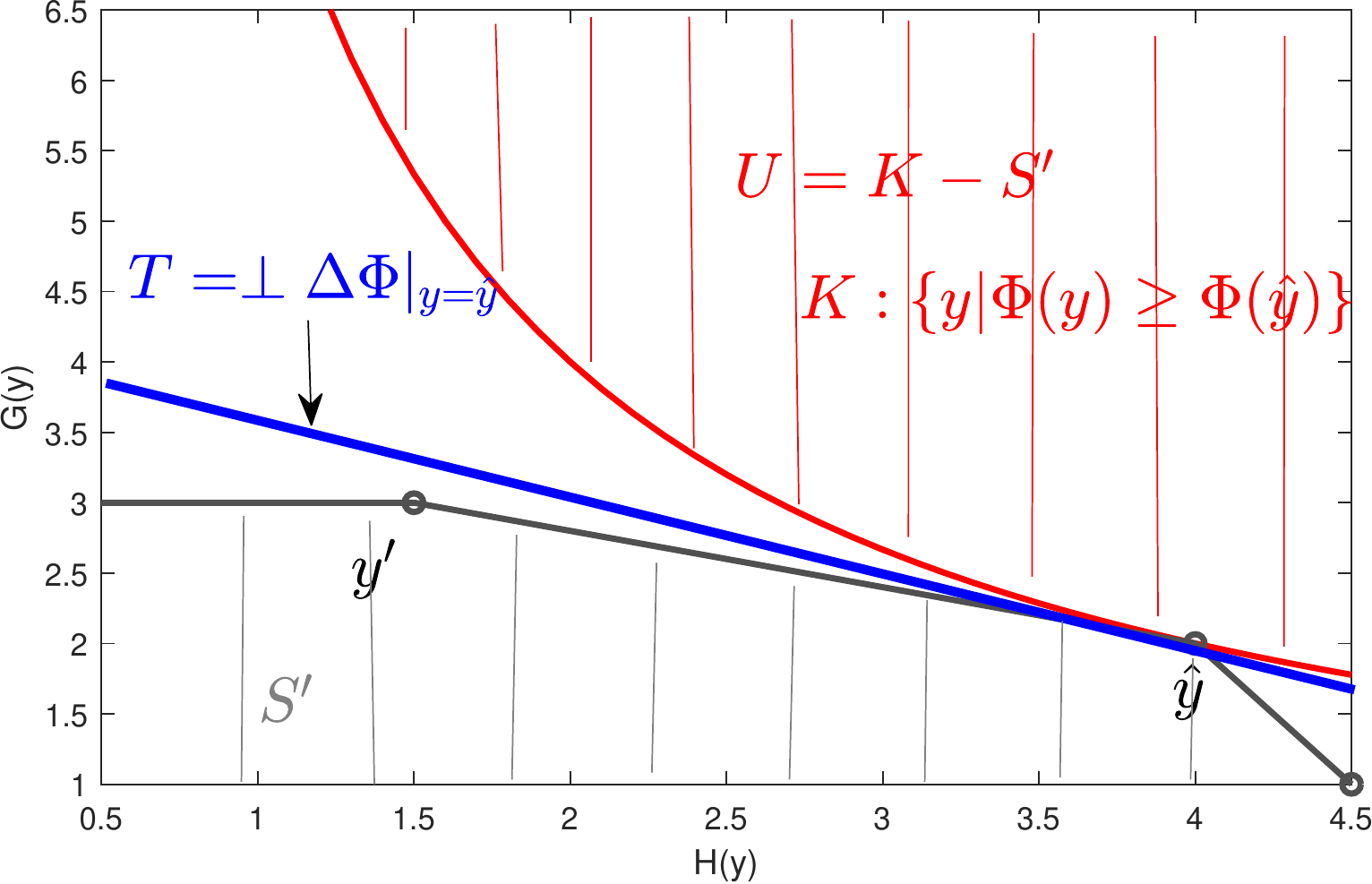}
\caption{$K\cap S'= \emptyset$}\label{sfig:chs1}
\end{subfigure}
\begin{subfigure}[b]{\linewidth}
%\begin{figure}[H]
%\centering     
\includegraphics[width=\linewidth]{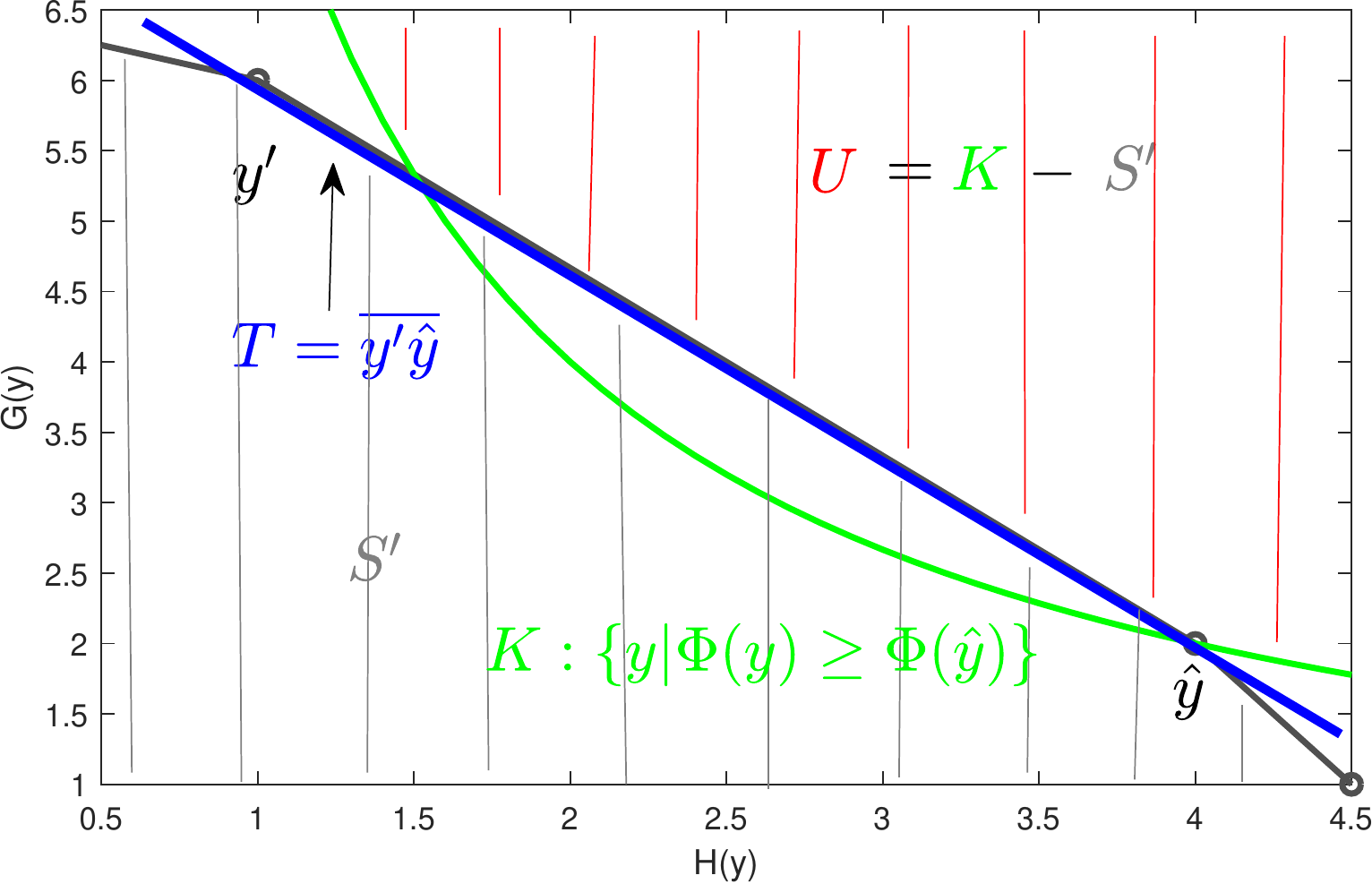}
\caption{$K\cap S'\neq \emptyset$}\label{sfig:chs2}
\end{subfigure}
\caption{An illustration of Convex Hull Search}\label{sfig:chs}
\end{figure}

 If $U\cap\mathcal{Y} \neq \emptyset$, $\mathcal{O}(\lambda^*)$ will return a new point above $T$, i.e. $\mathcal{O}(\lambda^*) \notin S$, and algorithm iterates again.  If  $\mathcal{O}(\lambda^*)\in S$, we can conclude that $U\cap\mathcal{Y} \neq \emptyset$ and terminate taking $y^*$ as the argmax labels among the argmax vertex and the fractional labels in neighboring edges.

 To prove such $T$ exists, consider two cases,  $K\cap S=\emptyset$ or $K\cap S\neq\emptyset$ .
 $K$ is convex since $K$ is a super level set of a quasi-concave function $\psi$. In the former case in Figure \ref{sfig:chs1},  $T$ is the unique supporting hyperplane at $\hat{y}$ defined by the gradient. It exists followed from the supporting hyperplane theorem and convexity of $K$ since $U=K$. 
 For the latter case in Figure \ref{sfig:chs2}, the edge $T=\overline{\hat{y}y'}$  uniquely separates $U$ and $S'$  because any vertex of convex hull, e.g. $y'$, should be on or below the boundary of $K$ and $K\cap S'\neq \emptyset$. 

 Here we describe the procedure to obtain $\tilde{y}^*$ when after obtaining $\hat{y}$.
 If  $K\cap S=\emptyset$ (Figure 
\ref{sfig:chs1}), $\tilde{y}^*=\hat{y}$.
 If $K\cap S\neq\emptyset$ (Figure 
\ref{sfig:chs2}), $\tilde{y}^*$ is on one neighboring edge of $\hat{y}$. This is done in GetMaxFract function in Algorithm    \ref{alg:convexhull_search}, which finds the maximum point in two neighboring edges.

\qed
\end{theorem}

Now we can prove Theorem \ref{thm:chs_correctness} from Theorem \ref{thm:optimal_lambda} that at each iteration, the oracle returns one vertex of convex hull, and terminates if it returns the previously found vertex.

\subsection{Optimality}  
   In Theorem 1, the number of oracle calls is upper bounded by $|V|$.  To evaluate this result, we give a common example when  $g(y)\in\{0,1,\dots,m\}$ and $|\mathcal{Y}|=M^m$. This implies that  $Y$ coordinate of every vertex is an integer. Then,  $|V|\le 2m+1$, and the maximum number of oracle call is $m$ because for $\lambda>0$ only the half of the vertex can be returned\footnote{See Figure \ref{fig:all_points}. $\lambda$ determines the direction. With $\lambda\ge0$, only the right half of the vertices of the convex hull can be returned. With $\lambda<0$, the left half is returned.}. The number of oracle calls grows only {\em linearly} with the length $m$ of the label.  On comparison, bisecting search in [2], the number of the call cannot be upper bounded.  

$\lambda^*$ in Theorem \ref{thm:optimal_lambda} is the optimal in an adversarial setting. Specifically,  we can view our algorithm to find $\tilde{y}^*$ as a game between two players, a $\lambda$ querier versus a label revealer. The objective of the game is to find $\tilde{y}^*$ with the minimum number of the oracle queries, and for each $\lambda$, the label revealer returns a label $\mathcal{O}(\lambda)$ that is consistent with the previous queries. $\lambda^*$ used in the Algorithm \ref{thm:optimal_lambda} is the unique {\em optimum } to mandate the label revealer to reveal a new label or if the label revealer does not, we can terminate our game since if returned label is not new, it implies that the optimum is already previously returned. Following example in the lemma shows that using such $\lambda^*$ is crucial in minimizing the number of call to the oracle.  The lemma shows for a certain $\psi$, but the results can be generalized for other $\psi$ functions. Our algorithm
terminates in 3 iterations. 

\begin{lemma} For  $\psi$ function in the bi-criteria surrogate loss such that $\psi(a,b)=\psi(b,a)>0$,  there exists  an example with 3 distinct labels, $A$, $B$, and $C$, such that $\Phi(A)=\Phi(B)<\Phi(C)$, and  $\mathcal{O}(\lambda)$ oscillates between two points $A$ and $B$ until $\lambda^*=1$ is used, and only then  it  return $C$. Specifically, for iteration $i=1,2,\dots, T$, let $\lambda_i>0$ be the $\lambda$ used in iteration $i$, $\lambda_T=1$, and $\epsilon=\min_{i\le T-1} \frac{1}{2}|\lambda_i-1|\neq 0$.
Then, 
\begin{align*}
\mathcal{O}(\lambda_i)=
\begin{cases} A & \mbox{if } \lambda_i>1 \\
B & \mbox{o.w}
\end{cases}
\end{align*}
and  $\mathcal{O}(1)=C$.
\proof
 $A=[2,4], B=[4,2],$ and $C=[3+\epsilon,3]$. $\Phi(A)=\Phi(B)<\Phi(C)$ follows from the quasi concavity of $\psi$, and the monotonicity in the definition. See Figure \ref{sfig:opt}.
\end{lemma}
\begin{figure}[H]
\includegraphics[width=\linewidth]{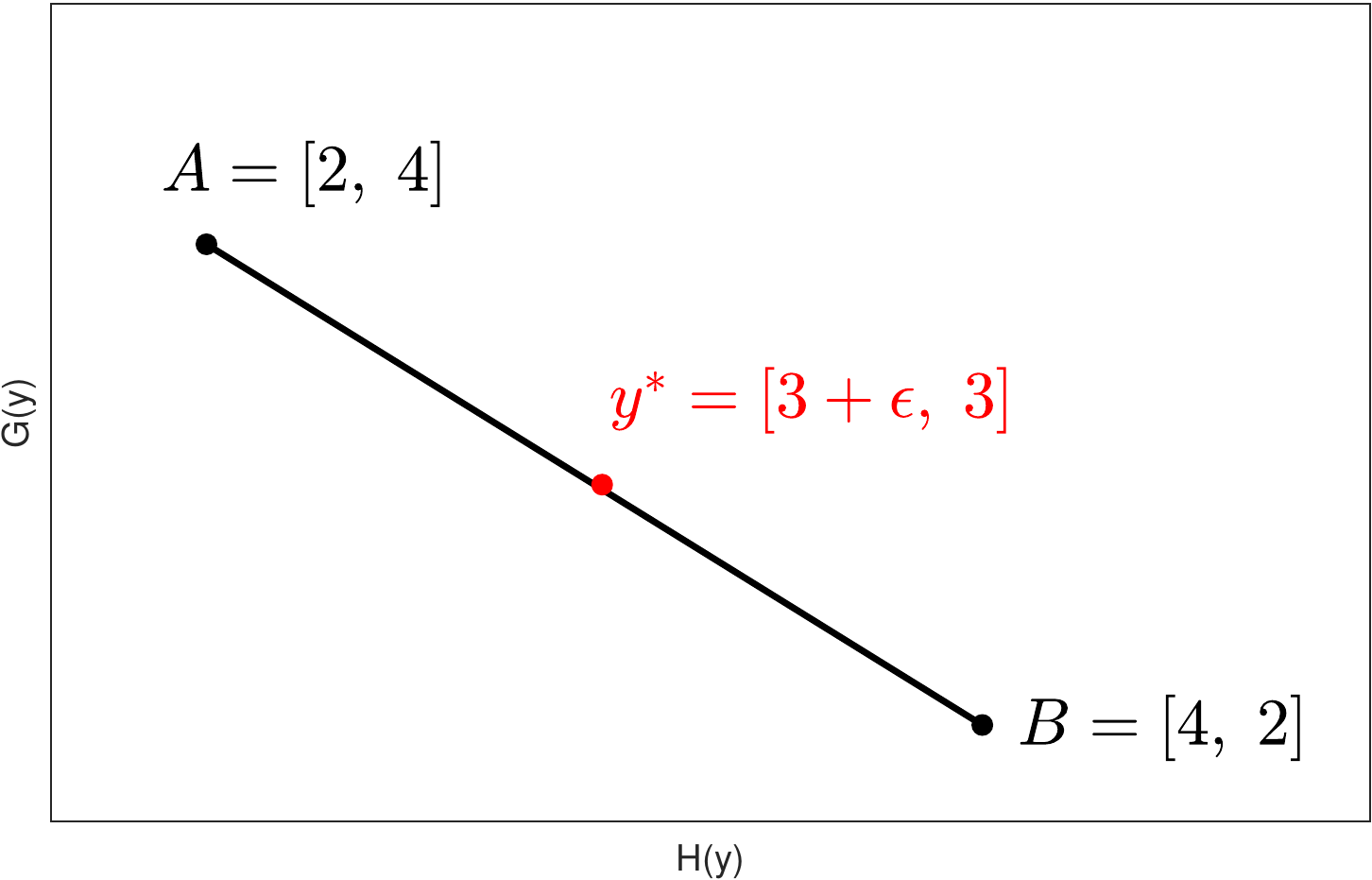}
\caption{Optimality of $\lambda^*$}\label{sfig:opt}
\end{figure}

 Another useful method to even facilitate  the process more in practice is that since our algorithm commonly used for an iterative optimization such as SGD, labels that comprised $V$ in the previous epoch can work as an initial guess of $V$ at the current iteration, and thus it is useful to memorize labels that comprised $V$ in the previous epoch. This can be useful as a warm start for our method.  This is similar to  {\em label cache} in \cite{choi2016fast}. For instance, we initially construct $V$ from previous labels that comprised $V$.  If $V$ does not change much, our algorithm can terminate at  iteration $1$ by checking $\hat{y}$ in the previous epoch is still validate $\hat{y}$.
For NER problem in \ref{sec:NER},  the average $2.0$ oracle calls were needed to obtain $y^*$ and it only took $2.3$ times more time than margin rescaling!

\section{Experiments} \label{sec:experiment}
In this section, we will empirically validate our approach in standard datasets by comparing with previous methods and showing the performance gain from new surrogate losses.
To guarantee obtaining an integral label, a more powerful oracle is required such as a k-best or an oracle with ban list, which will be discussed later. The oracles often exist for oracles that uses dynamic programming. We compare margin rescaling with Probloss for name entity recognition task and dependency parsing task where such oracle exists, and for the multi-label problem, we test with $\beta$ scaling surrogate.

\subsection{Comparison of $\lambda$-oracle optimization methods}

  \begin{table}[H]
  \begin{center}
    \begin{tabular}{|l|c| c  | c |c|c|c|}
    \hline
&      &  Sarawagi & Bisecting & Angular & Convex Hull\\
    \hline
    \multirow{ 3}{*}{Yeast}
&

      \# of queries &     15.3&  7.8&       4.5&    3.1  \\
   &Time  &         55.7& 20.5&   23.3&    9.5 \\
 & Fail-Max\% &        .36&  .26&  0 &    .07\\
    \hline
    \multirow{ 3}{*}{RCV1}
&

      \# of queries &  14.7  & 13  & 5.5 & 7.29  \\
   &Time  & 125& 87  &70& 53 \\
 & Fail-Max\% & .46& .26  & 0 & .02\\
    \hline

    \end{tabular}\\
  \end{center}
\caption{Comparison of $\lambda$ oracle optimization methods. %Convex hull search is fastest among all without the additional constraints of Angular search.
} \label{tabel:loo_comparison}
\end{table}

We compare the $\lambda$-oracle optimization methods for slack rescaling inference in the literature with respect to the number of the oracle calls, runtime, and the quality of the argmax label returned in Table \ref{tabel:loo_comparison}. Each column corresponds to the methods: Sarawagi\cite{sarawagi2008accurate}, Bisecting search\cite{choi2016fast}, Angular search\cite{choi2016fast}, and our method Convex hull search.  Each row is the number of average queries, average time spent, and the ratio that fails to achieve the maximum.

The quality of search is measured by the ratio that each method failed to be the maximum among the methods.  Similarly as done in \cite{choi2016fast}, we take $80$ instances  of the two multi-label  problems,
Yeast \cite{elisseeff2001kernel} % enron\footnote{\url{<http://www.cs.cmu.edu/~enron/>}}
 and RCV1 \cite{lewis2004rcv1}. For RCV1, labels are reduced to most frequent $30$ labels. To compare all the methods, the cutting-plane method optimization is used, and for each loss augmented inferences, we try all the $\lambda$-oracle optimization methods. As shown in Table \ref{tabel:loo_comparison}, while the Angular search is very accurate with the added constraints in the oracle, and has a small number of the oracle calls,  Convex hull search is the fastest since each oracle runs fast not having the additional constraints that slow oracle. Additionally, note that compared to Bisecting search, the number of calls to the oracle of Convex hull search is much smaller. Also since Bisecting search does not have a good termination criterion and it needs to stop when the binary search range is sufficiently small, which stops the search early, and hurts the quality of the argmax. Failure of convex hull search only happens in the very early stage of the optimization when $w$ is degenerate. While being the fastest search methods, note that Convex hull search can optimize quasi-concave functions other than slack rescaling formulation, and it does not require the oracle to have additional constraints as Angular search, which is not often available. 

   %Table \ref{table:label_cache}, show the effect of active set approach that stores labels appeared in the previous iteration. 

%We empirically test our method in two real-world tasks.

\subsection{Multi-label Classification}
 We experimented on a multi-label problem with a fully pairwise MRF model in a much larger scale than in \cite{finley2008training,choi2016fast}. Statistics of the dataset is shown in Table \ref{table:data_statistics}. Note that slack rescaling with 159 multi-labels is a challenging problem. One advantage  of our inference method,  flexibility of the surrogate loss, is demonstrated with the result of $\beta$-scaling. We test with the regularization constants $C\in \{10^{-4},10^{-3},10^{-2},10^{-1}\}$ with  one versus all (OVA), margin rescailing (MS), $\beta$-scaling ($\beta$-S) with $\beta\in\{1/4,1/2,3/4\}$, and slack rescaling (SS) on the  validation dataset, and report the best result on the test set. Macro-F1 seems to be sensitive to rounding scheme used for indecisive labels in test time, and omitted. Table \ref{table:ML_result} summarizes the result. $\beta$-scaling outperforms others, and it shows that a sensitive control of the surrogate loss can have much influence on the result. 
\begin{table}[H]
\begin{center}
\begin{tabular}{|c|c|c|c|c|}%ll}
\hline
& $d$ & $M$  & $N$ & $teN$ \\ %& $\bar{N}$ & $\bar{L}$ 
\hline 
Mediamill&120&     101&     30K&   129K \\%&  1902.15 &4.38\\
%\hline
Bibtex&1836 &159 &4880 & 2515 \\%&  111.71 & 2.40  \\
%\hline
RCV1& 299K& 103& 23K & 781K \\%& & 3.24
\hline
\end{tabular}
\caption{Data Statistics for multi-label problem.}
\label{table:data_statistics}
\end{center} 
\end{table}

\begin{table}[H]
\begin{center}
\begin{tabular}{|c|c | cc | cc |}
\hline
&  \multicolumn{1}{c|}{ Bibtex    }
&  \multicolumn{2}{c|}{ RCV1    }
&  \multicolumn{2}{c|}{ Mediamill   }
\\
\hline 
& MiF1 
& Acc& MiF1 %& MaF1
& Acc& MiF1 \\%& MaF1\\
\hline 
OVA& .423 
&.698 &.770 %&0.448
   & .404    &.534  % & 0.041   
\\
MR  & .419 
 & .721 & .785 %& .463
 & .404 & .533 %& .042
\\
$\beta$-S & {\bf.444 }
  & {\bf.733} & {\bf.791} %& {\bf0.476  }
& {\bf.418}  & {\bf.558}  %& {\bf0.074 }
\\
%$\beta$-S& 1.27 & 0.444 \\
%$\beta$-S(0.5)& 1.27 & 0.425 \\
%$\beta$-S(0.75)& 1.24 & 0.429 \\
SR  &.352 
 & {\bf .733} & .789 %& 0.438 
 & .416 & .556 % & 0.059
 \\
\hline
\end{tabular}
\caption{A fully pairwise MRF: By the precise control over the surrogate loss, $\beta$-scaling outperforms other models.    }
\label{table:ML_result}
\end{center}
\end{table}

\subsection{Name Entity Recognition (NER)}\label{sec:NER}
We apply our surrogate losses to the problem of named-entity recognition (NER) with English data of new articles from the CoNLL
2003 shared task \cite{TjongKimSang:2003:ICS:1119176.1119195}, which is done in  \cite{gimpel2010softmax}. The problem is to identify four entity types: person, location, organization, and miscellaneous. Following \cite{gimpel2010softmax}, we used additional token shape features and simple gazetteer features as in \cite{kazama2007new}. We obtain the code from the author, and implemented our ProbLoss. Similar to \cite{gimpel2010softmax}, for each method, for tuning the hyperparameters, regularization constant $C\in\{0.01,\; 0.1,\; 1,\;  10\}$, and learning rate $\rho\in \{0.01,\; 0.001\}$ is tested in a validation set after  training with SGD for $100$ epochs, and report the test result for the best parameter in a validation set. To obtain integral labels for ProbLoss, we perform a Viterbi path algorithm with $k$-ban list. Specifically, if Convex hull search results in a fractional label, we add the vertices that consist the edge of the fractional label into the ban list, and loop until an integral label is found or the potential of the fractional label is less than that of the best integral label that is previously found.  Compared to margin rescaling, which is a hard baseline for the task, our proposed loss function, ProbLoss, outperforms it regarding Macro-F1. We also tested with the Micro-F1 surrogate. Our search is very efficient. An average 2\ oracle calls were needed to obtain $y^*_{I}$ for ProbLoss. For results, see Table \ref{table:result_NER}, and proposed micro-F1 surrogate also outperforms it regarding Micro-F1.

\begin{table}[H]
\begin{center} 
\begin{tabular}{|c|c|}%c|}
\hline
& Macro-F1\\% & Micro-F1  \\ 
\hline 
Margin rescaling &    85.28 \\%& 62.82 \\
%\hline
ProbLoss& 85.62   \\%& - \\
%\hline
%Micro-F1 surrogate & - & 63.57\\
\hline
\end{tabular}

\begin{tabular}{|c|c|}
\hline
&  Micro-F1  \\ 
\hline 
Margin rescaling &     62.82 \\
%\hline
Micro-F1 surrogate  & 63.57\\
\hline
\end{tabular}
\caption{Result on NER}
\label{table:result_NER}
\end{center} 
\end{table}
\subsection{Dependency Parsing on Deep Network}
In the era of deep networks, we apply our method to dependency parsing problem in NLP as in \cite{kiperwasser2016simple} that uses a deep network. In dependency parsing tasks, we assign dependency on the words in a sentence as a  form of a tree. Each arc is labeled as a part of speech (POS) tag. We applied our method on the structured prediction part of a graph based parser which uses bidirectional-LSTMs (BiLSTM) with the codes from  \cite{kiperwasser2016simple}.  All the parameters are same as in \cite{kiperwasser2016simple} for margin rescaling, and trained and tested on the Penn Treebank   \cite{marcus1993building} dataset. We warm-start with margin rescaling and applied ProbLoss from iteration $10$ till $30$ with a learning rate of $0.0001$. To obtain the integral labels, we modify the oracle using the dynamic programming to return $k$-best labels varying $k$, and used the same ban-list strategy as in NER: in case of fractional label, we added the two vertices that consisted the fractional label  into the ban-list, and  increase $k$ by 2, and iterate until an integral label is found.  Table \ref{table:result_DP} summarizes the improvement over margin rescaling.

\begin{table}[H]
\begin{center} 
\begin{tabular}{|c|c|c|}
\hline
& Labeled   Attachment Score &Unlabeled Attachment Score  \\ 
\hline 
Margin rescaling &    92.6 & 94 \\
%\hline
ProbLoss& 92.7   & 94.2 \\
%\hline
%Micro-F1 surrogate & - & 63.57\\
\hline
\end{tabular}
\caption{Result on Dependency Parsing}
\label{table:result_DP}
\end{center} 
\end{table}

 We also tested for the statistical significance. We  tested with $10000$ sentences $10$ times,  and the test shows that the result is statistically significant with $0.0075$ p-value of a two-sided Wilcoxon signed rank test.
\begin{table}[H]
\begin{center} 
\begin{tabular}{|c|c|}
\hline
& Labeled   Attachment Score \\%&Unlabeled Attachment Score  \\ 
\hline 
Margin rescaling &    92.1$\pm$.1 \\%& 94 \\
%\hline
ProbLoss& 92.2$\pm$.1  \\%& 94.2 \\
%\hline
%Micro-F1 surrogate & - & 63.57\\
\hline
\end{tabular}
\caption{Statistical Significance Testing on Dependency Parsing}
\label{table:data_statistics}
\end{center} 
\end{table}

Figure \ref{fig:dp_dev} shows the improvement with ProbLoss during the training.
\begin{figure}[H]
\centering     
\includegraphics[width=\linewidth]{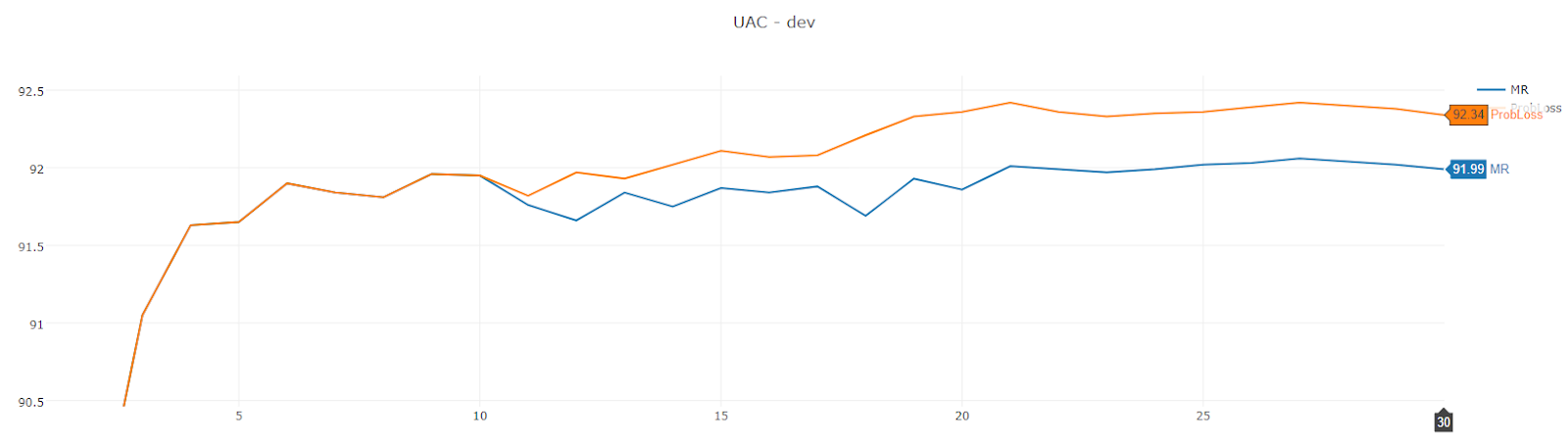}
\caption{Improvement of unlabeled attachment score with ProbLoss in the Development set.}\label{fig:dp_dev}
\end{figure}

\section{Conclusion} \label{sec:conclusion}
 In this chapter we presented an easy to implement and practical method for the  loss augmented inference of bi-criteria surrogate loss. The bi-criteria surrogate loss can be characterized by a quasi-concave relationship between two factors, which are usually structural loss and margin loss. Since it is previously shown that without the modification of the oracle, the exact optimal label cannot be returned. Thus, our method, convex hull search, focuses on LP relaxed space of labels. Efficiency is explained by  the optimality of the  choice of $\lambda$ of convex hull search in an adversarial game, which gives an upper bound of the number of the calls in  geometrical terms of labels, which is often very small in practice. We also presented a method for obtaining integral labels using a ban-list strategy, which is often more available than the constrained oracle  in \cite{choi2016fast}. We empirically compare other $\lambda$-oracle methods to show the efficiency of the convex hull search. We also showed an improvement utilizing our method in multi-label  problem, name entity recognition problem, and dependency parsing problem in a deep network. 

\maketitle

\begin{abstract}
We present improved methods of using structured SVMs in a large-scale
hierarchical classification problem, that is when labels are leaves,
or sets of leaves, in a tree or a DAG.  We examine the need to
normalize both the regularization and the margin and show how doing so
significantly improves performance, including allowing achieving
state-of-the-art results where unnormalized structured SVMs do not
perform better than flat models.  We also describe a further extension
of hierarchical SVMs that highlight the connection between
hierarchical SVMs and matrix factorization models.
\end{abstract}
\end{comment}
   In this chapter, we investigate the regularization issue in the hierarchical classification. The labels decompose into nodes in the hierarchy, and each classifier decomposes into micro-classifiers. In such a setting, we argue the importance of normalizing the regularization bias imposed by the structural imbalance  from the label structure. Then, we move forward and present a new norm that does not suffer from such structural imbalances, which incorporates not only label structure but also data. 

This chapter is mainly based on the previous publication \cite{choi2015normalized}.
       
\section{Introduction}

 We assume that each label decomposes into  fixed micro-labels, and the micro-labels corresponds to a path (or union of paths) from the root to a leave (or leaves) in a tree or a DAG (directed acyclic graph). Such problem is called hierarchical classification.  
The hierarchy represents grouping of labels which exists in real world datasets.  Such hierarchies
have been extensively used to improve accuracy
  \cite{Mccallum98improvingtext,Silla:2011:SHC:1937796.1937884,Vural:2004:HMM:1015330.1015427}
in domains such as document categorization \cite{cai2004hierarchical}, web content
classification \cite{Dumais:2000:HCW:345508.345593}, and image
annotation \cite{Huang:1998:AHI:290747.290774}.  In some problems,
taking advantage of the hierarchy is essential since each individual
labels (leaves in the hierarchy) might have only a few training
examples associated with it.

Hierarchical classifier is structured prediction problem where labels are subset of nodes that are union of paths from root to leaves.
Each nodes are the micro-classifers, and sum of the micro-classifers are final classifier. This belongs to our sum of classifier setting.

We focus on hierarchical SVM \cite{cai2004hierarchical}, which is a structured SVM
problem with the structure specified by the given hierarchy.
Structured SVMs are simple compared to other hierarchical
classification methods, and yield convex optimization problems with
straight-forward gradients.  However, as we shall see, adapting
structured SVMs to large-scale hierarchical problems can be problematic and requires care.
We will demonstrate that ``standard'' hierarchical SVM suffers from
several deficiencies, mostly related to lack of normalization with
respect to different path-length and different label sizes in
multi-label problems, which might result in poor performance, possibly
not providing any improvement over a ``flat'' method which ignores the
hierarchy.  To amend these problems, we present the Normalized
Hierarchical SVM (NHSVM).  The NHSVM is based on normalization weights
which we set according to the hierarchy, but not based on the data.
We then go one step further and learn these normalization weights
discriminatively.  Beyond improved performance, this results in a
model that can be viewed as a constrained matrix factorization for
multi-class classification, and allows us to understand the
relationship between hierarchical SVMs and matrix-factorization based
multi-class learning  \cite{amit2007uncovering}.
As a result, the new model suggests to resolve imbalance in the reguarization in the sum of micro-classifiers problem.

We also extend hierarchical SVMs to issues frequently encountered in
practice, such as multi-label problems (each document might be labeled
with several leaves) and taxonomies that are DAGs rather then trees.

We present a scalable training approach and apply our methods to large
scale problems, with up to hundreds of thousands of labels and tens of
millions of instances, obtaining significant improvements over
standard hierarchical SVMs and improved results on a
hierarchical classification benchmark.

%\begin{comment}
\section{Related Work}

Hierarchical classification using SVM is introduced in
 \cite{cai2004hierarchical}.  An extension to the multi-label case was presented
by  \cite{Cai_2007}, but optimization was carried out in the dual and
so does not scale, and no care was paid to the issue of normalization.
%\end{comment}

%\begin{comment}
In  \cite{Rousu06kernel-basedlearning}, dual program directly models kernel-based structured SVM in a multi-label hierarchical problem. This differs from our work since it only focus on dual objective. Our work focuses on direct optimization on primal objective, and
the number of the variables does not depend on the number of the instances of the
problem, which is more suitable for a large scale learning.
% And structure is limited to a tree.
%\end{comment}

%\begin{comment}
%Would be good to discuss a few non-SVM approaches.  Not only citations, but say what they are.  I couldn't understand from these sentences:

Alternatives to SVMs for hierarchical
classification include  \cite{Weinberger_largemargin} and  \cite{Vural:2004:HMM:1015330.1015427} method which divides
the data into two subsets with respect to hierarchy till every instance is
labeled to one. These models are only applicable to single label
problem, and has limited scalability for a large size data with a
large structure.
%\end{comment}

%\begin{comment}
In  \cite{Cesa-Bianchi:2006:HCC:1143844.1143867}, new loss function H-loss is introduced to evaluate the discrepancy of two labels, which prohibits prediction
error of the children node to add up if the parent node is already misclassified.
Then, a stochastic  Bayes-optimal classifier is
built according to the loss.
For training, SVM is trained at each node, and for inference message passing is used. 
%\end{comment}

\section{Label Structure}

Let  $\mathcal{G}$ be a tree or a {\em directed acyclic graph} (DAG) representing a label
structure with  $M$ nodes. Denote the set of leaves  nodes in $\mathcal{G}$ as  $\mathcal{L}$.   For each $n\in [M]$,
define the sets of parent, children, ancestor, and descendent nodes of $n$ as $\mathcal{P}(n)$, 
$\mathcal{C}(n)$, 
  $\mathcal{A}(n)$,
  and  $\mathcal{D}(n)$,  respectively. Additionally, denote the ancestor nodes of $n$  including node $n$ as
  $\overline{\mathcal{A}}(n)=\{n\}\cup
  \mathcal{A}(n)$, and similarly, denote $\overline{\mathcal{D}}(n)$ for $\mathcal{D}(n)=\{n\}\cup\mathcal{D}(n)$. We also extend the notation
above for  sets of nodes to indicate the union of the corresponding
sets, i.e.,
$\mathcal{P}(A)=\cup_{n\in A}\mathcal{P}(n)$.

Let $\{ (x_i,y_i)\}^N_{i=1}$ be the training data of $N$ instances.
Each $x_i \in \mathbb{R}^d$ is a feature vector and it is labeled with
either a leaf (in single-label problems) or a set of leaves (in
multi-label problems) of $\mathcal{G}$.  We will represent the labels
$y_i$ as subsets of the nodes of the graph, where we include the
indicated leaves and all their ancestors.  That is, the label space
(set of possible labels) is $\mathcal{Y}_{s}=\{\overline{\mathcal{A}}(l)|l\in
\mathcal{L} \}$ for single-label problems, and 
$\mathcal{Y}_{m}=\{\overline{\mathcal{A}}(L)|L\subseteq
\mathcal{L}\}$ for multi-label problems.

\section{Hierarchical  Structured SVM}
We review the hierarchical structured SVM
introduced in  \cite{cai2004hierarchical} and extended to the multi-label case in  \cite{Cai_2007}.
Consider $W\in\mathbb{R}^{M\times d}$, and let the $n$-th row vector $W_n$ be
be weights of the node $n\in [M]$.
 Define $\gamma(x,y)$ to be the  potential of label $y$ given feature $x$, which is the sum of the inner
products of $x$ with the weights of node $n\in y,$
$  \gamma(x,y)=\sum_{n \in y} W_n \cdot  x$.
If we vectorize $W$, $w=vec(W)=[W_1^T \; W_2^T\; \dots \; W_M^T]^T\in \mathbb{R}^{d\cdot M}$, and
define the class-attribute $\wedge(y)\in\mathbb{R}^M$, $[\wedge(y)]_n=1$ if $n\in y$
or 0 otherwise, then 
\begin{align} 
  \gamma(x,y)=\sum_{n \in y} W_n \cdot  x=w\cdot (\wedge(y)\otimes x) \label{classa}
\end{align} 
  where $\otimes$ is the Kronecker product. 
With weights $W_n$, prediction of an instance $x$ amounts to finding
the maximum response label 
$\hat{y}(x)={\arg\max}_{y \in \mathcal{Y}} \gamma(x,y)={\arg\max}_{y \in \mathcal{Y}} \sum_{n\in y} W_n x \label{argmax}$.
 Given a structural error $\triangle(y',y)$, for instance a hamming
 distance $\triangle^{H}(y',y)=|y'-y|= \sum_{n\in [M]}
 |\textbf{1}_{n\in y'}-\textbf{1}_{n\in y}|$, a training a
 hierarchical structured SVM amounts to optimizing:
\begin{align}
&\min_W\lambda \underset{n}{\sum} \|W_n\|^2_2+ 
\underset{i}{\sum}\max_{y\in\mathcal{Y}}\left \{ \sum_{n \in y}W_n x_i -\sum_{n \in y_i} W_n x_i  +\triangle(y,y_i)\right
\} \label{HSVM}\end{align}
Equivalently, in terms of $w$ and class-attribute $\wedge(y)$,
\begin{align}
&\min_W\lambda  \|w\|^2_2
+ \underset{i}{\sum}\max_{y\in\mathcal{Y}}\left \{ w\cdot ((\wedge(y)-\wedge(y_i))\otimes x_i) +\triangle(y,y_i)\right\}. \label{HSVM2}\end{align}

  \section{Normalized Hierarchical SVM}\label{sec:NHSVM}

%\begin{wrapfigure}{l}{.9\textwidth}
%\begin{comment}
\begin{figure}
      \includegraphics[width=\textwidth]{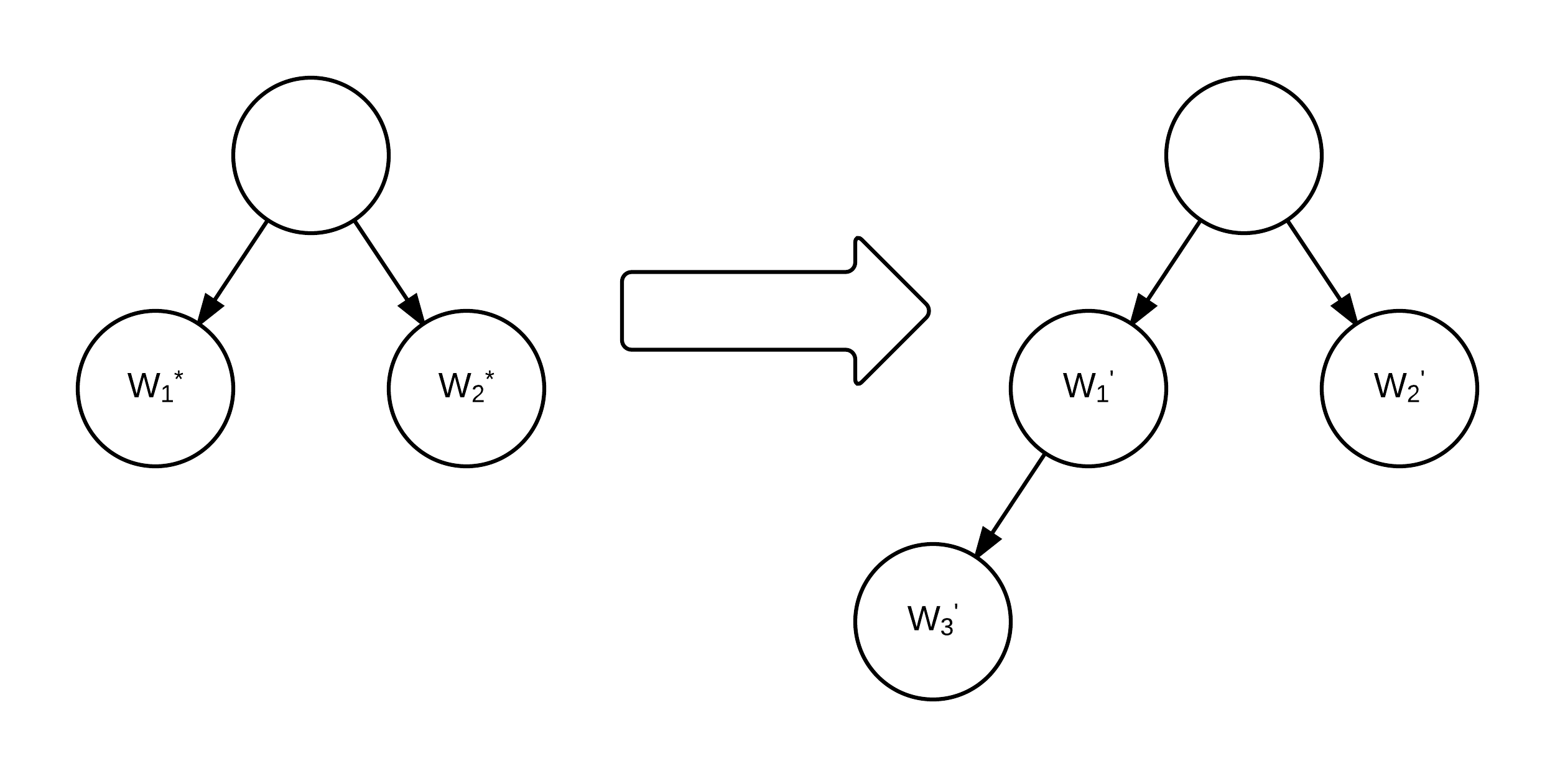}
    \caption{Regularization penalty for label $y_1$ (left branch) 
    %$\|W_1^*\|^2_2$  
    is halved to 
    %$\|W_1'\|^2_2+\|W_3'\|^2_2$ 
    without
 changing decision boundary due to difference in the label structure. }
  \label{fig:nregularization}
\end{figure}
%\end{wrapfigure}
%\end{comment}

A major issue we highlight is that unbalanced structures (which are
frequently encountered in practice) lead to non-uniform regularization
with the standard hierarchical SVM.  To illustrate this issue,
consider the two binary tree structures with two leaves shown in figure
\ref{fig:nregularization}.  Implicitly both structures describes the
same structure. Recall that the regularization penalty is
$\|W\|^2_F=\sum_n\|W_n\|^2_F$ where each row of $W$ is a weight vector
for each node. In the left structure, the class attributes are $\wedge(y_1)=[1\; 0]^T$, and
$\wedge(y_2)=[0\; 1]^T$, assume $\|x\|_2=1$, and let the optimal
weights of node 1 and node 2 in the left structure be $W_1^*$ and
$W_2^*$. Now add a node 3 as a child of node 1, so that $M=3,
\wedge(y_1)=[1\; 0\; 1]^T,\wedge(y_2)=[0\; \; 1\; 0]^T$. Let $W_1'$
and $W_3'$ be the new weights for the nodes 1 and 3. Taking
$W_1'=W_3'=\frac{1}{2}W^*_1$, and the potential function and thus the
decision boundary remains the same, however the regularization penalty
for $y_1$ is halved, 
$\|W_1'\|^2_2+\|W_3'\|^2_2=\frac{1}{2}\|W_1^*\|^2_2$, and
$\|W^*\|_F^2>\|W'\|_F^2$. This can be generalized to any depth, and
the regularization penalty can differ arbitrarily for the model with
the same decision boundary for different structures. In the given
example, the structure on the right imposes half the penalty for the
predictor of $y_1$ than that of $y_2$.
 
The issue can also be understood in terms of the difference between
the norms of $\wedge(y)$ for $y\in\mathcal{Y}$.
Let  $\phi(x,y)\in \mathbb{R}^{d\cdot
M}$ the feature map for an instance vector $x$ and a label $y$ such that $\gamma(x,y)=w\cdot \phi(x,y)$.
From (\ref{classa}),
\begin{align*}
w \cdot (\wedge(y)\otimes x)= w\cdot \phi(x,y)
\end{align*}
  
$\wedge(y)\otimes x$ behaves as a feature map in hierarchical structured SVM.
While the model regularizes $w$, the norm of $\phi(x,y)$ is different for $y$ and scales as $\|\wedge(y)\|_{2}$.
\begin{align*}
\|\gamma(x,y)\|_2=\|\wedge(y)\otimes x\|_2=\|\wedge(y)\|_2\cdot\|x\|_2
\end{align*}

Note that  $\|\wedge(y)\|_{2}=\sqrt{|\overline{\mathcal{A}}(y)|}$ % and norm of feature
and the differences in regularization can grow linearly with the depth
of the structure.
 
To remedy this effect, 
for each node $n$ we introduce a weight $\alpha_n\ge0$ 
   such that 
\begin{align}
 \sum_{n\in \overline{\mathcal{A}}(l)}\alpha_n=1, \;\;\;\;\;\; \forall l\in \mathcal{L}.\label{alpha}
\end{align}
Given such weights, we define  the normalized class-attribute $\tilde{\wedge}(y)\in \mathbb{R}^M$
 and  the normalized feature map $\tilde{\phi}(x,y)\in \mathbb{R}^{d\cdot M},$
\begin{align} \label{nclassa}
 &[\tilde{\wedge}(y)]_n=\begin{cases} \sqrt{\alpha_n} & \mbox{if } y\in n\\ 0 & \mbox{otherwise}\end{cases} &
&\tilde{\phi}(x,y)=\tilde{\wedge}(y)\otimes x %\label{phi}
\end{align}
The norm of these vectors are normalized to 1, independent
of  $y$, i.e., $\|\tilde{\wedge}(y)\|_2=1,$ $\|\tilde{\phi}(x,y)\|_2=\|x\|_2$
for $y\in \mathcal{Y}_{s}$, and the class attribute for each node $n$ is fixed to
$0$ or $\sqrt{\alpha_n}$ for all labels.
The choice of $\alpha$ is crucial and we present several alternatives
(in our experiments, we choose between them using a hold-out set). For
instance, using $\alpha_n= 1$ on the leaves $n\in \mathcal{L}$ and 0
otherwise will recover the flat model and lose all the information in
the hierarchy.  To refrain from having a large number zero weight and
preserve the information in the hierarchy, we consider setting
$\alpha$ optimizing:
\begin{equation}
\begin{aligned}\label{beta_1}
&\min &&\sum \alpha_n^{\rho}\\ 
&\mbox{s.t.}& &\sum_{n\in \overline{\mathcal{A}}(l)}  \alpha_n=1, && \forall  l\in \mathcal{L}\\
&&&\alpha_n\ge0 &&\forall n\in [M]
\end{aligned}
\end{equation}
%where $\epsilon$ is small positive scalar parameter, $\epsilon>0$.
where $\rho>1$.  In Section \ref{sec:invar}, we will show that as
$\rho\rightarrow 1$, we obtain weights that remedy the effect of the
redundant nodes shown in
Figure \ref{fig:nregularization}.  

We use \eqref{beta_1} with $\rho=2$ as a possible way of setting the
weights.  However, when $\rho=1$, the optimization problem
\eqref{beta_1} is no longer strongly convex and it is possible to recover weights of zeroes for most nodes.  Instead, for $\rho=1$, we consider the alternative
optimization for selecting weights:
\begin{equation}
\begin{aligned}\label{nbeta_1}
&\max  & &\min_n\alpha_n&&\\ 
&\mbox{s.t.}& &\sum_{n\in \overline{\mathcal{A}}(l)  }  \alpha_n=1, \;\;\; &&\forall  l\in \mathcal{L}\\
&&&\alpha_n\ge0, &&\forall n\in [M] \\ 
&&&\alpha_n\ge\alpha_p,& &\forall n\in [M], \forall p\in \mathcal{P}(n)
\end{aligned}
\end{equation} 
%can be used instead.
We refer to the last constraint as a ``directional constraint'', as it
encourage more of the information to be carried by the leaves and results more even distribution of $\alpha$.%is necessary for ensuring an unique optimum.

For some DAG structures, constraining the sum $\sum_{n\in
  \overline{\mathcal{A}}(l)}\alpha_n$ to be exactly one can
result in very flat solution.  For DAG structures we therefore relax
the constraint to
\begin{align}\label{con:range}
  1\le \sum_{n\in \overline{\mathcal{A}}(l) }  \alpha_n\le T, & &\forall  l\in \mathcal{L}.
 \end{align} 
for some parameter $T$ ($T=1.5$ in our experiments).
%\subsection{ Nonuniform Margin}

Another source of the imbalance is the non-uniformity of the required margin, which results from the norm of the differences of class-attributes, $\|\wedge(y)-\wedge(y')\|_2$.
The loss term of each instance in (\ref{HSVM2}) is,
%\begin{align*}
$\max_{y\in \mathcal{Y}}w \cdot (\wedge(y)-\wedge(y_{i}))\otimes x+\triangle(y,y_i)$.
%\end{align*}
And to have a zero loss $\forall y\in \mathcal{Y}$,
\begin{align*}
\triangle(y,y_i)\le
w\cdot ((\wedge(y)-\wedge(y_{i}))\otimes x) 
\end{align*}
$\triangle(y,y_i)$ works as the margin requirement to have a zero loss for $y$. 
The RHS of the bound scales as norm of $\wedge(y)-\wedge(y_i)$
scales. % for difference pair of $y$ and $y_i$. 
\begin{comment}

Assuming the simple tree structure as before, the maximum difference, $\max_{y_{i},y,y'\in \mathcal{Y}',y',y\neq y_i}\frac{\|\wedge(y')-\wedge(y_i)\|_2}{\|\wedge(y)-\wedge(y_i)\|_2}$, also grows approximately $\sqrt{d}$ where $d$ is the depth of the structure.

\end{comment}
 This calls for the use of structural
error that scales with  the bound. 
Define normalized structural
error $\tilde{\triangle}(y,y_i)$
\begin{align}
\tilde{\triangle}(y,y_i)=\|\tilde{\wedge}(y)-\tilde{\wedge}(y_i)\|=\sqrt{\sum_{n\in
y\triangle y_i} 
\alpha_n} \label{wserror}
\end{align}
and $y\triangle y'=(y_i-y) \cup (y-y_i)$
, and $\tilde{\wedge}(y)$
and
$\alpha$ are defined in (\ref{nclassa})(\ref{beta_1}).
Without the normalization, this is the square root of the hamming distance, and is similar to a tree induced distance in  \cite{dekel2004large}. This view of nonuniform margin gives a justification that the square root of hamming distance or tree induced
distance is preferable than hamming distance.
  
\subsection{Normalized Hierarchical SVM model}
 Summarizing the above discussion, we propose the Normalized
 Hierarchical SVM (NHSVM), which is given in terms of the following objective:
 \begin{align}
\min_W\lambda \underset{n}{\sum} \|W_n\|^2_2+ \underset{i}{\sum}\max_{y\in\mathcal{Y}}  \sum_{n \in y}\sqrt{\alpha_n}W_n x_i -  
 &\sum_{n \in y_i} \sqrt{\alpha_n}W_n x_i  +\tilde{\triangle}(y,y_i)
 \label{NHSVM}
\end{align}
 Instead of imposing a weight for each node, with change of variables $U_n=\sqrt{\alpha_n}W_n$,
 we can write optimization (\ref{NHSVM}) as changing regularization, 
 \begin{align}\label{NHSVM2}
\min_U\lambda \underset{n}{\sum} \dfrac{\|U_n\|^2_2}{\alpha_n}+ \underset{i}{\sum}\max_{y\in\mathcal{Y}}  \sum_{n \in y}U_n x_i -  
 &\sum_{n \in y_i} U_n x_i  +\tilde{\triangle}(y,y_i)
\end{align}
Also optimization \eqref{NHSVM} is equivalently written as 
 \begin{align}
\min_W\lambda  \|w\|^2_2+ \underset{i}{\sum}\max_{y\in\mathcal{Y}}w\cdot ((\tilde{\wedge}(y)-\tilde{\wedge}(y_i))\otimes x_i)+\tilde{\triangle}(y,y_i)
% \}
\end{align}
 Note that
for the single-label problem,  normalized hierarchical SVM can be viewed as a multi-class SVM changing the feature map function to  (\ref{nclassa}) and the loss term  to \eqref{wserror}. Therefore, it can be easily applied  to problems where flat SVM is used, and
also popular optimization method for SVM, such as  \cite{Shalev-Shwartz:2007:PPE:1273496.1273598}  \cite{lacoste2013block},  can be used.

Another possibly variant of optimization (\ref{NHSVM2}) which we experiment with is
obtained by dividing inside the max with
 $\|\tilde{\wedge}(y)-\tilde{\wedge}(y_i)\|_2$:
   \begin{align}
\min_W\lambda  \|w\|^2_2+ \underset{i}{\sum}\max_{y\in\mathcal{Y}}w\cdot \left (\frac{\tilde{\wedge}(y)-\tilde{\wedge}(y_i)}{\|\tilde{\wedge}(y)-\tilde{\wedge}(y_i)\|_2}\otimes x_i\right )+1
 \label{NHSVM3}
\end{align}
 There are two interesting property of the optimization \eqref{NHSVM3}. The norm of the vector right side of $w$ is normalized  $\left \|\frac{\tilde{\wedge}(y)-\tilde{\wedge}(y_i)}{\|\tilde{\wedge}(y)-\tilde{\wedge}(y_i)\|_2}\otimes x_i\right \|_2=\|x_i\|_2$. Also the loss term per instance at the decision boundary, which is also the required margin, is normalized to 1. However, because normalized class attribute in (\ref{NHSVM3}) does not decompose w.r.t nodes as in (\ref{NHSVM}), 
loss augmented
inference in (\ref{NHSVM3}) is not efficient for multi-label problems.

\subsection{Invariance property of the normalized hierarchical SVM}\label{sec:invar}

As we saw in figure \ref{fig:nregularization}, different hierarchical structures  can be used to  describe the same data, %For
%instance, in one dataset {\em action } movie is  a label, and in the other
%dataset  action movie is further categorized into {\em
%cop-action} movie and {\em
%hero-action} movie. 
and this causes undesired regularization problems. However, this is a common problem in real-world datasets. For
instance, a {\em action } movie label can be further categorized into a {\em
cop-action} movie and a {\em
hero-action} movie in one dataset whereas the other dataset uses a action movie as a label. 
Therefore, it is desired for the learning method of
hierarchical
model to adapts to this differences and learn a similar model if given dataset describes
similar  data.
 Proposed normalization can be viewed as an adaptation to this kind of distortions. In particular,  we show that NHSVM is invariant to  the node duplication.  

 Define duplicated nodes as follows. 
Assume that there are no unseen nodes in the dataset, i.e., $\forall n\in\mathcal{[M]},\exists
i, n\in \mathcal{A}(y_i)$.
Define two  nodes  $n_1$ and $n_2$ in $[M]$ to be {\em duplicated}  if $\forall i, 
 n_1\in y_i\iff n_2 \in y_i$.
Define the minimal graph $M(\mathcal{G})$ to
be  the graph having  a representative node per each duplicated
node set by merging each duplicated
node set to a node. 

\begin{thm}[Invariance property of NHSVM] \label{thm:invarinace}
Decision
boundary of NHSVM with $\mathcal{G}$ is arbitrarily close to that   of NHSVM with the minimum graph $M(\mathcal{G})$ as $\rho$ in (\ref{beta_1}) approaches 1, $\rho>1$. 
\proof 
We prove by showing that  for any $\mathcal{G,}$ variable $\alpha$ in (\ref{beta_1}) can
be reduced to one variable per each set of  duplicated nodes in
$\mathcal{G}$ using the optimality conditions, and optimizations    (\ref{beta_1})(\ref{NHSVM}) are equivalent to the corresponding optimizations of  $M(\mathcal{G})$  by  change of
the variables .

Assume there are no duplicated leaves, however, the proof can be easily generalized
for the duplicated leaves by introducing an additional 
constraint on $\mathcal{Y}$. 

%Here, introduce new notations for new graph $M(\mathcal{G})$. 
Let $\mathcal{F}(n')$ be a mapping from node $n'$ in graph $M(\mathcal{G})$ to a corresponding  set of duplicated nodes in $\mathcal{G}$.
Denote the set of nodes in $\mathcal{G}$ as $\mathcal{N}$, and the set of nodes in $M(\mathcal{G})$ as $\mathcal{N}'$,
 and the set of leaves in $M(\mathcal{G})$ as $\mathcal{L}'$.

%A set of duplicated leaves should treated as a single leaf, 
%if an ancestors of proper subset of a duplicated leaves can be used as a
%label, it describes different model. 
%Add a constraint that can should to restrain from partial nodes of duplicated %leaves being  a label add a 
%constraint for $\mathcal{Y}$ in (\ref{y_mc})(\ref{y_m}). Let $\mathcal{L}^D$
%be the set of leaves in $\mathcal{G}$ such that duplicated leave are  $\mathcal{L}^D=\{l|l\subseteq
%\mathcal{L}, \forall i, \forall l_1,l_2\in l, l_1\in y_i \iff l_2\in y_i\}$.
%\begin{align*}
%\mathcal{Y}_{MC}=\{\mathcal{A}(l)|l\in \mathcal{L}\}\cup\{\mathcal{A}|l\subseteq %\mathcal{L},\forall (l_1,l_2) \in
%\mathcal{D}_l,l_1\in l \iff l_2 \in l_2\}\\
%\mathcal{Y}_{ML}=\{\mathcal{A}(l)|l\subseteq \mathcal{L},\forall (l_1,l_2) %\in
%\mathcal{D}_l,l_1\in l \iff l_2 \in l_2\}
%\end{algin*}
% First we will show for multiclass case without no duplicated leaves, i.e., %$\mathcal{Y}=\{ \mathcal{A}(l)|l\in
%\mathcal{L}\}=\{ \mathcal{A}(l)|l\in
%\mathcal{L'}\}.$ 
%Because  $\mathcal{G}$ is duplicated graph of $M(\mathcal{G})$,
% there exists one to one mapping $\mathcal{F}$ from a node $n'\in\mathcal{N}'$
%  to set of leaves in $\mathcal{L}'$. such that $\mathcal{F}(l')=L, L\subseteq %\mathcal{L}$  duplicated label mapping node set $D_{n'}$ we can define duplicity between graphs as $n' \sim n$ such that  $n' \sim n$ is true if when $n'$ and  $n$ are duplicated nodes, for $n\in \mathcal{N}$ and $n\in \mathcal{N}'$, and denote $D_n=\{n'|n'\sim n\}$.

Consider (\ref{beta_1}) for $\mathcal{G}$. Note that (\ref{beta_1}) has a constraint on sum of $\alpha_{n}$ to be 1 for $n\in\{n\in \bar{\mathcal{A}}(l)| l\subseteq\mathcal{L}\}$. By the
 definition of the  duplicity, if two nodes $n_1$ and $n_2$ are duplicated nodes,
they are the ancestors of the same set of the leaves, and term $\alpha_{n_1}$ appears in the first constraints of (\ref{beta_1}) if and only if term $\alpha_{n_2}$ appears, thus we conclude that all the duplicated nodes will appear altogether. Consider a change
of variable for each $n'\in \mathcal{N}'$ 
\begin{align}
K_{n'}=\sum_{n \in \mathcal{F}(n')}\alpha_{n}  \label{sub_k}
\end{align}
  Then, (\ref{beta_1})   are functions of $K_{n'}$ and (\ref{beta_1})
 decompose w.r.t $K_{n'}$. From the convexity of function $x^\rho$ with $\rho>1$,
  $x>0$,
 and Jensen's inequality, $(\frac{1}{|\mathcal{F}(n')|}K_{n'})^\rho\le\frac{1}{|\mathcal{F}(n')|}\sum_{n\in
 \mathcal{F}(n')}\alpha_n^\rho$,
 minimum of (\ref{beta_1}) is attained when $\alpha_{n}= \frac{1}{|\mathcal{F}(n')|}K_{n'}$ for $\forall n\in \mathcal{F}(n')$. As $\epsilon$ approaches  0, where\ $\epsilon=\rho-1>0$,
\begin{align}
\sum_{n\in\mathcal{N} } \alpha_{n}^{\rho }=\sum_{n'\in \mathcal{N}'} |\mathcal{F}(n')  |\left(
 \dfrac{K_{n'}}{|\mathcal{F}(n')|}\right )^{\rho} =|\mathcal{F}(n')|^{\epsilon}  K_{n'}^{\rho} \label{KK}
%\approx \sum_{n\in\mathcal{N}'} K_{n'}^{\rho}
\end{align}

Plugging (\ref{KK}) (\ref{sub_k}) into (\ref{beta_1}), 
\begin{align*}
&\min&& \sum_{n'\in \mathcal{Y}'} K_{n'}^{\rho}\\
&\mbox{s.t.}&& \sum_{n'\in y'} K_{n'}=1, && \forall\ y'\in \mathcal{Y}'
\end{align*}
 These formulations are same as (\ref{beta_1}) 
 for $M(\mathcal{G})$.
 
  Thus  given $n'$,  $\alpha_{n}= \frac{K_n}{|\mathcal{F}(n')|}$ is fixed 
for
$\forall n \in \mathcal{F}(n')$, and with the same argument for $W_{n}$ in (\ref{NHSVM}), change of variables gives , $W'_{n'}=\sum_{n \in \mathcal{F}(n')}W_{n}$. Then (\ref{NHSVM}) is a minimization  w.r.t $W'_{n'}$, 
and the minimum is when  $W_{n}= \frac{W'_{n'}}{|\mathcal{F}(n')|}$ for $\forall n\in \mathcal{F}(n')$, plugging
this in (\ref{NHSVM}),
\begin{align}
\lambda \sum_{n\in \mathcal{N}}|\mathcal{F}(n')|&\dfrac{\|W'_{n'}\|^2_2}{|\mathcal{F}(n')|^2} +
\sum_i\max_{y\in \mathcal{Y}}\left (\sum_{n\in y} |\mathcal{F}(n')|\cdot
\ \sqrt{\frac{K_n}{|\mathcal{F}(n')|}}  \cdot
\frac{W'_{n'}}{|\mathcal{F}(n')|}\right.  \nonumber \\
 &-\left. \sum_{n\in y_i}|\mathcal{F}(n')|\cdot\sqrt{\frac{K_{n'}}{|\mathcal{F}(n')|}} \frac{W'_{n'}}{|\mathcal{F}(n')|}\right )\cdot x_i 
%\\
+\tilde{\triangle}(y,y_i)\label{plug_W}
\end{align}
%\lambda \sum_{n\in \mathcal{N}}|\mathcal{F}(n')|\dfrac{\|W'_{n'}\|^2_2}{|\mathcal{F}(n')|^2} %+
%\sum_i\max_{l\in \mathcal{L}}\left (\sum_{n\in \mathcal{A}(l)} |\mathcal{F}(n')|\cdot
%\right.
%\\
% \sqrt{\frac{K_n}{|\mathcal{F}(n')|}}  \cdot
%\frac{W'_{n'}}{|\mathcal{F}(n')|}-
% \sum_{n\in y_i}|\mathcal{F}(n')|\cdot\sqrt{\frac{K_{n'}}{|\mathcal{F}(n')|}} \left .\frac{W'_{n'}}{|\mathcal{F}(n')|}\right )\cdot x_i \\
%+\triangle^{K}(\mathcal{A}(l),y_i)  
%where $\triangle^{K}(y_1,y_2)= \sqrt{\sum_{n\in
%y_1\triangle y_2} 
%\frac{K_n}{|\mathcal{F}(n')|} }$. 
By substituting $W_n''=\frac{1}{\sqrt{|\mathcal{F}(n')|}}W'_{n'}$,
\begin{align*}
\eqref{plug_W}=\lambda \sum_{n} \|W''_n\|^2_2 +\sum_i\max_{y\in\mathcal{Y}}
\left (\sum_{n\in y}  
\sqrt{K_{n'}} W''_n-
\sum_{n\in y_i}\sqrt{K_{n'}} W''_n\right )
%\\&
\cdot x_i+\tilde{\triangle}(y,y_i) 
\end{align*}

(\ref{NHSVM}),(\ref{beta_1}) for $\mathcal{G}$ are equivalent
to those of $M(\mathcal{G})$, thus two solutions are equivalent with a change of
variables and the decision boundaries are the same.

\qed  
\end{thm}

\section{Shared SVM: Learning with Shared Frobenius norm}\label{sec:SharedSVM}

In the NHSVM, we set the weights $\alpha$ based the graphical structure
of the hierarchy, but disregarding the data itself.  We presented
several options for setting the weights, but it is not clear what the
best setting would be, or whether a different setting altogether would
be preferable.  Instead, here we consider discriminative learning the
weights from the data by optimizing a joint objective over the weights
and the predictors.  The resulting optimization is equivalent to  regularization with a new norm which we  call {\em shared Frobenius norm}  or {\em shared norm}. It explicitly  incorporates the information of  the label structure $\mathcal{G}$. Regularization with the shared Frobenius norm promotes the models to utilize shared information, thus it  is a complexity measure  suitable for  structured learning. An efficient algorithm for tree structure is discussed in section \ref{sec:opt}.

Consider the
formulation \eqref{NHSVM2} as a joint optimization over both $\alpha$ and
$U=[U_1^T\; U_2^T\; \dots U_M^T]^T$ with fixed
$\tilde{\triangle}(y,y_i)=\triangle(l,l_i)$ (i.e.~we no longer
normalize the margins, only the regularization):
\begin{equation} 
\begin{aligned} \label{SSVMp}
&\min_{U,\alpha}&& \lambda \underset{n}{\sum} \dfrac{\|U_n\|^2_2}{\alpha_n}+ \underset{i}{\sum}\max_{l\in[Y]}  \sum_{n \in \overline{\mathcal{A}}(l)}U_n x_i -  
 \sum_{n \in \overline{\mathcal{A}}(l_i)} U_n x_i  +\triangle(l,l_i) && \\
&\mbox{s.t.}& &\sum_{n\in \overline{\mathcal{A}}(l)}  \alpha_n\le1, && \forall  l\in [Y]\\
&&&\alpha_n\ge0 &&\forall n\in [M]
\end{aligned}
\end{equation}
We can think of the first term as a regularization norm
$\|\cdot\|_{s,\mathcal{N}}$ and write
\begin{align} \label{ssvm}
&\min_U\lambda  \|U\|^2_{s,\mathcal{N}}+ 
\underset{i}{\sum}\max_{l\in |Y|}U_l\cdot x_i - U_{l_i}\cdot x_i  +\triangle(l,l_i) \end{align}
where the the {\em structured shared Frobenius norm} $\|\cdot\|_{s,\mathcal{N}}$ is defined as:
\begin{equation}\label{rho2}
\begin{aligned}
 \|U\|_{s,\mathcal{N}}=&\min_{ a\in\mathbb{R}^M,W\in \mathbb{R}^{M\times d}}& &\|A\|_{2\rightarrow\infty}\|V\|_F   \\
&\mbox{s.t. } & &AV=U\\ 
&&& A_{l,n}= \begin{cases}0 & otherwise\\ a_n
& n\in \overline{\mathcal{A}}(l)\end{cases} \\
&&& a_n\ge0, & \forall n\in [M]
\end{aligned}
\end{equation}
where $\|A\|_{2\rightarrow \infty}$ is the maximum of the $\ell_2$ norm of
row vectors of $A$.  Row vectors of $A$ can be viewed as coefficient
vectors, and row vectors of $W$ as factor vectors which decomposes the
matrix $U$.  The factorization is constrained, though, and must
represent the prescribed hierarchy.  We will refer (\ref{ssvm}) to
{\em Shared SVM} or {\em SSVM}.

To better understand the SSVM, we can also define the {\em shared Frobenius norm}  without the structural constraint as 
\begin{align}
 \|U\|_{s}=\min_{ AV=U} \|A\|_{2\rightarrow \infty}\|V\|_F   \label{eq:snorm1}
\end{align}
The {\em Shared Frobenius norm} is a norm between the trace-norm (aka
nuclear norm) and the max-norm (aka $\gamma_2:1\rightarrow\infty$
norm), and an upper bounded by Frobenius norm:
\begin{thm}\label{thm:compare}
For $\forall U\in \mathbb{R}^{r\times c}$
\begin{align*}
 \dfrac{1}{\sqrt{rc}}\|U\|_* \le\dfrac{1}{\sqrt{c}}\|U\|_s\le\|U\|_{\max},
 \;\;\;\;\|U\|_{s}\le\|U\|_{s,\mathcal{N}}\le \|U\|_F
\end{align*}
where $\|U\|_*=\min_{AW^T=U}\|A\|_F\|W\|_F$ is then the trace norm, and $\|U\|_{\max}=$
$\min_{AW^T=U}$ $\|A\|_{2\rightarrow
\infty}\|W\|_{2\rightarrow
\infty}$ is  so-called the max norm \cite{srebro2005rank}.

\proof The first  inequality follows from the fact that $\frac{1}{\sqrt{r}}\|U\|_{F}\le \|U\|_{2\rightarrow\infty}$,
and the second  inequality  is from taking $A=I$, or $A_{l,n}=1$ when $n$ is  an
unique node for $l$ or 0 for all other nodes in (\ref{rho2}) respectively.  \qed
\end{thm}

We compare the shared norm to the other norms to illustrate the behavior
of the shared norm, and summarize in table \ref{compare_norms}.  An
interesting property of the shared norm is that if there is no possible sharing, it reduces to Frobenius norm, which is the norm used for multi-class SVM.  This differs from the trace norm, which we can see from specifically in disjoint feature in table \ref{compare_norms}. Therefore, this property justifies the view of SSVM that it extends multi-class SVM
to shared structure.

         \begin{table}[h!]
\begin{center}
\begin{tabular}{ |l|c | c | c| c| }
\hline
   & $\|U\|_{s}$& $\|U\|_{s,\mathcal{N}}$ & $\|U \|_F$ & $\|U\|_*$\\
\hline
  Full sharing & $\|u\|_2$ & $\|u\|_2$ & $\sqrt{Y}\|u\|_2$ & $\sqrt{Y}\|u\|_2$\\
\hline
  No sharing  & .& $\|U\|_F$ & $\|U\|_F$ & $\cdot$ \\
  \hline
  Disjoint feature & $\sqrt{\sum_l \|u_l\|_2^2}$& $\sqrt{\sum_l \|u_l\|_2^2}$ & $\sqrt{\sum_l \|u_l\|^2_2}$ & $\sum_l \|u_l\|_2$ \\
  \hline
  Factor scaling & $\max_i |a_i|{\|u\|_2}$ & . &  $\sqrt{\sum_i a_i^2}\|u\|_2$ & $\sqrt{\sum_i a_i^2}\|u\|_2$ \\
%  \hline
%  Balanced tree & . & $\sqrt{(\sqrt{Y}-1)(\sqrt{2}+1)}$ & $\sqrt{Y\log_{2}Y}$ & $\cdot$\ \\
  \hline\end{tabular}
\end{center}

  \caption{ Comparing $\|U\|_{s}, \|U\|_{s,\mathcal{N}}, \|U\|_F$ and
  $\|U\|_* $
in different situations. See the text for   details.
(1) Full sharing,
$U=[u\; u\;
\dots\; u]^T,\exists n',\forall l,n'\in \bar{\mathcal{A}}(l)$. (2) No sharing,  $\forall l\neq l', \bar{\mathcal{A}}(l)\cap \bar{\mathcal{A}}(l')=\emptyset$.
(3) Disjoint  feature, 
$U=[u_{1}\; u_{2}\;
\dots \; u_Y]^T, \forall l_1\neq l_2, \mbox{Supp}(u_{l_1})\cap \mbox{Supp}(u_{l_2})=\emptyset$. (4) Factor scaling, $U=[a_1 u\; a_2 u
\dots \; a_Y u]$.
%(5) Balanced tree with $Y$ leaves.\   
}\label{compare_norms}
\end{table} 

We first show a lower bound for $\|\cdot\|_{s}$, $\|\cdot\|_{s,\mathcal{N}}$ which will be useful for the later proofs.
\begin{lem} \label{lower_bound}For $U\in \mathbb{R}^{Y \times d}$,
\begin{align*}
\|U\|_{s,\mathcal{N}}\ge \|U\|_{s}\ge\max_y\|U_y\|_2^{2}
\end{align*}
where $U_y$ is $y$-th row vector of $U$.
\proof
 Let $A\in \mathbb{R}^{Y\times M}, V\in \mathbb{R}^{M \times D}$ be the matrices which attain minimum in 
 $\|U\|_s=\min_{AV=U,\|A\|_{2\rightarrow\infty}\le 1}\|V\|_F$. Since  $A_{r,\cdot} V_{\cdot, c}=U_{r,c}$ and from the cauchy-schwarz,  $\|U_{r,c}\|\le \|A_{r,\cdot}\|_2\cdot \|V_{\cdot, c}\|_2= \|V_{\cdot, c}\|_2$, and if we square both sides and sum over $c$, $\|U_{r,\cdot}\|^2_2\le  \|V\|_F^2=\|U\|^{2}_s$ which holds for all $r$. \qed
\end{lem}
 Following are the detailed descriptions for table \ref{compare_norms} and  the sketch of the proofs. 

\begin{description}
\item 
[Full sharing] 
If  all weights are  same for all classes, i.e., $U=[u\; u\;
\dots\; u]^T\in \mathbb{R}^{Y\times d}$ for $u\in \mathbb{R}^{d}$, and there exists a node $n$ that it
is shared among all $y$, i.e., $\exists n,\forall l,n\in \bar{\mathcal{A}}(l)$, then $\|U\|_{s,\mathcal{N}}^{2}=\|u\|^{2}_2$
whereas $\|U\|^2_F=\|U\|_*^2=Y\cdot\|u\|^{2}_2$.

$\|U\|_{s,\mathcal{N}}^{2}=\|u\|^{2}_2$ can be shown with  matrix  $A=[{\bf{1}_{Y,1}} \;  {\bf{0}_{Y,M-1}}]\in \mathbb{R}^{Y\times M}$  and $V= \begin{bmatrix} u\\ {\bf{0}_{M-1,d}}\end{bmatrix}$  where  ${\bf{n}_{r,c}}\in\mathbb{R}^{r\times c}$ is a matrix with all elements set to $n$. $U=AV$ and the factorization attains the minimum of \eqref{eq:snorm1} since it attains the lower bound from lemma \ref{lower_bound}. $\|U\|_*^2=Y\cdot\|u\|^{2}_2$ is easily shown from the fact that $U$ is a rank one matrix with a singular value of $\sqrt{Y}\cdot\|u\|_2$.

\item[No sharing]\label{lem:no_sharing}  If there is no shared node, i.e., $\forall l,l'\in [Y],l\neq l', \bar{\mathcal{A}}(l)\cap \bar{\mathcal{A}}(l')=\emptyset$, then  $\|U\|_{s,\mathcal{N}}^{2}=\|U\|^{2}_F$.
%\proof

 To show this, let $A$ and $V$ be the matrices which attain the minimum of (\ref{rho2}).
  $m$-th element of $A_{y}$ is zero  for all $y$ except one and $V_{m,d}$ is nonzero only for one $y$ such
  that $m\in y$.
Therefore, (\ref{rho2})
 decomposes w.r.t $A_{y}$ and $V_{y,d}$, where $V_{y,d}$
 is the $d$-th column vector of
 $V$  taking only for row $y$.
  \begin{align*}
 \min_{AV=U}\|V\|^{2}_F= \sum_y\sum_d\min_{A_yV_{y,d}=U_{y,d}}\|V_{y,d}\|^{2}_2
 \end{align*}
  Given $\|A_{y}\|_2=1$, 
\begin{align*}
\|V_{y,d}\|_2=\|A_{y}\|_2\|V_{y,d}\|_2\ge |A_{y}\cdot V_{y,d}|=|U_{y,d}|\\
\end{align*}
And  let $A_y=V_{y,d}/\|V_{y,d}\|_2$ which attains the lower bound.
\begin{align*}
\therefore\min_{AV=U}\|V\|^{2}_F= \sum_y\sum_d |U_{y,d}|^2=\|U\|^2_F
\end{align*}

%\qed
\item[Disjoint feature] If  $U=[u_{1}\; u_{2}\;
\dots \; u_Y]^T\in \mathbb{R}^{Y\times d}$ for $l\in [Y]$, $u_{l}\in \mathbb{R}^{d}$, and the support of $w_y$ are all disjoint, i.e., $\forall y_1\neq y_2, \mbox{Supp}(u_{y_1})\cap \mbox{Supp}(u_{y_2})=\emptyset$, then $\|U\|_{s,\mathcal{N}}^{2}=\|U\|_F^{2}=\sum_y\|u_{y}\|^{2}_2$
and $\|U\|_*^2=(\sum_y \|u_y\|_{2})^2$.
%\proof
 
For $\|\cdot\|_s$, it is similar to  no sharing. The factorization decomposes w.r.t. each column $u$. For the trace norm, since the singular values are invariant to permutations of rows and columns, $U$ can be transformed to a block diagonal matrix by  permutations of rows and columns, and the singular values decompose w.r.t block matrices with corresponding singular values of $\|u_y\|$.

\item[Factor scaling] 
If  $U=[a_1 u\; a_2 u
\dots \; a_Y u]\in \mathbb{R}^{Y\times d}$ for $l\in [L]$, $u\in \mathbb{R}^{d}$, %and there exists a node $n$ such that is shared among all $y$, i.e., $\exists n,\forall l,n\in \bar{\mathcal{A}}(l)$,
 then $\|U\|_{s}^{2}=\max_{l}a_l^2\|u\|_2^{2}$
and $\|U\|_F^2=\|U\|_*^2=\|a\|_2^{2}\cdot\|u\|_2^2. $
 
 Proof is similar to full sharing. For $\|\cdot\|_s$, $A=\dfrac{1}{\max_i a_i}[{[a_1 \; a_2 \; \dots a_Y]^T \; \bf{0}_{Y,M-1}} ]$  and $V=\max_i a_i \begin{bmatrix} u\\ {\bf{0}_{M-1,d}}\end{bmatrix}$  is a feasible solution which attains the minimum in lemma \ref{lower_bound}. For the trace norm, singular values can be easily computed with knowing $U$ is a rank $1$ matrix.  

%,$ and  $\forall n'$, where $W_{n'}$ is $n'$-th row vector of $W$. 

\begin{comment}
\item[Balanced tree] 
Let $\mathcal{G}(\mathcal{N},\mathcal{E})$ be a balanced binary tree graph
of depth $D=\log_2Y\ge1$.
% Assign an unique  index $k_e\in [d], d=|\mathcal{E}|,$ for each edge $e\in \mathcal{E}$. 
Let label
$l\in [Y]$ be leaf nodes, and for $U\in \mathbb{R}^{Y\times |2^{D+1}-1|}$, $U_{l,e}=1$ if the path from $l$ to
root passes the edge $e$, or 0 otherwise for each label $l$.  Then, $\|U\|^2_{s,\mathcal{N}}=(\sqrt{2}^D-1)(\sqrt{2}+1) $, and $\|U\|^2_F=D2^D$.

This follows directly from plugging the tree structure into lemma \ref{lem:tree_shared} \end{comment}
%,$ and  $\forall n'$, where $W_{n'}$ is $n'$-th row vector of $W$. 

\end{description}
\section{Optimization} \label{sec:opt}
 In this section, we discuss the details of optimizing objectives \eqref{NHSVM} and \eqref{SSVMp}. Specifically, we show how to obtain the most violating label for multi-labels problems for  objective \eqref{NHSVM} and an efficient algorithm to optimize objective \eqref{SSVMp}.  
\subsection{Calculating the most violating label for multi-label problems}
We optimize our training objective \eqref{NHSVM} using SGD \cite{Shalev-Shwartz:2007:PPE:1273496.1273598}.
 In order to do so, the most challenging part is calculating
\begin{align} \label{argmax_}
\hat{y}_{i}=\arg\max_{y\in\mathcal{Y}}  \sum_{n \in y}\sqrt{\alpha_n}W_n x_i -  
 &\sum_{n \in y_i} \sqrt{\alpha_n}W_n x_i  +\tilde{\triangle}(y,y_i)=\arg\max_{y\in\mathcal{Y}}
 L_{i}(y)
 \end{align}
 at each iteration. For single label problems, we can calculate $\hat{y}_i$ by enumerating all the labels. However, for a multi-label problem, this is intractable
because of the exponential size of the label set. Therefore, in this subsection,  we describe how to calculate $\hat{y}_i$ for  multi-label problems. 

If $L_i(y)$ decomposes as a sum of functions with respect to its nodes, i.e., $L_{i}(y)=\sum_n L_{i,n}(\text{1}\{n\in y\})$, then 
$\hat{y}_i$ can be found efficiently. 
Unfortunately,  
%For an efficient computation of $\hat{y_i}$,  $L_i(y)$ in \eqref{argmax_} %needs to decomposes as a sum of  functions with respect to its nodes, i.e., %$L_{i}(y)=\sum_n L_{i,n}(\text{1}\{n\in y\})$. But 
 $\tilde{\triangle}(y,y_i)$ does not decompose with respect to the nodes.   In order to allow efficient computation for  multi-label problems, we actually replace $\tilde{\triangle}(y,y_i)$ with a decomposing approximation  $\triangle_\mathcal{L} l(y,y_i)=|\{y\cap \mathcal{L}\}-\{ y_{i}\cap \mathcal{L}\}|$
$=\sum_{l\in\mathcal{L}}\left ({\bf{1}}_{\{l\in y- y_{i}\}}-{\bf{1}}_{\{l\in y\cap y_{i}\}}\right)+|\{ y_i\cap \mathcal{L}\}|$ instead. When $\triangle_\mathcal{L} l(y,y_i)$ is used and the  graph $\mathcal{G}$ is a tree,   $\hat{y}_i$ can be computed in time $O(M)$ using dynamic programming. 
%\footnote{We tried using the decomposable upper bound of the structured loss $\tilde{\triangle}$, but there was not a significant improvement. For the details, see appendix \ref{apx:decomposable}. }

\begin{comment}
However, structured error term (\ref{wserror})  does not decompose  w.r.t
the nodes.  We use   decomposable upper bound for $\tilde{\triangle}(y,y_i)$  for a multi-label problem.
 \begin{lem}\label{lem:decompose}
 $\triangle_d(y,y_i)= \sum_{n\notin y_i} 1\{n\in y\}\sqrt{\alpha_n}-\sum_{n \in y_i} (1-1\{n\in y\}) \sqrt{\alpha_n}$ is a tight decomposable upper bound for $\tilde{\triangle}(y,y_i)$.
\proof 
$e_n=1\{n\in y\}$ if $n\notin y_i$ or $1-1\{n\in y\}$ otherwise. 
Then $\tilde{\triangle}(y,y_i)=\sqrt{\sum_n e_n\alpha_n}\le \sum_n \sqrt{e_n\alpha_n} = \sum_n e_n\sqrt{\alpha_n}.$
\qed
\end{lem}
\end{comment}

When the graph $\mathcal{G}$ is a DAG,  dynamic programming is not applicable.
However, finding \eqref{argmax_} in a DAG structure can be formulated
into  the following integer programming.  
\begin{align}
\hat{z}=&\underset{ z\in \{0,1\}^M}{\arg\min} &\sum_{n=1}^Mz_n \cdot  r_n   \label{integerprogram} \\
&\mbox{s.t} 
%&0\le z_n \le 1, &&\forall n \nonumber \\
&\sum_{c \in \mathcal{C}(n)}z_c \ge z_n, & & \forall n \nonumber\\
&&z_c \le z_n,  && \forall n, \forall c \in \mathcal{C}(n) \nonumber\\
\nonumber
&& \sum_{l\in \mathcal{L}}z_l\ge1  
\end{align}
where $r_n=\sqrt{\alpha_n}W_n x_i +{\bf
{1}}_{\{n\notin  y_{i},n\in\mathcal{L}\}}-{\bf
{1}}_{\{n\in y_{i},n\in\mathcal{L}\}}$. The feasible label  from \eqref{integerprogram} is the set of labels where if a node $n$ is in the label $y$, at least one of its child node is in $y$, i.e.,  $\forall n\notin \mathcal{L}, n\in y \implies \exists c\in\mathcal{C}(n), c\in y$, and all the parents of $n$ are in the label, i.e., $\forall n\in y\implies \forall p\in \mathcal{P}(n), p\in y$. The feasible set  is equivalent to $\mathcal{Y}_m$.
 The search problem (\ref{integerprogram}) can be shown to be NP-hard by reduction from the set cover problem. We relax the  integer program into a linear program for training. Last constraint of $\sum_{l\in \mathcal{L}}z_l\ge1$ is not needed for an integer program, but yields a tighter LP relaxation. In testing, we rely on the binary integer programming only if the solution to
LP is not integral. In practice, integer programming solver is effective for this problem, only 3 to 7
times slower than linear relaxed program using gurobi solver \cite{gurobi}.
%\footnote{In some datasets, intermediate nodes also can be labels, which is not the focus of this paper. In this case, finding $\hat{y}_i$ can be also be done using linear program. The linear program always shown to have an integral solution. See appendix \ref{apx:intermediate}.}

\subsection{Optimizing with the shared norm}
 Optimization \eqref{SSVMp} is a convex  optimization jointly  in  $U$ and $\alpha$, and thus,  has a global optimum. 
%First we show that the optimization with the shared norm (\ref{SSVMp}) is jointly convex in $U_n$ and $\alpha_n$. 
 %Optimization w.r.t $\alpha$ is QCQP.

\begin{lem}
Optimization \eqref{SSVMp} 
is a convex optimization jointly in $U$ and $\alpha$.  \proof
 Let $f(U,\alpha)=\sum_n \sum_d f_{n,d}(U_n,\alpha_n)$ where $f_{n,d}={U}_{n,d}^2/\alpha_n$. The Hessian of each $f_{n,d }$ can be calculated easily by differentiating twice. Then, the Hessian is a semipositive
 definite matrix for $\forall \alpha_n\ge0$, since   if $\alpha_n>0$, $\nabla^{2} f_{n,d}=\begin{bmatrix} \frac{\partial^2 f_{n,d}}{(\partial U_{n,d})^2} & \frac{\partial^2 f_{n,d}}{\partial\alpha_n\partial U_{n,d}} \\ \frac{\partial^2 f_{n,d}}{\partial\alpha_n\partial U_{n,d}} & \frac{\partial^2 f_{n,d}}{(\partial\alpha_n )^2}\end{bmatrix}=\frac{2}{\alpha_n}\begin{bmatrix} 1 & -\frac{U_{n,d}}{\alpha_n} \\ -\frac{U_{n,d}}{\alpha_n} & \frac{U_{n,d}^2}{\alpha_n^2}\end{bmatrix}$, and if $\alpha_n=0$ we can assume $\|U_n\|_2=0$ by restricting the domain and  the hessian to be  a zero matrix. Thus, $ \underset{n}{\sum} \dfrac{\|U_n\|^2_2}{\alpha_n}$ is a convex function jointly in $U_n$ and $\alpha_n$, and the lemma follows from the fact that the rest of the objective function in \eqref{SSVMp} is convex in $U_n$. \qed
\end{lem}
 Since it is not clear how to  jointly optimize efficiently with respect to $U$ and $\alpha$,  we present an efficient method to optimize \eqref{SSVMp} alternating between $\alpha$ and $V_{n}=\dfrac{U_n}{\sqrt{\alpha_n}}$. Specifically, we show how to calculate the optimal $\alpha$ for a fixed $U$ in closed form in time $O(M)$ when  $\mathcal{G}$ is a tree where $M$ is the number of nodes in the graph, and for fixed $\alpha$, we optimize the objective using SGD \cite{Shalev-Shwartz:2007:PPE:1273496.1273598} with change of variables $V_{n}=\dfrac{U_n}{\sqrt{\alpha_n}}$.

As for the calculation of $\alpha$,  
optimization (\ref{SSVMp}) for fixed $U$ 
is an optimization with respect to  $\alpha$, i.e.,
 \begin{align}\label{optimize_alpha}
&\min_{\alpha_n}&
 \lambda \sum_{n\in [M]}\dfrac{\|U_n\|_2^2}{\alpha_n}&
\\
&\mbox{s.t}&\sum_{n\in \mathcal{A}(y)} \alpha_n\le 1,&&\forall y\in \mathcal{Y} \nonumber
\\
& &\alpha_n\ge 0,&&\forall n\in \mathcal{N} \nonumber.
 \end{align}
%If  $\mathcal{G}$ is a tree structure, the optimal $\alpha$ in \eqref{optimize_alpha} can be calculated very efficiently.

\begin{algorithm}
  \caption{Calculate the optimal $\alpha$ in \eqref{optimize_alpha} for the tree structure $\mathcal{G}$ in $O(M)$.  We assume nodes are sorted in increasing order of depth,i.e., $\forall n, p\in \mathcal{P}(n), n>p$.
}\label{argmin_alpha}
  \begin{algorithmic}[1]
    \Procedure{ArgminAlphaTree}{$U,\mathcal{G}$}\Comment{return $\arg\min \alpha$}
    \INPUT {$U\in \mathbb{R}^{M\times d},$ a tree graph $\mathcal{G}$}
    \OUTPUT{$\alpha \in \mathbb{R}^{M}$.}
    \INIT  {$\alpha=N=E=[0\;0\dots\;0] \in\mathbb{R}^M, L=[1 \;1\dots1]\in\mathbb{R}^M$}
    \For {$n = M \to 1$} \Comment{Calculate minimum norm with a bottom-up order}
        \State $N_n\gets \|U_{n,\cdot}\|_{2}^2,E_n\gets 0$
      \If{$|\mathcal{C}(n)|\neq0$} \Comment{If $n$ is not a leaf}
      \State $S\gets \sum_{c\in\mathcal{C}(n)}N_c$
        \If{$\sqrt{S}+\sqrt{N_n}\neq0$} 
        \State $N_n \gets (\sqrt{N_n}+\sqrt{S})^2,E_n\gets \dfrac{\sqrt{N_n}}{\sqrt{S}+\sqrt{N_n}}$
        \EndIf{}
      \EndIf{}

    \EndFor
    \State $L_1\gets 1- E_1$
    \For {$n = 2 \to M$} \Comment{Calculate minimum $\alpha$ with a top-down order}
        \State $p\gets \mathcal{P}(n)$
        \State $\alpha_n=\sqrt{L_{p}\cdot E_n}$
        \State $L_n=L_{p}\cdot(1-E_n)$
    \EndFor
      \State \textbf{return} $\alpha$
    \EndProcedure
  \end{algorithmic}
\end{algorithm}
 Algorithm \ref{argmin_alpha} shows how to calculate optimum $\alpha$ in  \eqref{optimize_alpha} in time $O(M)$ for a tree structure.
%\begin{lem}
%\label{lem:tree_shared} If $\mathcal{G}$ has a tree structure, algorithm %\ref{argmin_alpha} finds optimal $\alpha$ in \eqref{optimize_alpha} in $O(M)$ in a closed form. 
% \proof
 The following briefly describes the algorithm \ref{argmin_alpha}. Let $f(n,l)=\min_{\sum_{n\in D(n)}\alpha_n\le l}$ $\sum_{n\in \bar{D}(n)}$ $\dfrac{\|U_n\|^2_2}{\alpha_n}$
 where $\bar{D}$ denotes the union set of $\{n\}$ and descendent nodes of $n$.
the following recursive relationship
 holds, since  $\mathcal{G}$ has a tree structure. 
 
\begin{align} 
f(n,l)=\begin{cases} \dfrac{\|U_n\|^2_2}{l} & \mbox{if $n$ is a leaf node} \\ \min_{0<k<1}\dfrac{\|U_n\|^2_2}{l\cdot k}+\sum_{c\in C(n)}f(c,l(1-k))  &  otherwise\end{cases}
\label{recursive}
\end{align}
If $n$ is a parent node of  leaf nodes,
\begin{align}
f(n,l)= \min_{0<k<l}\dfrac{B_1}{l\cdot k}+\dfrac{B_2}{l(1-k)}
\label{eq:min_k}
\end{align}
where $C(n)$ denotes the set of children nodes of $n$, $B_1=\|U_n\|^2_2$ and
$B_2=\sum_{c\in C(n)}\|U_{c}\|^2_2$. %{eq:min_k}
This has a
closed form solution, 
\begin{align}
f(n,l)= 
\frac{1}{l}(\sqrt{B_{1}}+\sqrt{B_{2}})^2 
\label{leaf_f}
\end{align}
and the minimum is attained at $k=\frac{\sqrt{B_1}}{\sqrt{B_1}+\sqrt{B_2}}$.
 For nodes $p$ of $n$, $f(p,l)$
will also have a form of (\ref{leaf_f}), since  the equation (\ref{leaf_f})  has a form of leaf node,  and the recursive relationship
(\ref{recursive}) holds. We  continue this process until the root node $r$ is reached,
and $f(r,1)$ is the optimum. The optimal $\alpha$ can be calculate backward. 

% \qed
%\end{lem}
%Algorithm \ref{argmin_alpha} describes the implementation of  the method in lemma \ref{lem:tree_shared}. 
%By fixing $\alpha$  we can rely on SGD$ \cite{Shalev-Shwartz:2007:PPE:1273496.1273598}$ % and optimize w.r.t $V_{n}=\dfrac{U_n}{\sqrt{\alpha_n}}$, and in the experiments, %we find it very efficient for large datasets to optimize (\ref{SSVMp}) alternating %between $V$ and $\alpha$ for the tree structured datasets.  Specifically, 
In the experiments, we optimize $\alpha$ using algorithm \ref{argmin_alpha} with $U_n=\sqrt{\alpha_n} V_n$ after a fixed number of epochs of SGD with respect to $V$, and repeat this until the objective function converges. We find that the algorithm is efficient enough to scale up to large datasets.

\section{Experiments}
 In this section, the performance of NHSVM and SSVM are evaluated empirically from synthetic
 datasets to a large scale data. Data statistics
 is summarized in table \ref{exper_spec}.   Results are compared with HSVM \cite{cai2004hierarchical}, and flat SVM, as well as, HR-SVM \cite{gopal2013recursive} in competition data.  HR-SVM is similar to HSVM with  multi-task loss.  For each experiments, the different parameters are tested on the the hold out dataset. Fixed set of $\lambda$ is tested, $\lambda\in\{10^{-8},10^{-7},\dots,10^{2}\}$. For NHSVM is tested with $\rho=2$, and $\rho=1$  in \eqref{NHSVM} and \eqref{NHSVM3}. Also $\rho=2$ is tested with directional constraints. For both SWIKI, $T=1.5$ is used in \eqref{con:range}.
And each model with the parameters which had the best holdout error  is trained
 with all the training data, and we report test errors.  We also added comparisons with
%\ M3L\cite{hariharan2010large} in a multi-label dataset with linear  kernel without prior information of labels, $R=I$, and 
FastXML\cite{prabhu2014fastxml} in the competition dataset. FastXML is very efficient ranking method suitable for large structures. Since FastXML  predicts rankings of full labels rather than list of labels, we predicted with the same number of labels as NHSVM, and compared the result. As for the training time of HR-SVM, it is parallelizable and faster than NHSVM since it does not rely on linear programming, but not as fast as FastXML.

\begin{figure}
\begin{center}
%\begin{floatrow}
%\capbtabbox{%
\begin{tabular}{| l | c | c |c|c|c|}
    \hline
     & $M$& $d$&$N$&$|\mathcal{L}|$& $\overline{\mathcal{L}}$\\
    \hline
    Synthetic(B)  & 15 & 1K& 8K & 8 & 1\\
    Synthetic(U)  & 19 & 1K& 10K & 11 & 1\\
%    \hline
%    LSHTC-small & 2387 & 17374&4463&1563 \\
%    \hline
%        CLEF & 97 & 80 &   11006 & 63   \\
        IPC & 553& 228K& 75K& 451 & 1
\\   
    SWIKI d5 & 1512& 1000&41K&1218 & 1.1\\
%    \hline
    ImageNet & 1676&51K&100K&1000 & 1\\
%    \hline
    DMOZ10  & 17221 & 165K& 163K & 12294 & 1\\
    SWIKI  & 50312& 346K&456K&36504 & 1.8\\
    \hline
    \end{tabular}
  \end{center}
%}{%
  \caption{Data statistics: $M$ is the number of nodes in the graph. $d$
  is the dimension of the features. $N$ is the number of the instances.  $|\mathcal{L}|$
  is the number of labels.
 $\overline{\mathcal{L}}$ is the average labels per instance. $\overline{\mathcal{L}}=1$
 denotes a single-label dataset.
 }\label{exper_spec}
%}
\end{figure}

\subsection{Synthetic Dataset}
For the intuitive hierarchical synthetic datasets, we compare the performance of the different models. For the balanced synthetic data, a weight vector $W_n\in\mathbb{R}^d$ for each node $n\in[2^3-1]$ in the complete balanced tree with depth 4 and an instance vector $x_i\in\mathbb{R}^d$,$i\in [N]$,$N=15000,$ for each instances are sampled from the standard multivariate normal distribution. Instances are assigned
to labels which have maximum potential.
For the unbalanced synthetic data,
 we sample $x_i\in \mathbb{R}^d$ from the multivariate normal distribution
 with $d=1,000, i\in [N], N=10,000$, and normalize its norm to 1.  We divide
 the space $\mathbb{R}^d$ with a random hyperplane recursively so that the divided spaces form an unbalanced binary tree structure, a binary tree growing only
in one direction.  Specifically, we divide the space  into two spaces with a random hyperplane, which form two child
spaces, and recursively
divide only one of the child space with a random hyperplane until the depth of the binary tree reaches 10. 
%Corresponding structure for depth 3 is shown in the figure \ref{fig:syn}. 
Each $x$ is assigned to leaf nodes if $x$ falls into the corresponding space. $\rho=2$ is used. For results  see the Table \ref{syn}. HSVM fails to exploit the hierarchical structure of the unbalanced dataset with  the accuracy close to flat model, whereas NHSVM achieves higher accuracy by 5\% over flat model. The accuracy gain of NHSVM against HSVM for the balanced dataset, shows the advantage of (\ref{NHSVM2}) and normalized structured loss(\ref{wserror}). For the unbalanced dataset, SSVM further achieves 2\% higher accuracy compared to NHSVM learning the underlying structure from the data.  For the balanced dataset, SSVM performs similar to NHSVM.
\begin{comment}
\begin{centering}
\begin{figure}
      \includegraphics[width=.4\textwidth]{figs/synthetic_dataset.pdf}
    \caption{Simple hierarchical synthetic dataset with unbalanced binary tree structure. Each branch represents a random hyperplane. Here depth is 3, and depth 10 is used
    for the experiments}
  \label{fig:syn}
\end{figure}
\end{centering}      
\end{comment}

%\capbtabbox{%
\begin{figure}
\begin{center}
 \begin{tabular}{| c | c | c|c|}

    \hline
    Method & Balanced& Unbalanced \\
    \hline
   % \multirow{2}{*}{}    
     SSVM & 63.6 & 75.1 \\
     NHSVM &63.5 & 73.1 \\
%     NHSVM(\ref{NHSVM3})& 0.739 \\
     HSVM &62.7   & 68.6 \\
     Flat SVM  &60.5 & 68.5\\
    \hline  
%    \begin{tabular}{| c | c | c |}
%    \hline
%    \hline
%    Method & Depth=5 & Depth=10 \\
%    \hline
%    \multirow{2}{*}{}    
%     NHSVM(\ref{NHSVM2})& .741 & 0.741 \\
%     NHSVM(\ref{NHSVM3})& .739 & 0.739 \\
%     HSVM \cite{Cai:2004}    & .690 & 0.690 \\
%Flat SVM & 0.639 & 0.688\\
%    \hline  
    \end{tabular}
%}{%
  \caption{Accuracy on synthetic datasets.}\label{syn}
%}
%\end{floatrow}
\end{center}
\end{figure}

\subsection{Benchmark Datasets}

We show the results on  several benchmark datasets in different fields without
restricting domain to the document classification, such as  ImageNet in table
\ref{bench_result}. Results  show a consistent improvement
over our base model.     NHSVM outperforms our
 base methods. SSVM shows an additional increase in the performance.
DMOZ 2010 and SWIKI-2011 are
from LSHTC competition\footnote{http://lshtc.iit.demokritos.gr/}.  
%  \begin{enumerate}
%\item[.] CLEF \cite{Dimitrovski08:proc} dataset is a single label medical X-Ray image data from CLEF 2007 competition. $M=96$, $d=80$, $N=10000$, $|\mathcal{L}|=63$ .
IPC\footnote{http://www.wipo.int/classifications/ipc/} is a single label patent document dataset. DMOZ 2010 is a  single label web-page collection. 
%LSHTC-small is a  single label web-page collection. % $M=2,387$,$d=51,033$, $N=4,463,$ $|\mathcal{L}|=1563$.
%\item[.] 
SWIKI-2011
is a multi-label dataset of wikipedia pages, depth is cut to 5 (excluded
labels with depth more than 5). % $M=1,512$,$d=1,000$,$N=41,921$
%,$|\mathcal{L}|=1,218$.
%\item[.] 
ImageNet data \cite{ILSVRCarxiv14} is a single label image data with
SIFT BOW features from development kit 2010.
SWIKI and ImageNet have DAG structures, and the others have tree structures.%For DMOZ 2010 and ImageNet data, (\ref{simple_nclass}) is used. For SWIKI-2011 (both for depth 5 and full dataset), $T=1.5$, (\ref{nbeta_1}) is used.
  %$M=1,676$,$d=51,033$, $N=100,000,$
%$|\mathcal{L}|=1,000$.
%\end{enumerate}
\begin{figure}
\begin{center}
    \begin{tabular}{| l | c |c|c|c| }
    \hline
      Method & IPC 
      &DMOZ 
%      &CLEF 
      &  SWIKI d5 & Imagenet \\
%    \hline
%    \multirow{2}{*}{CLEF}    
%     & NHSVM & 0.803 \\
%     & HSVM     & 0.803 \\
% & Flat SVM & 0.802\\
    \hline
    %\multirow{2}{*}{LSHTC-small}    
%    LSHTC-small
%     & NHSVM & 0.492 \\
%     & HR-SVM & 0.4531 \\
%\textsl{a}     & HMSVM( $\bigtriangledown$ ) & 0.435 \\
%     & HSVM     & 0.471 \\r
% & Flat SVM & 0.489\\
%    \hline
%    

      SSVM &  52.9     
      &45.5 
%      &0.805
      & *      & *\\ 
      NHSVM & 52.2     
      &45.5 
%      &0.801 
      & 61 &8.2\\
      HSVM     & 50.6  
      &45.0 
%      &0.802 
      & 59 & 7.5\\
      Flat SVM & 51.7 
      & 44.2 
%      &0.799 
      & 57.7 & 7.4\\
%      M3L\cite{hariharan2010large} & & & 40.9 & \\
    \hline
    
    \end{tabular}
  \caption{Accuracy on benchmark datasets: * denotes that SSVM was not able to be applied due to the graphical structure of the data.}
\label{bench_result}
\end{center}
\end{figure}
\subsection{Result on LSHTC Competition }
 In table \ref{comp_result}, we show the result on full competition dataset, SWIKI-2011, and compare with results currently reported. NHSVM was
 able to adapt to the large scale of SWIKI-2011 dataset with the improved
 results. Only 98,519 features appear in the test set are used. With a computer with  Xeon(R) CPU E5-2620  processor, optimization took around 3 weeks with  matlab code.
\begin{figure}
\begin{center}
%\begin{subtable}{.4\linewidth}\centering
    \begin{tabular}{ |c | c |}
    \hline
      Method & Accuracy \\
    \hline
      NHSVM     &  42.38 \\
      HSVM      &  38.7 \\
      %HR-SVM
       HR-SVM*\cite{gopal2013recursive}\    & 41.79 \\
       FastXML**\cite{prabhu2014fastxml}\ & 31.6\\
      Competition Winner    & 37.39  \\
    \hline
    \end{tabular}
  \caption{Results on full S-WIKI. *The inference of HR-SVM relies on the other meta learning method\cite{gopal2010multilabel} for high accuracy. ** NHSVM is used to predict the number of labels in the inference.   }\label{comp_result}
\end{center}
\end{figure} 

\begin{comment}
% \begin{table}
%\begin{subtable}{.6\linewidth}\centering
\begin{figure}
    \begin{tabular}{| l | c |c|c|c| }
    \hline
      Method & IPC 
      &DMOZ 
%      &CLEF 
      &  SWIKI d5 & Imagenet \\
%    \hline
%    \multirow{2}{*}{CLEF}    
%     & NHSVM & 0.803 \\
%     & HSVM     & 0.803 \\
% & Flat SVM & 0.802\\
    \hline
    %\multirow{2}{*}{LSHTC-small}    
%    LSHTC-small
%     & NHSVM & 0.492 \\
%     & HR-SVM & 0.4531 \\
%\textsl{a}     & HMSVM( $\bigtriangledown$ ) & 0.435 \\
%     & HSVM     & 0.471 \\
% & Flat SVM & 0.489\\
%    \hline
%    

      SSVM &  52.9     
      &45.5 
%      &0.805
      & *      & *\\ 
      NHSVM & 52.2     
      &45.5 
%      &0.801 
      & 61 &8.2\\
      HSVM     & 50.6  
      &45.0 
%      &0.802 
      & 59 & 7.5\\
      Flat SVM & 51.7 
      & 44.2 
%      &0.799 
      & 57.7 & 7.4\\
    \hline
    
    \end{tabular}
  \caption{Accuracy on benchmark datasets: * denotes that SSVM was not able to be applied due to the graphical structure of the data.}

\label{bench_result}
\end{figure}
%\end{subtable}
\begin{figure}
%\begin{subtable}{.4\linewidth}\centering
    \begin{tabular}{ |c | c |}
    \hline
      Method & Accuracy \\
    \hline
      NHSVM     &  42.38 \\
      HSVM      &  38.7 \\
      %HR-SVM
       \cite{gopal2013recursive}\    & 41.79 \\
      Competition Winner    & 37.39  \\
    \hline
    \end{tabular}
  \caption{Results on full S-WIKI. }\label{comp_result}
\end{figure}
%  \end{subtable}
%\begin{comment}  
%\caption{
%(\ref{bench_result})
%}
%\end{table} 

\end{comment}

\section{Conclusion}

This chapter addressed issues of non-uniform regularization and margin in hierarchical SVM by introducing a new model,  normalized hierarchical
SVM. We
show its invariance property over duplicating nodes. From the motivation
of normalized hierarchical SVM, a new norm, structured shared Frobenius norm, is proposed.
Structured shared Frobenius norm favors matrices 
with a shared structure. 
 We  introduced an efficient method for learning in a large size
multi-label problem in complicated structures. The experimental results show the benefit of our method including an improved results on the very large competition dataset. 

\chapter{Conclusion}

In this dissertation, we first investigated surrogate losses of structured prediction.  We presented a wide general class of surrogate losses while maintaining the efficiency of  inference. Efficiency was based on the decomposability. We also explored the limitation of this approach and presented a practical and efficient inference methods for the generalized the surrogate loss. Efficiency and the performance improvement were empirically shown in several problems including the models utilizing the deep networks.

 In the second part of the dissertation, we showed that the regularization bias was induced by the structural imbalance in the hierarchical classification problem. We presented an approach to remove the bias. We extended the data dependent method to adapt to the data as well as the structure, which can be represented as a new norm suitable for the unbalanced structure. Empirical study shows that our normalization improves over the unnormalized method, which plays an important role in high performance for  a large dataset with an imbalanced label structure.

%\part{Regularization}

%\include{chap1}
%\include{chap2}
\appendix
%\include{appa}
%\include{appb}
%% This defines the bibliography file (main.bib) and the bibliography style.
%% If you want to create a bibliography file by hand, change the contents of
%% this file to a `thebibliography' environment.  For more information 
%% see section 4.3 of the LaTeX manual.
\begin{singlespace}
\bibliography{main}
%\bibliography{NHSVM_arvix}
%\bibliography{slack_rescaling}
\bibliographystyle{plain}
\end{singlespace}

%\bibliographystyle{plain}
%\bibliography{main}

\end{document}